%% file: main.tex
\pdfoutput=1
\documentclass{article}
\usepackage{iclr2027_conference,times}
\iclrfinalcopy

\input{math_commands.tex}
\input{counts.tex}

\input{counts_clamp.tex}
\usepackage{hyperref}
\usepackage{url}
\usepackage{booktabs}
\usepackage{amsmath,amssymb,amsthm}
\usepackage{dsfont}
\usepackage{graphicx}
\usepackage{tikz}
\usetikzlibrary{calc}
\usepackage{pgfplots}
\pgfplotsset{compat=1.17}
\usepackage{xcolor}
\usepackage{placeins}

\newtheorem{theorem}{Theorem}
\newtheorem{proposition}[theorem]{Proposition}
\newtheorem{lemma}[theorem]{Lemma}
\newtheorem{corollary}[theorem]{Corollary}
\theoremstyle{definition}
\newtheorem{definition}[theorem]{Definition}

\theoremstyle{remark}
\newtheorem{remark}[theorem]{Remark}

\newcommand{\pre}{L_{\mathrm{pre}}}
\newcommand{\gated}{L_{\mathrm{gated}}}
\newcommand{\post}{L_{\mathrm{post}}}
\newcommand{\gatedh}{L_{\mathrm{gated\_h}}}
\newcommand{\postw}{L_{\mathrm{post\_w}}}

\newcommand{\postwv}{L_{\mathrm{post\_wv}}}
\newcommand{\gatedhv}{L_{\mathrm{gated\_hv}}}
\newcommand{\act}{\operatorname{act}}
\newcommand{\PX}{\mathcal P_{\mathcal X}}
\newcommand{\QM}{\mathcal Q_{M}}
\newcommand{\KX}{\mathcal K_{\mathcal X}}
\newcommand{\Esc}{\operatorname{Esc}}

\renewcommand{\eqref}[1]{(\ref{#1})}

\title{A Tropical Geometry View of Forgetting:\\
A Per-Unit Projector for\\
Knowledge-Preserving Fine-Tuning}

\input{authors}

\begin{document}
\maketitle
\lhead{Preprint}

\begin{abstract}
Fine-tuning a language model on new text degrades what it already does. Replay-free projectors such as
Adam-NSCL and GPM forbid one shared subspace of a layer's inputs in every row of the update. The tropical
geometry of a ReLU layer shows why this is too coarse. In data space, the units' walls are tropical
hypersurfaces whose cells are dual to the upper vertices of a zonotope; in weight space, each old token is a
hyperplane, and the tokens cut out a polyhedron, the closure of the weights that keep every token on its
side. An exact identity joins the two pictures: the squared change of the layer's output under any weight
change splits into in-cell,
open$\to$closed and closed$\to$open terms, and the first two live on the tokens each unit fires on (its open
tokens). The identity names a gate-aware per-unit projector, and a budget-separation theorem prices exact
protection: it costs a unit the rank of its own open tokens, while a shared subspace pays at least the rank of
their union in every row. On
OPT-1.3b, where $96\%$ of (token, unit) pairs are closed, the projector forgets less than Adam-NSCL at all
six matched budgets from $9$ to $60$ constrained directions per row ($24$ of $27$ seed-pairs,
$p=1.5\times10^{-6}$), the gap widening from $1.1\times$ to $4.3\times$; with $1/5.5$ of the directions it
halves the forgetting of Adam-NSCL at GPM's energy threshold. On OPT-6.7b, it matches Adam-NSCL's
forgetting at matched budget while learning more. As the theory predicts, the open/closed partition is the
operative variable: open tokens beat random, sign-blind and anti-gate token sets on $18$ of $18$
seed-pairs and are equivalent to the first-order criterion, which inherits the partition. In pruning repair,
the identity shows that every derivative-based local model of the output error at the dense weights is blind
to pairs that open: the minimisers of the gate-weighted objective can leave the polyhedron, the objective's
closed-form solution is $1.94$ nats worse than no repair on OPT-1.3b, and a convex one-sided penalty bounds
the escape.
\end{abstract}
\section{Introduction}\label{sec:intro}

A language model fine-tuned on new text loses some of what it could already do. A leading family of
replay-free remedies constrains the update geometrically: Adam-NSCL \citep{wang2021adamnscl} and GPM
\citep{saha2021gpm} project each layer's update onto the approximate null space of the layer's old inputs;
in language models, the same subspace orients or initialises low-rank adapters
\citep{yang2024corda,tang2026loranull}, and AlphaEdit \citep{fang2025alphaedit} projects edits of the MLP
output matrix with it. All of them share one object, a single input subspace per layer, forbidden in
\emph{every} row of the update.

On a matrix whose outputs pass through a ReLU, that object is too coarse. A unit that is closed on a token
contributes exactly zero there, and keeps contributing zero until its pre-activation crosses the threshold;
in OPT-1.3b, $96\%$ of (token, unit) pairs are closed, and a shared subspace constrains all of them.
ReLU-family language models remain a live design point because this sparsity can be exploited at inference
\citep{mirzadeh2024relu,song2025prosparse}.

We start from the geometry (\S\ref{sec:identity}). In data space, a ReLU layer's walls are tropical
hypersurfaces and its gate patterns are dual to the upper vertices of a zonotope; dually, in weight space, the old
tokens cut out a polyhedron, the closure of the weights that keep every token on its side. An exact identity
(Theorem~\ref{thm:decomp}) splits the squared change in a layer's output under \emph{any} weight change into
in-cell, open$\to$closed and closed$\to$open terms, with no remainder and no reference to an objective; the
first two live on the tokens each unit fires on. From it we derive (\S\ref{sec:method}) a gate-aware
per-unit projector: after each step it holds each unit's pre-activation fixed on up to $C$ of the old tokens
it fires on and leaves its closed tokens free. A budget-separation theorem prices this exact protection: it costs a unit the rank of its own open tokens, while a shared subspace pays at
least the rank of their union in every row (Theorem~\ref{thm:budget}). The theory makes six predictions
(P1--P6, end of \S\ref{sec:method}), which organise the experiments.

The measurements bear them out (\S\ref{sec:exp}). On OPT-1.3b, the per-unit projector forgets less than
Adam-NSCL at all six matched budgets from $9$ to $60$ constrained directions per row ($24$ of $27$ seed-pairs,
$p=1.5\times10^{-6}$), the gap widening from $1.1\times$ to $4.3\times$; with $1/5.5$ of the directions it
halves the forgetting of Adam-NSCL at GPM's energy threshold. On OPT-6.7b, it matches Adam-NSCL's forgetting
at matched budget and learns more. At equal per-unit dimension, open tokens beat random, sign-blind and
anti-gate token sets on $18$ of $18$ seed-pairs and are equivalent to the first-order criterion. In pruning
repair, every derivative-based local model of the output error at the dense weights is blind to the
closed$\to$open term, so the minimisers of the gate-weighted objective can leave the polyhedron (the
objective's closed-form solution is $1.94$ nats worse than no repair on OPT-1.3b), and a one-sided penalty
bounds the escape (\S\ref{sec:repair}).

\paragraph{Contributions.}
\begin{itemize}
\item The two dual pictures of a threshold-activation layer, made rigorous: the tropical form and its dual
  subdivision, the activation-preserving polyhedron, and an exact three-term decomposition of the output
  change under any weight change (\S\ref{sec:identity}).
\item A per-unit projector derived from them, with exactness, budget-separation and first-order-inheritance
  results, a Ky Fan characterisation of the best per-unit and shared constraints, and six predictions that
  the experiments test (\S\S\ref{sec:method}--\ref{sec:exp}).
\item For pruning repair: curvature blindness, polyhedral escape and confinement by a convex one-sided
  penalty (\S\ref{sec:repair}).
\end{itemize}

\input{sec_theory}

\section{Experiments}\label{sec:exp}

\paragraph{Setup.} Pre-LN OPT-125m, OPT-1.3b and OPT-6.7b \citep{zhang2022opt} (ReLU, $\theta=0$; on OPT-6.7b
the \texttt{fc1} of the last four blocks). Each run is one fine-tuning transition from the pretrained model, the
setting of \citet{biderman2024loraforgets}, CorDA and LoRA-Null \citep{yang2024corda,tang2026loranull}. The
knowledge to preserve is the pretraining distribution, represented by MiniPile \citep{kaddour2023minipile},
a subset of the Pile, parts of which OPT was pretrained on; it is scored by held-out loss, as in
\citet{ibrahim2024continual}, and by five zero-shot tasks. A $2\,048$-token support builds the projectors,
and a disjoint $1\,024$-token probe measures forgetting. The new task is WikiText-103
\citep{merity2017pointer} or Python code from CodeParrot \citep{codeparrot2022clean}. A run trains
\texttt{fc1} on $65\,536$ tokens: AdamW \citep{loshchilov2019adamw} at lr $10^{-4}$, cosine schedule, $128$
steps, batch $4$. \emph{Forgetting} is the rise in old-probe NLL, \emph{learning} the change in held-out
new-task NLL (more negative is better). Runs are paired by seed; we report paired $t$-tests, sign-flip
permutation tests and random-effects meta-analyses \citep{dersimonian1986meta} (Appendix~\ref{app:cl:setup}).

\subsection{Matched budget, capacity and scale (P1, P4)}\label{sec:budget}

Theorem~\ref{thm:budget} prices protection in constrained directions per row, so that is the axis we hold
equal: each cap is paired with Adam-NSCL at $k_{\mathrm{sh}}=\mathrm{round}(\bar k)$, $\bar k$ the mean of its $k_i$
(Figure~\ref{fig:pareto}; every row in Table~\ref{tab:app:matched}). On OPT-1.3b, the per-unit projector
forgets less at all six budgets from $9$ to $60$ dims/row, on $24$ of $27$ seed-pairs (sign-flip permutation
$p=1.5\times10^{-6}$; random-effects difference $-0.018$ nats, 95\% CI $[-0.028,-0.009]$), and Adam-NSCL
learns at most $0.003$ nats more. The gap widens from $1.1\times$ at $9$ dims/row to $4.3\times$ at $60$: the
shared subspace plateaus at $0.046$--$0.050$ nats between $15$ and $47$ dims/row (seeds $0$--$2$), while the per-unit projector
keeps falling, to $0.010$. The reading P1 suggests is that each added shared direction is spent in every row,
including on the $96\%$ of pairs that are closed, while each added per-unit direction lands on a token its
own unit fires on. OPT-125m's units fire on more tokens (a mean of $111$ of $2\,048$, against $72$), and as P4
predicts, the per-unit advantage needs a larger cap there: at $5.7$ dims/row, where the cap covers under a
tenth of the open pairs, the shared subspace forgets less, and the two meet at $17$. From $17$ upward
the per-unit projector forgets less on $14$ of $18$ seed-pairs (exact permutation $p=5\times10^{-4}$) and
learns more, by $0.023$ nats at $17$ dims/row and $0.008$ at $42$.

\input{fig_pareto}

\paragraph{Capacity.} Against Adam-NSCL at a uniform $k_{\mathrm{sh}}=259$ per row (the mean dimension GPM's
energy criterion selects at $0.95$), our cap $256$ at $47.3$ dims/row forgets $2.3\times$ less ($3/3$,
$p=0.009$) and learns $0.052$ nats more, and cap $128$ learns more and forgets less on WikiText-103 ($5$ of
$6$ seeds) and on Python ($3$ of $3$), where every constrained method ends below the base old-probe NLL, so
forgetting less means improving the old probe more. Cap $256$ removes $95.9\%$ of unconstrained
forgetting on OPT-1.3b and $97.7\%$ on OPT-125m (Appendix~\ref{app:cl:capacity}).

\paragraph{Scale.} On OPT-6.7b, cap $128$ constrains $68.7$ dims/row and forgets the same as Adam-NSCL at $k_{\mathrm{sh}}=69$ ($0.0192$ against
$0.0184$, equivalent under TOST), and it learns $0.011$ nats more ($3/3$ seeds, paired $p=1.5\times10^{-4}$).
Both remove $85\%$ of the unconstrained forgetting of $0.126$.

\subsection{What the criterion reads (P2, P5)}\label{sec:necessity}

To isolate the criterion, we hold each unit's $k_i$ fixed (cap $64$, $\bar k=23.5$) and change only which
tokens fill it (Table~\ref{tab:necessity}). The unit's own open tokens beat rank-matched random tokens, a
sign-blind ranking by $|z-\theta|$ and the most closed tokens, on all $18$ of $18$ seed-pairs (permutation
$p=3.8\times10^{-6}$). The random-token contrast holds again on OPT-125m ($-0.0114$, $p=0.005$), on
Python ($-0.0224$, $p=0.0009$) and on OPT-6.7b at cap $128$ ($-0.0098$, $3/3$, $p=0.007$).

\begin{table}[t]\centering\footnotesize
\caption{\textbf{Identical per-unit dimension, different tokens.} OPT-1.3b, WikiText-103, $\bar k=23.5$, $n=6$
($n=3$ for random-within-open). Gate minus the other selector. Superiority: one-sided paired $t$, 95\% CI;
equivalence: TOST \citep{schuirmann1987tost} at $\delta=0.0068$ nats, $2.5\%$ of unconstrained forgetting,
90\% CI.}
\label{tab:necessity}
\begin{tabular}{lccc}
\toprule
other selector & $\Delta$ forget & CI & verdict \\
\midrule
random tokens (rank-matched) & $-0.0317$ & $[-0.0436,-0.0197]$ & gate better, $p=0.0005$ \\
sign-blind $|z-\theta|$      & $-0.0575$ & $[-0.0756,-0.0393]$ & gate better, $p=0.0002$ \\
anti-gate (most closed)      & $-0.0660$ & $[-0.0799,-0.0520]$ & gate better, $p<10^{-4}$ \\
\midrule
first-order $|\partial L/\partial z|$ & $+0.0024$ & $[-0.0011,+0.0059]$ & equivalent \\
random within the open set   & $+0.0015$ & $[-0.0021,+0.0051]$ & equivalent \\
\bottomrule
\end{tabular}
\end{table}

The two equivalences are the ones P2 names. Random tokens drawn \emph{within} each unit's open set protect
as well as its most active ones, although the two sets differ on $30\%$ of pairs and cover $35.2\%$ against
$59.6\%$ of the active energy. The criterion is therefore the partition itself, not a ranking inside it. The
first-order criterion agrees because it inherits the partition (Proposition~\ref{prop:firstorder}): $99.6\%$
of the pairs $|\partial L/\partial z|$ selects are open. The gate reads the same set off one forward pass,
with no gradient. As P5 predicts, protecting whole units is far weaker: freezing the \texttt{fc1} rows of the
units with the largest old-support activity, at the same $23.5$ dims/row, forgets $0.210$ against the
per-unit projector's $0.025$ ($6/6$, $p<10^{-4}$) and learns $0.039$ nats less. Freezing enough rows to
hold the same $59.6\%$ of the active energy costs $69$ dims/row and still forgets $0.172$
(Appendix~\ref{app:cl:rowfreeze}).

\subsection{Against the published method families}\label{sec:families}

Figure~\ref{fig:families} places the per-unit projector among the replay-free families on the
forgetting--learning plane, each swept over its own hyperparameter on \texttt{fc1}. EWC \citep{kirkpatrick2017ewc},
whose diagonal Fisher reads open pairs only (Proposition~\ref{prop:firstorder}), is dominated from $\lambda=4$
to $4\times10^4$ and reaches the per-unit frontier at $\lambda=4\times10^5$ ($0.0175$, between our caps $128$
and $256$; no paired difference significant). MIGU \citep{du2024migu}, which masks rows by output magnitude
on the new batch, forgets $2\times$ more than our cap $128$ at its default mask ratio $0.7$ and learns $0.051$ nats less ($p=0.0005$).
LoRA \citep{hu2022lora}, to which O-LoRA \citep{wang2023olora} reduces at the first task, never comes within
$5\times$ of our forgetting; the four configurations that learn significantly more than ours, by
$0.035$--$0.042$ nats, forget $12$--$16\times$ more. LoRA-Null \citep{tang2026loranull}, which initialises the
adapter in the null space of the old inputs, cuts LoRA's forgetting by $2.3$--$3\times$ and still forgets
$3.6$--$5.1\times$ more than our cap $128$ (dominated at rank $32$). L2-to-init, EWC with a unit Fisher, forgets
at $\lambda=0.3$ as little as our cap $256$ (TOST-equivalent) and learns $0.049$ nats less ($3/3$, $p=0.02$).
Adam-NSCL with its dimension set by GPM's energy criterion (mean $259$ per row) learns $0.014$ nats more than
our cap $128$ at similar forgetting, which the per-unit projector reaches with $1/7.5$ of the
dimensions and halves with $1/5.5$ ($-0.0117$, $3/3$), while learning $0.016$ nats less
(Appendix~\ref{app:cl:families}). Post-hoc interpolation toward the base, WiSE-FT \citep{wortsman2022wiseft},
composes with both projectors: interpolated, every constrained run learns more at zero forgetting than the
interpolated unconstrained run ($0.027$--$0.043$ nats, $3/3$ seeds for each method; Appendix~\ref{app:cl:wiseft}).

\input{fig_families}

\paragraph{Downstream accuracy.}\label{sec:beyond} The protection measured in old-probe NLL carries over to zero-shot accuracy on LAMBADA, HellaSwag,
ARC-easy, Winogrande and SciQ \citep{paperno2016lambada,zellers2019hellaswag,clark2018arc,sakaguchi2020winogrande,welbl2017sciq}
(Table~\ref{tab:zs}; LM Evaluation Harness, \citealp{gao2023lmeval}). Unconstrained fine-tuning on
WikiText-103 costs OPT-1.3b $2.94$ points of mean accuracy, and our cap $128$ costs $1.02$ ($3/3$ seeds, $p=0.013$); at matched budget the two
projectors agree to within $0.1$ points.

\section{Second application: repair after pruning}\label{sec:repair}

One-shot pruning zeroes a fraction of each weight matrix under a mask, such as Wanda's saliency mask
\citep{sun2024wanda}, and SparseGPT \citep{frantar2023sparsegpt} then \emph{repairs} the surviving weights so that the
layer reproduces its old output. We repair only the wall matrices (\texttt{fc1}, \texttt{gate\_proj}); attention
and the output matrix stay mask-only (Appendix~\ref{app:setup}). We call $\Esc(\hat w):=\sum_{x\in\mathcal X\cap\mathcal C_\le}\max(z_{\hat w}(x)-\theta,0)^2$, the squared
amount by which a repaired row's closed calibration tokens overshoot the threshold, its \emph{escape} from
$\PX$: it is zero exactly when every inactive constraint holds, and at $\theta=0$ it equals the
closed$\to$open term of Eq.~\eqref{eq:identity}. Here a row is $w\in\mathbb R^d$ with its bias fixed,
$\mathcal X$ is the calibration multiset, $\PX$ is the slice of Definition~\ref{def:poly} at that bias, $M$ is
the mask with kept coordinates $K$, $\QM$ is the set of masked $\hat w$ that satisfy every inactive
constraint, and $\zeta=(\langle w_{K^c},x_{K^c}\rangle)_{x\in\mathcal X\cap\mathcal C_>}$ is the mask's residual.

\begin{lemma}[Curvature blindness]\label{lem:blind}
Assume no calibration token lies on the wall, $z_w(x)\neq\theta$ for all $x\in\mathcal X$. Then $\post=\gated$
on a neighbourhood of $w$, where the gate-weighted objective
$\gated(\hat w)=\sum_{x\in\mathcal X\cap\mathcal C_>}(z_{\hat w}(x)-z_w(x))^2$
is the in-cell term evaluated as if no token changed cell. Hence every local model of $\post$ at $w$ built
from its derivatives, of any order (the second-order Taylor model, the Gauss--Newton matrix, the Fisher of a
Gaussian output model), is the corresponding model of $\gated$; both crossing terms vanish identically on the
neighbourhood.
\end{lemma}

\begin{theorem}[Polyhedral escape]\label{thm:escape}
Write $X_K=[A;\tilde A]$ for the kept columns of the calibration inputs, split into active and inactive rows.
\textbf{(i)} $\gated$ depends on the active tokens alone, through $A$ and the residual $\zeta$. \textbf{(ii)} If
$\ker A\not\subseteq\ker\tilde A$, in particular if $X_K$ has full column rank and there are fewer active
tokens than kept columns, the minimiser set of $\gated$ is an affine subspace on which $\gated$ is constant and
the escape is unbounded above. \textbf{(iii)} For $n$ i.i.d.\ calibration tokens with finite second moments, a
fixed mask and a non-singular active second moment on the kept coordinates, the minimiser is almost surely
unique for large $n$ and converges to a point $w^\star$ the active tokens determine, and its per-token escape
converges to $e^\star=\mathbb E\big[\mathds 1[z_w(x)\le\theta]\,(z_{w^\star}(x)-\theta)_+^2\big]$, which is
positive if and only if $\Pr[z_w(x)\le\theta<z_{w^\star}(x)]>0$.
\end{theorem}

\begin{theorem}[Confinement by a one-sided penalty]\label{thm:confine}
For $\lambda>0$, $\gatedh(\hat w,\lambda)=\gated(\hat w)+\lambda\,\Esc(\hat w)$ over masked $\hat w$ is convex
and attains its minimum; its penalty is zero exactly on $\QM\supseteq\PX\cap\{\hat w=\hat w\odot M\}$; and every
$\hat w$ with $\gatedh(\hat w,\lambda)\le\gatedh(w\odot M,\lambda)$, every minimiser in particular, has escape
at most $\|\zeta\|^2/\lambda+E_0$, where $\|\zeta\|^2=\gated(w\odot M)$ and $E_0=\Esc(w\odot M)$ is the escape of the
mask alone. The escape is the same at every minimiser, is non-increasing in $\lambda$, and tends to $0$ as
$\lambda\to\infty$ if and only if $\QM\neq\varnothing$.
\end{theorem}

Proofs are in Appendices~\ref{app:repair} and~\ref{app:proof} (Remark~\ref{rem:fatconvex} covers the
implemented penalty at $\theta>0$). The bias witness $b\le\theta$ certifies $\QM\neq\varnothing$ on $81$--$96\%$
of the wall units measured (Remark~\ref{rem:E0}). Solved in closed form in \texttt{fp64},
$\gated$ is $1.94$ nats worse than no repair (OPT-1.3b, sparsity $0.7$). Early-stopped, its crossings
run closed$\to$open ($1.15\%$ against $0.18\%$ open$\to$closed), as its unpriced inactive block
predicts; the pre-activation objective $\pre$ runs the other way ($0.52\%$ against $0.34\%$;
Appendix~\ref{sec:signature}). Pricing both crossings orders the objectives: $\post<\pre$ on $46$ of $47$
probe cells, $15/15$ on WikiText-2 \citep{merity2017pointer} and $20/20$ on two 7B ReLU-family models, and a
price $\lambda>1$ beats $\lambda=1$ on ReLU, and on GELU with the exact objective (Pythia,
\citealp{biderman2023pythia}; Appendix~\ref{app:lambda}).

\section{Related work}\label{sec:related}

\paragraph{Projection-based continual learning.} GPM \citep{saha2021gpm}, Adam-NSCL \citep{wang2021adamnscl}
and AdNS \citep{kong2022adns} confine each layer's update to the complement of a shared subspace of its old
inputs, OGD \citep{farajtabar2020ogd} does so in parameter space, later work relaxes the subspace
\citep{lin2022trgp,saha2023sgp} or builds a LoRA in its complement \citep{liang2024inflora}, and in language models it orients or initialises adapters
\citep{yang2024corda,tang2026loranull}; sparse activations enlarge the null space \citep{abbasi2022shd}. We
change which pairs are constrained.

\paragraph{Other continual-learning methods.} O-LoRA \citep{wang2023olora} keeps task adapters orthogonal, LoRA
learns less and forgets less \citep{biderman2024loraforgets}, and MIGU \citep{du2024migu}, HAT, AGS-CL
\citep{serra2018hat,jung2020agscl} and neuron-level editors \citep{jiang2025nse,pan2025fine} decide per unit
whether it may move; EWC \citep{kirkpatrick2017ewc}, LDIFS \citep{mukhoti2024ldifs} and replay
\citep{ibrahim2024continual} complete the set.

\paragraph{Activation patterns and repair.} Layer-wise repair after pruning or quantisation
\citep{frantar2023sparsegpt,frantar2023gptq,li2021brecq,kim2025guidedquant} fits pre-activations or
curvature-weighted outputs. For piecewise-linear activations, the Hessian's diagonal blocks equal the
Gauss--Newton blocks \citep{botev2017practical}; Lemma~\ref{lem:blind} extends this to the objectives near $w$.
The inequalities of $\PX$ appear in provable repair
\citep{tao2023aprnn}, Net-Trim \citep{aghasi2017nettrim} and \citet{pilanci2020convex}, and our penalty is the
inactive-side term of the activation-boundary loss of \citet{heo2019boundary} at zero margin.

\paragraph{Tropical geometry of ReLU networks.} ReLU networks are tropical rational maps
\citep{zhang2018tropical,maragos2021tropical} whose number of linear regions \citet{montufar2014linear}
bound; tropical compression approximates their Newton polytopes from the weights
\citep{smyrnis2020multiclass,misiakos2022tropical}. We use the data-space subdivision and its weight-space dual.

\section{Conclusion}\label{sec:conclusion}

An exact identity joins a ReLU layer's two dual pictures. In fine-tuning, it names the per-unit constraint;
in pruning repair, it names the block that curvature cannot see.

\subsection*{AI use statement}
The authors designed the study and every experiment. We used a large language model assistant (Claude,
Anthropic) for the following. \emph{Implementation and checking:} it wrote and debugged the code for the
method, the baselines, the experiments and the analyses, and it cross-checked reported numbers against
the recorded runs, derivations against their statements, and citations against their sources.
\emph{Retrieval and discovery:} it searched for and summarised related work. \emph{Writing:} it drafted
sections of the paper and helped edit and polish the text. The authors take full responsibility for all
content.

\subsection*{Reproducibility statement}
The code release that accompanies the paper contains the implementation of the gate-aware per-unit projector (Definition~\ref{def:proj}), of the shared-subspace baseline, of the decomposition of Eq.~\eqref{eq:identity} and of the repair objectives of \S\ref{sec:repair}, the training and evaluation code, the per-run histories, the scripts that regenerate every table and figure, a CPU self-test of Eq.~\eqref{eq:identity}, Definition~\ref{def:proj}, Proposition~\ref{prop:zero} and Theorem~\ref{thm:confine} on random layers, and numerical checks of the statements of \S\S\ref{sec:identity}, \ref{sec:method} and~\ref{sec:repair}. Every run logs its fully resolved arguments, and every
fine-tuning run checks Eq.~\eqref{eq:identity} against its own activations. Models, corpora, training settings and
statistics are given in Appendices~\ref{app:cl:setup} and~\ref{app:setup}.

\bibliographystyle{iclr2027_conference}
\bibliography{references,references_cl}

\appendix
\input{app_prelim}\FloatBarrier

\input{app_cl_setup}\FloatBarrier

\input{app_cl_budget}\FloatBarrier

\input{app_cl_mechanism}\FloatBarrier
\input{app_repair}\FloatBarrier
\input{app_theory}\FloatBarrier

\input{app_proofs23}
\end{document}

%% file: math_commands.tex
\usepackage{amsmath,amsfonts,bm}

\def\eqref#1{equation~\ref{#1}}

\def\1{\bm{1}}

\DeclareMathAlphabet{\mathsfit}{\encodingdefault}{\sfdefault}{m}{sl}
\SetMathAlphabet{\mathsfit}{bold}{\encodingdefault}{\sfdefault}{bx}{n}

%% file: counts.tex
\newcommand{\PostPreProbe}{$46$ of $47$}   %

\newcommand{\PostPreWT}{$15$ of $15$}   %

\newcommand{\PostPreConfig}{$23$ of $24$}   %

\newcommand{\LamWin}{$9$ of $11$}   %

\newcommand{\LamWinFloor}{$7$ of $11$}   %

\newcommand{\LamRows}{$5$ of $9$}   %

\newcommand{\LamSeedTwo}{$6$ of $6$}   %

\newcommand{\SevenB}{$20$ of $20$}   %

\newcommand{\SevenBFloorVal}{$0.0061$}
\newcommand{\NoiseFloor}{$0.0068$}
\newcommand{\WTFloorVal}{$0.0150$}
\newcommand{\WTFloorMax}{$0.0284$}

%% file: authors.tex
\author{\normalfont
\begin{tabular}[t]{@{}l@{\hspace{4em}}l@{}}
\textbf{Yuyang Zhang} & \textbf{Xiaoyin Chen} \\
McGill University & Mila -- Quebec AI Institute \\
\texttt{yuyang.zhang@mail.mcgill.ca} & Universit\'e de Montr\'eal \\
 & \texttt{xiaoyin.chen@mila.quebec} \\[2ex]
\textbf{Chunlin Ren} & \textbf{Qihuang Zhang} \\
Nanyang Technological University & McGill University \\
\texttt{chunlin001@e.ntu.edu.sg} & \texttt{qihuang.zhang@mcgill.ca}
\end{tabular}}

%% file: sec_theory.tex
\section{Two dual pictures of a ReLU layer}\label{sec:identity}

A ReLU layer can be read in two spaces (Figure~\ref{fig:duality}). In \emph{data space} its walls tile the
inputs into cells, one per gate pattern, dual to the upper vertices of a zonotope. In \emph{weight space} each unit
is a point and each old token a hyperplane, and the tokens cut out a polyhedron, the closure of the weights
that keep every token on its side. Fine-tuning moves the point: inside it the output changes only on the
tokens the unit fires on, and crossing a hyperplane switches one token's gate. The identity below makes the
two kinds of damage exact; we use it to decide which motion to forbid, and which to price.

\input{fig1_tropical}

\paragraph{Notation.} A layer has $m$ units on inputs $x\in\mathbb R^d$. We write $\tilde x=(x,1)$ and
$\tilde w_i=(w_i,b_i)\in\mathbb R^{d+1}$, so $z_i(x)=\langle\tilde w_i,\tilde x\rangle$ carries the bias as the
weight of a ones column. The old tokens $\mathcal X=\{x_t\}_{t=1}^T$ (a multiset) give
$\tilde X\in\mathbb R^{T\times(d+1)}$ with rows $\tilde x_t$; $\tilde X_S$ keeps the rows indexed by $S$.

\begin{definition}[Wall and cells]\label{def:wall}
Let $\act_\theta(u)=u\cdot\mathds 1[u>\theta]$ with $\theta\ge0$. The \textbf{wall} of a unit with weights
$\tilde w$ is $\mathcal W_{\tilde w}=\{x:z_{\tilde w}(x)=\theta\}$. It splits $\mathbb R^d$ into the
\textbf{open cell} $\mathcal C_>=\{z_{\tilde w}>\theta\}$ and the \textbf{closed cell}
$\mathcal C_\le=\{z_{\tilde w}\le\theta\}$, on each of which the unit is affine. Unit $i$'s \textbf{open set}
is $O_i=\{t:z_i(x_t)>\theta\}$, and $G(x)=\{i:z_i(x)>\theta\}$ is the \textbf{gate pattern} of $x$.
\end{definition}

\paragraph{Data space.} In the max-plus semiring ($a\oplus b=\max(a,b)$, $a\odot b=a+b$) a ReLU unit is the
tropical binomial $0\oplus z$, whose wall is the tropical hypersurface on which its two terms tie
\citep{zhang2018tropical,maclagan2015tropical}. The sum of the units is their tropical product, with one term
per set $J$ of units, and its lifted Newton polytope is a zonotope. Write
$\tilde w^\theta_J=\sum_{i\in J}(w_i,b_i-\theta)$ and $Z_\theta=\sum_{i=1}^m[0,(w_i,b_i-\theta)]$.

\begin{proposition}[Tropical form and the dual subdivision]\label{prop:tropical}
\textbf{(i)} $\sum_i\max(0,z_i(x)-\theta)=\max_{J\subseteq[m]}\langle\tilde w^\theta_J,\tilde x\rangle$, and
$J$ attains the maximum iff $G(x)\subseteq J\subseteq\{i:z_i(x)\ge\theta\}$; so $G(x)$ always attains it,
uniquely iff $x$ lies on no wall. \textbf{(ii)} $J\mapsto\tilde w^\theta_J$ maps the gate patterns realised
off the walls injectively into the vertices of $Z_\theta$ exposed by functionals $(y,1)$, and onto them when
no unit has $(w_i,b_i)=(0,\theta)$. \textbf{(iii)} If $x,x'$ lie on no wall and $G(x')=G(x)\sqcup\{i\}$, the
two vertices span an edge of $Z_\theta$ parallel to $(w_i,b_i-\theta)$, dual to a nonempty relatively open
piece of the wall $\mathcal W_i$.
\end{proposition}

At $\theta=0$ this is Figure~\ref{fig:duality}a--c; for $\theta>0$ the gate pattern is that of
$\max(0,z_i-\theta)$, so the proposition describes its cells verbatim. \paragraph{Weight space.} Fix one unit with base weights $\tilde w$. Each token defines the hyperplane
$H_t=\{\tilde v:\langle\tilde v,\tilde x_t\rangle=\theta\}$ of the unit's weight space ($\tilde x_t\neq0$),
and the pairing is symmetric: $x_t$ lies on the wall of $\tilde v$ exactly when $\tilde v$ lies on $H_t$.

\begin{definition}[Activation-preserving polyhedron]\label{def:poly}
With $O$ the base open set of the unit on $\mathcal X$, let
\begin{equation}\label{eq:ineq}
\PX=\big\{\tilde v:\ \langle\tilde v,\tilde x_t\rangle\ge\theta\ \ \forall t\in O,\ \ \
\langle\tilde v,\tilde x_t\rangle\le\theta\ \ \forall t\notin O\big\},
\end{equation}
and let the \textbf{pattern cell} $\KX\subseteq\PX$ be the weights that give every token its base gate bit.
The first group of constraints is the \textbf{active block}, the second the \textbf{inactive block}.
\end{definition}

\begin{proposition}[Data--weight duality]\label{prop:duality}
\textbf{(a)} Token $t$ is open at $\tilde v$ iff $\tilde v\in H_t^>=\{\langle\cdot,\tilde x_t\rangle>\theta\}$.
It changes cell between $\tilde w$ and $\hat w$ iff exactly one of them lies in $H_t^>$, closing if that one
is $\tilde w$, and along the segment $[\tilde w,\hat w]$ it changes cell at most once. \textbf{(b)}
$\KX=\PX\setminus\bigcup_{t\in O}H_t$ is convex, its interior is a chamber of the arrangement $\{H_t\}$, and
$\overline{\KX}=\PX$. \textbf{(c)} $\PX$ is full-dimensional, and each of its facets is $\PX\cap H_t$ for a
token $t$, an active facet if $t\in O$ and an inactive one if not.
\end{proposition}

Part (a) is Figure~\ref{fig:duality}d--e: moving a unit's wall across a token in data space and moving its
weight point across that token's hyperplane are one event. Appendix~\ref{app:proofs23} proves all of
\S\S\ref{sec:identity}--\ref{sec:method} except Theorem~\ref{thm:decomp}.

\begin{theorem}[Exact decomposition by cell membership]\label{thm:decomp}
For any $\hat w\in\mathbb R^{d+1}$, any $\theta\ge0$ and finite $\mathcal X$, with
$\post:=\sum_{x\in\mathcal X}(\act_\theta(z_{\hat w}(x))-\act_\theta(z_{\tilde w}(x)))^2$,
\begin{equation}\label{eq:identity}
\post=\underbrace{\sum_{x\in S_{++}}\!\big(z_{\hat w}(x)-z_{\tilde w}(x)\big)^2}_{\text{in-cell}}
+\underbrace{\sum_{x\in S_{+-}}\! z_{\tilde w}(x)^2}_{\text{open}\to\text{closed}}
+\underbrace{\sum_{x\in S_{-+}}\! z_{\hat w}(x)^2}_{\text{closed}\to\text{open}},
\end{equation}
where $S_{++}=\mathcal X\cap\mathcal C_>\cap\hat{\mathcal C}_>$, $S_{+-}=\mathcal X\cap\mathcal C_>\cap\hat{\mathcal C}_\le$
and $S_{-+}=\mathcal X\cap\mathcal C_\le\cap\hat{\mathcal C}_>$, which with $S_{--}$ partition $\mathcal X$.
\end{theorem}

\begin{proof}
Each $x$ lies in exactly one of the four intersections. On $S_{++}$ both activations equal their
pre-activations; on $S_{+-}$ the new one is $0$; on $S_{-+}$ the old one is $0$; on $S_{--}$ both are $0$.
Summing the squared differences over the four sets gives Eq.~\eqref{eq:identity}.
\end{proof}

Summed over units, the four sets become sets of (token, unit) pairs $(t,i)$. The pattern-changing terms are
exact sums over pair sets fixed by the two weight points, where convergence analyses bound them as a
remainder \citep{du2019gradient}. Measured, the pairwise case split behind it holds to a relative gap of at most
$5.8\times10^{-5}$ over $27$ pruned OPT cells and $9.0\times10^{-5}$ over $18$ cells on two 7B ReLU-family
models \citep{song2025prosparse,sparsellm2023relullamafix} (Appendix~\ref{app:identity}).

\begin{corollary}[Weight-space reading]\label{cor:weight}
$S_{+-}\cup S_{-+}$ are the tokens whose hyperplanes separate $\tilde w$ from $\hat w$ in the sense of
Proposition~\ref{prop:duality}a, and $S_{+-}$ those with $\tilde w$ on the open side. $\post$ is quadratic on every set of weights with a fixed gate pattern on
$\mathcal X$; on $\KX$ it equals $\|\tilde X_O(\hat w-\tilde w)\|^2$, and at $\theta=0$ it does on all of $\PX$.
\end{corollary}

Moving inside $\KX$ costs only the in-cell term, and each separating token adds exactly one crossing term.
The two applications act on the two blocks: knowledge-preserving fine-tuning holds the active block on the
tokens it covers (exactly, or to within the ridge residual of Proposition~\ref{prop:zero}a, at the base inputs) and leaves the inactive block free (\S\ref{sec:method}); in pruning repair the curvature
objectives leave the inactive block unpriced, and pricing it bounds the escape (\S\ref{sec:repair}).

\section{The constraint the identity names, and what it predicts}\label{sec:method}

By Theorem~\ref{thm:decomp} a pair closed at base enters only the closed$\to$open term, and contributes zero
while it stays closed. At the base inputs the in-cell and open$\to$closed terms of unit $i$ are carried by
$O_i$ alone, and both vanish when row $i$'s update satisfies $\tilde X_{O_i}u_i=0$: one equality per open
token, and which tokens are open is a property of the unit, not of the layer.

\begin{definition}[Gate-aware per-unit projector]\label{def:proj}
Let $\tilde X$ hold the base model's inputs to the layer on the $T$-token \emph{old support}. For unit $i$ let
$h^{(i)}_t=\act_\theta(z_i(x_t))$, let the \emph{cap} $C$ bound each unit's budget, let $S_i\subseteq O_i$ index
the $C$ largest positive $h^{(i)}_t$ (ties broken by token index), and
\begin{equation}\label{eq:ki}
k_i=|S_i|=\min\bigl(C,\,|O_i|\bigr),\qquad A_i=\tilde X_{S_i}\in\mathbb R^{k_i\times(d+1)}.
\end{equation}
After each optimiser step the cumulative displacement $u_i=(dW_i,db_i)$ of row $i$ is replaced by
\begin{equation}\label{eq:proj}
u_i\ \leftarrow\ P_i(\rho_i)\,u_i=u_i-A_i^\top(A_iA_i^\top+\rho_iI)^{-1}A_iu_i ,
\end{equation}
with $P_i(0):=I-A_i^+A_i$, its limit as $\rho_i\to0$, and $P_i(\rho)=I$ when $k_i=0$.
\end{definition}

Everything in Eq.~\eqref{eq:proj} is built once from the base model, with $\rho_i=10^{-4}$ times the mean
diagonal of $A_iA_i^\top$ padded to $C\times C$ by the identity (Appendix~\ref{app:cl:setup}). At $\rho=0$ row $i$ loses $\operatorname{rank}A_i\le k_i$ directions (at $\rho>0$ they are shrunk,
not removed), so the reported dims/row, the mean $k_i$, bounds the directions removed from above. On OPT-1.3b at caps $16$ to $128$ the residual on the constrained
pairs, at the stored base inputs, is $5.4\times10^{-9}$ to $4.0\times10^{-8}$ of the layer's total
$|\Delta z|$ mass. All \texttt{fc1} layers train at once, so end-to-end numbers also include the drift of each
layer's inputs; the statements below concern one layer at its base inputs.

\begin{proposition}[Exactness]\label{prop:zero}
Let $u'=P_i(\rho)u$. \textbf{(a)} $P_i(0)$ is the orthogonal projector onto $\ker A_i$. For $\rho>0$ and $k_i\ge1$,
$\|A_iu'\|\le\frac{\rho}{\sigma_+^2+\rho}\|A_iu\|$, tightly, where $\sigma_+$ is the smallest positive singular
value of $A_i$. \textbf{(b)} At $\rho=0$, every pair in $S_i$ stays in $S_{++}$ and contributes zero to all
three terms, every pair closed at base contributes only to closed$\to$open, and unit $i$'s in-cell and
open$\to$closed terms are sums over $O_i\setminus S_i$, empty when $C\ge|O_i|$. \textbf{(c)} If $C\ge|O_i|$
and $\rho=0$, the flat $\tilde w_i+\ker\tilde X_{O_i}$ lies in $\bigcap_{t\in O_i}H_t^>$: the weight point
moves parallel to every active hyperplane, and can leave $\KX$ only through inactive facets
(Figure~\ref{fig:duality}f).
\end{proposition}

\begin{theorem}[Budget separation]\label{thm:budget}
Let $r_i=\operatorname{rank}\tilde X_{O_i}$, and call a subspace $U$ of row updates \emph{protecting} for unit $i$
if unit $i$'s in-cell term vanishes for every update in $U$. \textbf{(a)} $U$ is protecting for unit $i$ iff
$U\subseteq\ker\tilde X_{O_i}$, and then the open$\to$closed term vanishes on $U$ too; so protection costs at
least $r_i\le|O_i|$ directions, and with $C\ge|O_i|$ the projector $P_i(0)$ removes exactly $r_i$.
\textbf{(b)} One subspace shared by every row is protecting for all units iff it lies in
$\ker\tilde X_{O_\cup}$, $O_\cup=\bigcup_iO_i$, so it removes at least
$\operatorname{rank}\tilde X_{O_\cup}\ge\max_ir_i$ directions from every row. \textbf{(c)} With $B$
directions in total, per-unit subspaces protect any set $J$ of units with $\sum_{i\in J}r_i\le B$, while
freezing whole rows protects $\#\{i:r_i=0\}+\min(\lfloor B/(d+1)\rfloor,\#\{i:r_i>0\})$ units at best.
\end{theorem}

In OPT-1.3b with $T=2\,048$ old tokens and $d+1=2\,049$, the median unit fires on $11$ tokens and the mean on
$72.1$, so $r_i\le11$ for at least half the units and $19.5\%$ never fire. Caps $64$ and $128$ give mean
$k_i$ of $23.5$ and $34.4$ and cover $59.6\%$ and $71.7\%$ of the active post-activation energy. Adam-NSCL
instead forbids the top $k_{\mathrm{sh}}$ eigendirections of $\tilde X^\top\tilde X/T$ in every row.

\begin{proposition}[Budget-limited optimum]\label{prop:kyfan}
Let $\theta=0$ and let each row's proposed update $u_i$ satisfy $\mathbb E[u_iu_i^\top]=\sigma^2I$, projected
onto $U_i$ of codimension $q\le d+1$. Unit $i$'s in-cell plus open$\to$closed terms are at most
$\|\tilde X_{O_i}P_{U_i}u_i\|^2$, whose expectation $D_i(U_i)=\sigma^2\operatorname{tr}(\Xi_iP_{U_i})$,
$\Xi_i=\tilde X_{O_i}^\top\tilde X_{O_i}$, has minimum $\sigma^2\sum_{j>q}\lambda_j(\Xi_i)$, attained when
$U_i^\perp$ is a top-$q$ eigenspace of the unit's own open-token Gram $\Xi_i$. One subspace shared by all rows is
best when $U^\perp$ is a top-$q$ eigenspace of $\sum_i\Xi_i=\sum_tn_t\tilde x_t\tilde x_t^\top$, $n_t$ the number of
units open on token $t$, and its total $\sum_iD_i$ is never below that of the best per-unit subspaces of the
same total codimension.
\end{proposition}

\begin{proposition}[The freed term is a near-wall term]\label{prop:freed}
For a closed token $t$ of unit $i$ and an update $u$, let $\gamma_t=\theta-z_i(x_t)\ge0$ and
$a_t=\langle u,\tilde x_t\rangle$. Token $t$ opens iff $a_t>\gamma_t$, so only the near-wall tokens
$N(u)=\{t\notin O_i:\gamma_t<\|u\|\,\|\tilde x_t\|\}$ can open, and at $\theta=0$ the closed$\to$open term is
$\sum_{t\notin O_i}(a_t-\gamma_t)_+^2\le\|\tilde X_{N(u)}u\|^2$.
\end{proposition}

\begin{proposition}[First-order criteria inherit the partition]\label{prop:firstorder}
Let the inputs be fixed and a loss $L=g(h)$ be differentiable in the post-activations
$h=(h_{t,i})$, $h_{t,i}=\act_\theta(z_{t,i})$ (any downstream network), with $\act_\theta'(\theta):=0$ as in automatic
differentiation. Then $\partial L/\partial z_{t,i}=\mathds 1[z_{t,i}>\theta]\,\partial g/\partial h_{t,i}$
vanishes on every closed pair; the row gradient
$\nabla_{\tilde w_i}L=\sum_{t\in O_i}(\partial L/\partial z_{t,i})\,\tilde x_t$ lies in the row space of
$\tilde X_{O_i}$; and the Fisher of row $i$ annihilates $\ker\tilde X_{O_i}$ (with $O_i$ taken over the
Fisher data), while every term of EWC's diagonal Fisher is indexed by an open pair.
\end{proposition}

\paragraph{Which matrices carry the relaxation.} On an ungated matrix every pair acts as open, so the exact
per-unit constraint of every row is the shared null space $\ker\tilde X$, which Adam-NSCL approximates; the
two exact constraints coincide there and differ only on a matrix whose outputs each pass through a threshold,
in OPT \texttt{fc1}, where we compare the methods. In a
ReLU-gated GLU the same constraint on the gate and up rows keeps every open pair's output
(Lemma~\ref{lem:glu}).

\paragraph{Predictions.}\label{sec:predictions} The results fix what \S\ref{sec:exp} measures.
\textbf{P1} (Theorem~\ref{thm:budget}, Proposition~\ref{prop:kyfan}): exact protection costs a unit $r_i$
directions of its own row but a shared subspace $\operatorname{rank}\tilde X_{O_\cup}$ in every row, and in the
damage bound $D_i$ the best per-unit constraint of each codimension is never worse than the best shared one.
The gate projector attains $D_i=0$ once its tokens span the row space of $\tilde X_{O_i}$; below that the
theory does not rank it against Adam-NSCL, and we predict that at matched dims/row it forgets less once the
cap covers most of each unit's open-token span (\S\ref{sec:budget}).
\textbf{P2} (Propositions~\ref{prop:zero} and~\ref{prop:firstorder}): at the base inputs a covered open pair
is held exactly and a covered closed pair only against opening, so we predict that at fixed $k_i$ sets chosen
within $O_i$ protect comparably, sets chosen off it lose, and a first-order criterion, which stays inside $O_i$ whenever at least $k_i$ open pairs have
$\partial g/\partial h\neq0$, should match the gate; among the published
families EWC is the one whose criterion reads it (\S\S\ref{sec:necessity},~\ref{sec:families}).
\textbf{P3} (Proposition~\ref{prop:freed}, $\theta=0$): at the base inputs only near-wall closed tokens $N(u)$
can open, and their term is at most $\|\tilde X_{N(u)}u\|^2$, so we predict that a constrained update frees far
less closed$\to$open mass than an unconstrained one. Measured end to end on the probe it is $1\,785$ at
$C=64$ and $1\,622$ for the shared subspace of the same size, against $9\,239$ (seeds $0$--$2$,
Appendix~\ref{app:cl:crossmass}); the theory does not order the two constrained methods.
\textbf{P4} (Proposition~\ref{prop:zero}b): exactness is reached unit by unit at $C\ge|O_i|$, so where open
sets are larger the per-unit advantage should need a larger cap (\S\ref{sec:budget}).
\textbf{P5} (Theorem~\ref{thm:budget}c): a frozen row spends $d+1$ directions on one unit, so protecting
whole units should be weaker at equal budget (\S\ref{sec:necessity}).
\textbf{P6} (Lemma~\ref{lem:blind}, Theorems~\ref{thm:escape}--\ref{thm:confine}): every derivative-based local model of the repair
error at the dense weights sees open pairs only, so its minimisers should violate inactive constraints of $\PX$, and pricing them should bound
the escape (\S\ref{sec:repair}).

%% file: fig1_tropical.tex
\begin{figure}[t]
\centering
\includegraphics[width=\textwidth]{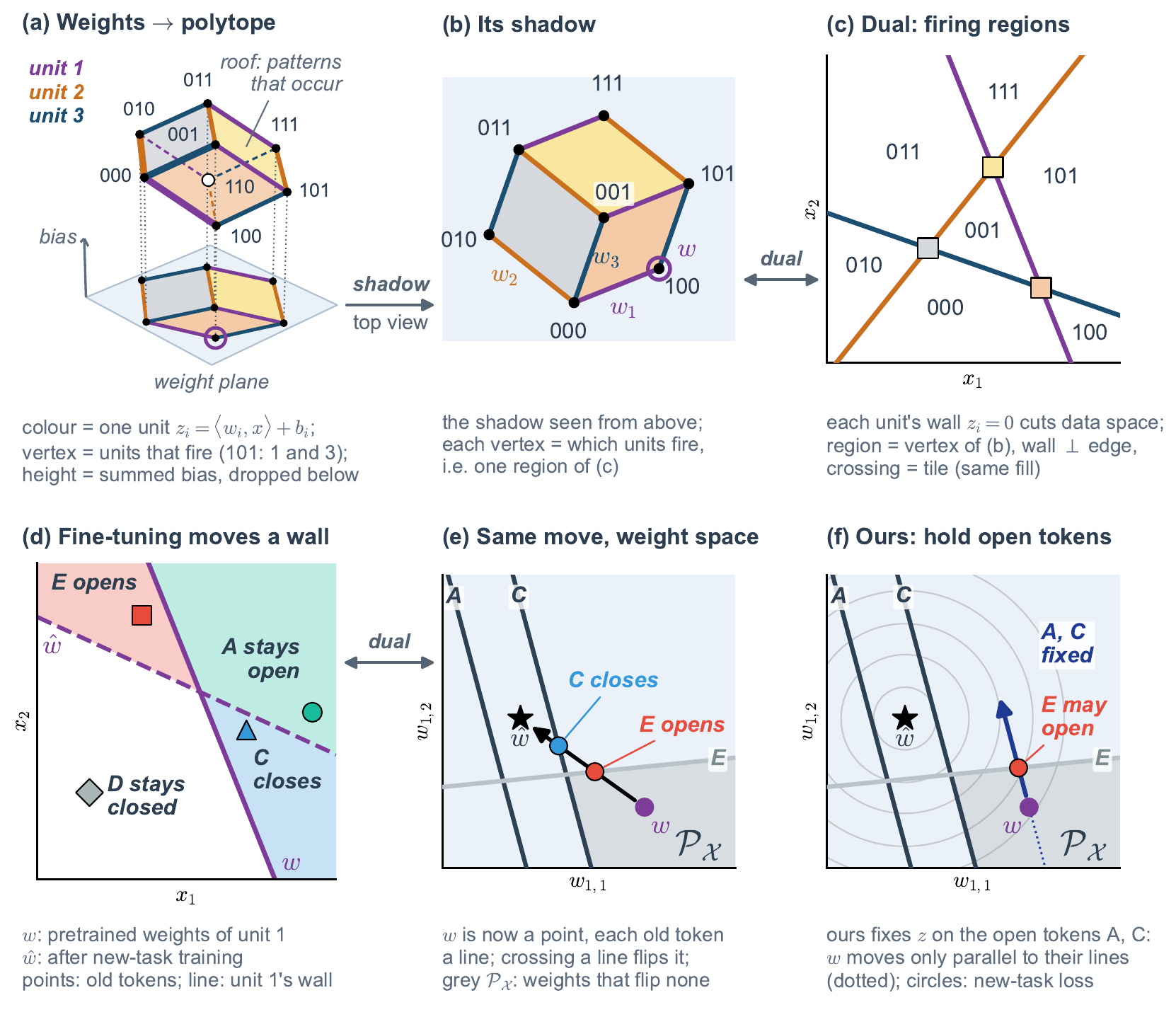}
\caption{\textbf{A layer's weights, its lifted Newton polytope, and its activation pattern}, for a toy layer
of three ReLU units on two inputs. Pale blue marks weight space and (c), (d) are data space; arrows marked
\emph{dual} join the two. \textbf{(a)} Each colour is a unit and each vertex a set of units that fire, in bits
($101$: units~1 and~3), placed at their summed weights and raised by their summed biases: the Newton polytope
of $\sum_i\max(0,z_i)$ lifted by the biases \citep{zhang2018tropical} (Appendix~\ref{app:prelim}).
\textbf{(b)} Its shadow, which subdivides the Newton polytope; the circled vertex is unit~1's weights $w$.
\textbf{(c)} The units' walls $z_i=0$ in data space. Regions are labelled by which units fire, each is a vertex
of (b), and pattern $110$ (hollow in (a)) occurs nowhere. \textbf{(d)} Fine-tuning moves unit~1's wall from
$w$ to $\hat w$, and colours say what that does to each old token: green stays open, blue closes, red opens,
grey stays closed (the sets $S_{++},S_{+-},S_{-+},S_{--}$ of Theorem~\ref{thm:decomp}). \textbf{(e)} The same
move in weight space: $w$ is a point and each old token a line, crossing which flips the token (dark: open at
$w$; light: closed; $D$'s line is off-panel). The grey region $\PX$ is the closure of the weights that flip
no token (Definition~\ref{def:poly}). \textbf{(f)} Our
projector (\S\ref{sec:method}) fixes the pre-activation on the open tokens $A$ and $C$, so $w$ moves only
parallel to their lines (dotted) and the closed token $E$ may open; circles are the new-task loss. Biases are fixed in (d)--(f),
and $A$, $C$ lie on one ray, so one direction stays free.}
\label{fig:duality}
\end{figure}

%% file: fig_pareto.tex
\begin{figure}[t]
\centering
\begin{tikzpicture}[font=\small]
\begin{scope}[xshift=0.0cm]
\begin{axis}[width=6.6cm,height=4.1cm,xlabel={constrained dims per row},ylabel={forgetting (nats)},xmin=0,xmax=138,ymin=0,title={OPT-1.3b},legend style={font=\scriptsize,draw=none,fill=none,at={(0.97,0.97)},anchor=north east},tick label style={font=\scriptsize},label style={font=\scriptsize},title style={font=\small},scaled y ticks=false,yticklabel style={/pgf/number format/fixed,/pgf/number format/precision=2},grid=major,grid style={gray!18}]
\addplot+[thick,gray!70!black,mark=square*,mark size=1.3,error bars/.cd,y dir=both,y explicit] coordinates {(9.000,0.0717) +- (0,0.0066) (12.000,0.0626) +- (0,0.0073) (15.000,0.0500) +- (0,0.0061) (20.000,0.0502) +- (0,0.0065) (24.000,0.0468) +- (0,0.0072) (28.000,0.0496) +- (0,0.0065) (34.000,0.0462) +- (0,0.0039) (40.000,0.0466) +- (0,0.0052) (47.000,0.0467) +- (0,0.0054) (60.000,0.0423) +- (0,0.0064) (128.000,0.0272) +- (0,0.0055)};
\addlegendentry{Adam-NSCL (shared)}
\addplot+[thick,red!75!black,mark=*,mark size=1.5,error bars/.cd,y dir=both,y explicit] coordinates {(9.067,0.0604) +- (0,0.0079) (15.048,0.0473) +- (0,0.0060) (23.519,0.0339) +- (0,0.0058) (34.421,0.0238) +- (0,0.0058) (47.257,0.0110) +- (0,0.0035) (60.332,0.0098) +- (0,0.0044)};
\addlegendentry{ours (per-unit, gate)}
\addplot[dashed,gray!70!black,domain=0:138,samples=2,forget plot] {0.0255};
\node[font=\scriptsize,gray!40!black,anchor=north east] at (axis cs:116,0.0255) {Adam-NSCL, uniform 259};
\end{axis}
\end{scope}
\begin{scope}[xshift=7.2cm]
\begin{axis}[width=6.6cm,height=4.1cm,xlabel={constrained dims per row},ylabel={forgetting (nats)},xmin=0,xmax=69,ymin=0,title={OPT-125m},legend style={font=\scriptsize,draw=none,fill=none,at={(0.97,0.97)},anchor=north east},tick label style={font=\scriptsize},label style={font=\scriptsize},title style={font=\small},scaled y ticks=false,yticklabel style={/pgf/number format/fixed,/pgf/number format/precision=2},grid=major,grid style={gray!18}]
\addplot+[thick,gray!70!black,mark=square*,mark size=1.3,error bars/.cd,y dir=both,y explicit] coordinates {(6.000,0.0289) +- (0,0.0014) (8.000,0.0209) +- (0,0.0010) (16.000,0.0202) +- (0,0.0013) (17.000,0.0195) +- (0,0.0018) (32.000,0.0191) +- (0,0.0009) (42.000,0.0126) +- (0,0.0012) (62.000,0.0142) +- (0,0.0016) (64.000,0.0144) +- (0,0.0013)};
\addlegendentry{Adam-NSCL (shared)}
\addplot+[thick,red!75!black,mark=*,mark size=1.5,error bars/.cd,y dir=both,y explicit] coordinates {(5.700,0.0637) +- (0,0.0030) (17.022,0.0224) +- (0,0.0011) (42.080,0.0086) +- (0,0.0012) (62.047,0.0046) +- (0,0.0012)};
\addlegendentry{ours (per-unit, gate)}
\end{axis}
\end{scope}
\end{tikzpicture}
\caption{\textbf{Forgetting at matched constrained dimensions.} Means over seeds 0--2 $\pm$ one standard error, WikiText-103, lr $10^{-4}$; dashed: Adam-NSCL at a uniform $259$ dims/row, the mean dimension GPM's energy criterion selects at $0.95$.}
\label{fig:pareto}
\end{figure}
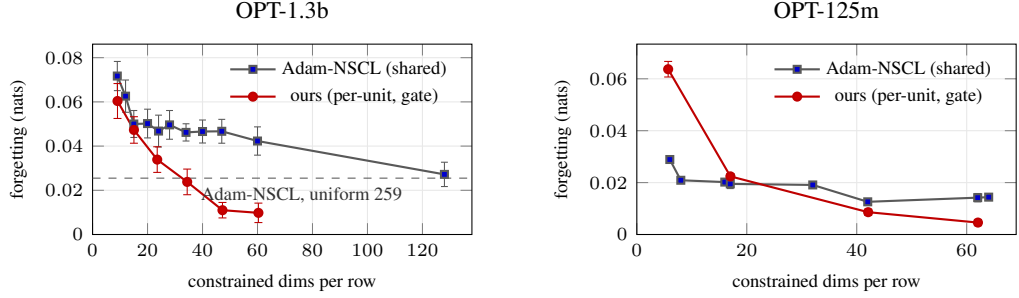

%% file: fig_families.tex
\begin{figure}[t]
\centering
\begin{tikzpicture}[font=\small]
\begin{axis}[width=10.6cm,height=4.3cm,x dir=reverse,y dir=reverse,xlabel={learning: change in new-task NLL (nats) $\longrightarrow$ more learning},ylabel={forgetting (nats) $\longrightarrow$ less},ymode=log,log basis y=10,xmin=-0.675,xmax=-0.545,ymin=0.0062,ymax=0.7,ytick={0.01,0.02,0.05,0.1,0.2,0.5},yticklabels={0.01,0.02,0.05,0.1,0.2,0.5},legend cell align={left},legend style={font=\scriptsize,draw=none,fill=none,at={(1.02,1)},anchor=north west},tick label style={font=\scriptsize},label style={font=\scriptsize},grid=major,grid style={gray!18}]
\addplot[black!55,line width=1.1pt,opacity=0.6,forget plot] coordinates {(-0.6157,0.0098) (-0.6157,0.0110) (-0.6189,0.0110) (-0.6189,0.0175) (-0.6241,0.0175) (-0.6241,0.0227) (-0.6346,0.0227) (-0.6346,0.2841) (-0.6626,0.2841)};
\draw[-latex,thick,black!70] (axis cs:-0.645,0.14) -- (axis cs:-0.669,0.03) node[pos=0.45,sloped,above,font=\scriptsize,black!80]{better};
\node[font=\scriptsize,black!80,anchor=west] at (axis cs:-0.6201,0.0086) {$\leftarrow$ ours, caps $512$ and $256$};
\addplot[only marks,black,mark=x,mark size=2.2] coordinates {(-0.5878,0.2684)};
\addlegendentry{no constraint}
\addplot[only marks,orange!85!black,mark=triangle*,mark size=1.8] coordinates {(-0.5850,0.2695) (-0.5863,0.2743) (-0.5863,0.2578) (-0.5962,0.2092) (-0.6144,0.0937) (-0.6241,0.0175)};
\addlegendentry{EWC}
\addplot[only marks,violet,mark=diamond*,mark size=1.9] coordinates {(-0.5698,0.0486)};
\addlegendentry{MIGU}
\addplot[only marks,teal,mark=pentagon*,mark size=1.7] coordinates {(-0.5553,0.1360) (-0.6567,0.3601) (-0.6320,0.2558) (-0.6611,0.3319) (-0.6563,0.3875) (-0.6626,0.2841) (-0.5927,0.5328)};
\addlegendentry{LoRA}
\addplot[only marks,cyan!55!black,mark=pentagon,mark size=1.9] coordinates {(-0.5602,0.0853) (-0.6325,0.1221)};
\addlegendentry{LoRA-Null}
\addplot[only marks,brown!80!black,mark=o,mark size=1.8] coordinates {(-0.5944,0.2648) (-0.6328,0.1893) (-0.6252,0.0558) (-0.5694,0.0103)};
\addlegendentry{L2-to-init}
\addplot[only marks,gray!70!black,mark=square*,mark size=1.4] coordinates {(-0.6289,0.0717) (-0.6252,0.0468) (-0.6229,0.0462) (-0.5671,0.0255) (-0.6078,0.0272) (-0.6346,0.0227)};
\addlegendentry{Adam-NSCL}
\addplot[only marks,red!75!black,mark=*,mark size=1.7] coordinates {(-0.6280,0.0604) (-0.6237,0.0339) (-0.6210,0.0238) (-0.6189,0.0110) (-0.6157,0.0098)};
\addlegendentry{ours}
\addlegendimage{black!55,line width=1.1pt,opacity=0.6}
\addlegendentry{Pareto frontier}
\end{axis}
\end{tikzpicture}
\caption{\textbf{Method families on the forgetting--learning plane.} OPT-1.3b, WikiText-103, means over seeds 0--2. Axes reversed: up and right is better. Each marker is one setting of its family's own hyperparameter. Grey steps: the Pareto frontier of the plotted points, whose low-forgetting end is our caps $512$ and $256$. Settings, and five points off the plotted range, each learning less than unconstrained fine-tuning, are in Appendix~\ref{app:cl:families:more}.}
\label{fig:families}
\end{figure}
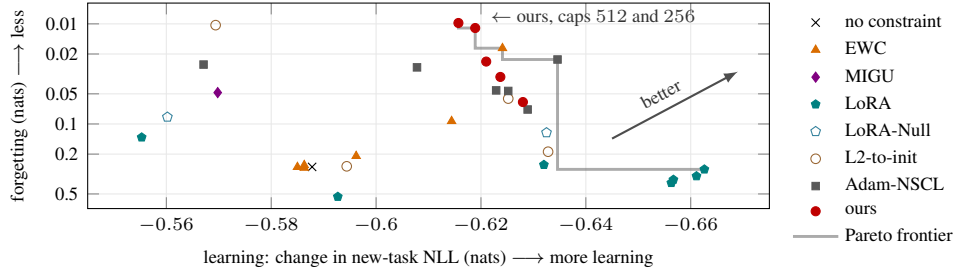

%% file: app_prelim.tex
\providecommand{\KX}{\mathcal K_{\mathcal X}}
\theoremstyle{plain}
\ifcsname theorem*\endcsname\else\newtheorem*{theorem*}{Theorem}\fi
\ifcsname proposition*\endcsname\else\newtheorem*{proposition*}{Proposition}\fi
\ifcsname corollary*\endcsname\else\newtheorem*{corollary*}{Corollary}\fi
\theoremstyle{remark}

\section{Preliminaries: a ReLU unit as a tropical binomial}\label{app:prelim}

In the max-plus notation of Section~\ref{sec:identity}, the summed output of a ReLU layer is the tropical
product of its units' binomials $0\oplus z_i$, and multiplying it out gives one monomial per subset $J$ of
units, the case $\theta=0$ of Proposition~\ref{prop:tropical}(i):
\begin{equation}\label{eq:app:layerpoly}
\sum_i\max(0,z_i)\;=\;\bigodot_i\,(0\oplus z_i)\;=\;\max_{J}\Big(\Big\langle\sum_{i\in J}w_i,\,x\Big\rangle+\sum_{i\in J}b_i\Big).
\end{equation}
Following \citet{zhang2018tropical}, lift each exponent $\sum_{i\in J}w_i$ by its coefficient
$\sum_{i\in J}b_i$. The convex hull of the lifted points is the zonotope $Z_0=\sum_i[0,(w_i,b_i)]$, and
projecting its upper faces gives a subdivision of the Newton polytope $\sum_i[0,w_i]$ whose vertices are the gate patterns
realised off the walls and whose cells are dual to the points where walls meet
(Figure~\ref{fig:duality}a--c). With real weights the exponents are real, so the polynomial is strictly a
tropical signomial; the duality used here is that of regular subdivisions and holds unchanged. For
$\theta>0$ the activation $\act_\theta$ is discontinuous and is not tropical; Proposition~\ref{prop:tropical}
describes its cells through the gate pattern it shares with $\max(0,z-\theta)$.
Figure~\ref{fig:app:tropical} draws Eq.~\eqref{eq:app:layerpoly} for a layer with one input, where every
object of Figure~\ref{fig:duality}a--c is one dimension lower, and Figure~\ref{fig:app:decomp} reads
Theorem~\ref{thm:decomp} one token at a time on the toy layer of Figure~\ref{fig:duality}.

\section{Proofs for Sections~\ref{sec:identity} and~\ref{sec:method}}\label{app:proofs23}

Every numbered statement of \S\S\ref{sec:identity}--\ref{sec:method} except Theorem~\ref{thm:decomp}, which the
body proves, is stated here in full and proved. A full form keeps the part letters of its body statement
(an unlettered one is split into parts); parts that appear only here carry a primed letter, as in (a$'$),
or a letter after the body's last one.

All statements concern one layer at fixed inputs. Tokens may repeat, so $\mathcal X=\{x_t\}_{t=1}^T$ is
a finite multiset and $\tilde X$ may have equal rows; nothing below assumes distinct tokens, full rank or
general position unless stated. Accordingly, the sets $O$, $O_i$, $S_i$, $S_{++}$, $S_{+-}$, $S_{-+}$,
$S_{--}$ and $N(u)$ are sets of token indices $t\in[T]$, a sum over $x\in\mathcal X\cap A$ runs over the
indices $t$ with $x_t\in A$, and $\tilde X_S\in\mathbb R^{|S|\times(d+1)}$ has one row per index in $S$;
for $S=\varnothing$ it is the empty matrix, of rank $0$ and with kernel $\mathbb R^{d+1}$. Throughout,
$d$, $m$ and $T$ are positive integers, $\theta\ge0$ is a real threshold, $[m]=\{1,\dots,m\}$, and
$\langle\cdot,\cdot\rangle$ and $\|\cdot\|$ are the Euclidean inner product and norm, with $\|\cdot\|_2$
the induced operator norm of a matrix. Every lifted token $\tilde x_t=(x_t,1)$ has last coordinate $1$,
so $\tilde x_t\neq0$; we say where this is used. For one unit with base weights $\tilde w$ and new
weights $\hat w=\tilde w+u$, $u\in\mathbb R^{d+1}$, we write $z_t=\langle\tilde w,\tilde x_t\rangle$ and
$\hat z_t=\langle\hat w,\tilde x_t\rangle=z_t+\langle u,\tilde x_t\rangle$; \emph{at the base inputs}
means that the $\tilde x_t$ are the base model's inputs to the layer and do not change with $u$.
Confining a row's update to a linear subspace $U\subseteq\mathbb R^{d+1}$ \emph{removes}
$\operatorname{codim}U=d+1-\dim U$ directions, and $P_U$ denotes the orthogonal projector onto $U$. For a
symmetric positive semidefinite $\Xi\in\mathbb R^{n\times n}$ with eigenvalues
$\lambda_1(\Xi)\ge\dots\ge\lambda_n(\Xi)$, listed with multiplicity, a \emph{top-$q$ eigenspace},
$0\le q\le n$, is the span of $q$ orthonormal eigenvectors whose eigenvalues are
$\lambda_1(\Xi),\dots,\lambda_q(\Xi)$.

For a set $V\subseteq\mathbb R^n$, $\operatorname{conv}V$ and $\operatorname{aff}V$ are its convex and
affine hulls and $\dim V=\dim\operatorname{aff}V$; the relative interior of a convex set is its interior
within its affine hull. For a linear functional $\ell$ on $\mathbb R^n$ and a compact convex set
$P\subseteq\mathbb R^n$, the \emph{face of $P$ exposed by $\ell$} is
$F_\ell(P)=\arg\max_{p\in P}\ell(p)$; a \emph{vertex} is an exposed face that is a single point, and an
\emph{edge} is an exposed face that is a segment of positive length. For a polyhedron $P$, possibly
unbounded, a \emph{face} is a set $P\cap\{y:c^\top y=\beta\}$ for an inequality $c^\top y\le\beta$
valid on $P$, and a \emph{facet} is a face of dimension $\dim P-1$. A \emph{chamber} of a finite family
of hyperplanes is a connected component of the complement of their union.

\begin{figure}[tbp]
\centering
\includegraphics[width=0.75\textwidth]{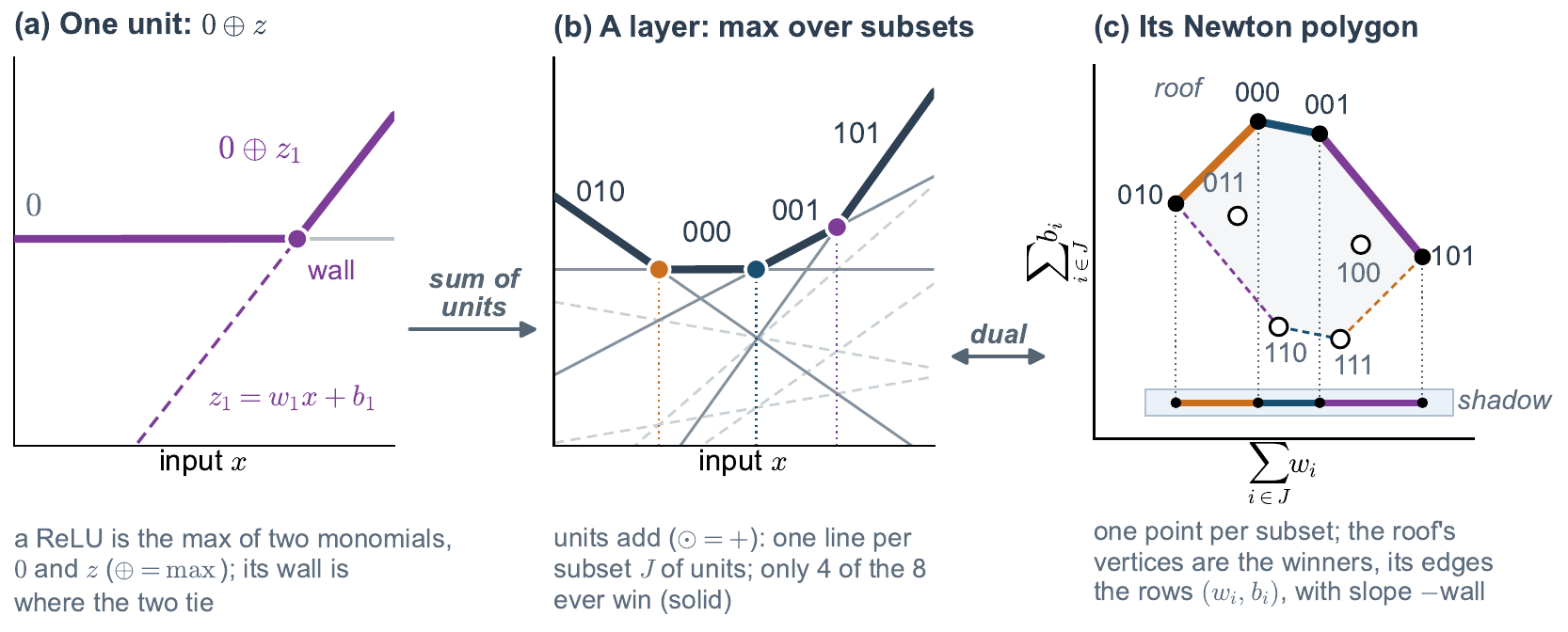}
\caption{\textbf{A ReLU layer as a tropical polynomial, for one input.} The toy layer has three units on a
scalar input $x$, with $z_1=x-1.2$, $z_2=-0.8x-0.8$ and $z_3=0.6x-0.12$; each unit keeps its colour from
Figure~\ref{fig:duality} (unit~1 purple, unit~2 amber, unit~3 navy).
\textbf{(a) One unit.} The grey line is the monomial $0$ and the dashed purple line the monomial $z_1$.
Their maximum, the thick line, is the ReLU $0\oplus z_1=\max(0,z_1)$. The two monomials tie at $x=1.2$, the
unit's wall: to its left the unit outputs $0$, to its right it outputs $z_1$.
\textbf{(b) The whole layer.} Adding the three outputs is tropical multiplication, $\bigodot_i(0\oplus z_i)$,
and multiplying out gives one line $\sum_{i\in J}z_i$ for each of the $8$ subsets $J$ of units
(Eq.~\eqref{eq:app:layerpoly}). The thick envelope is their maximum, which is the layer's summed output. Each
piece of the envelope is labelled by the subset that attains it, written as a bit string in which bit $i$ is
$1$ when unit $i$ is in $J$, that is, open. Reading left to right: at $x=-1.7$ only unit~2 fires, so the
winning line is $z_2$ and the output is $0.56$ (pattern $010$); between $x=-1.0$ and $x=0.2$ no unit fires
($000$); between $0.2$ and $1.2$ only unit~3 fires ($001$); beyond $1.2$ units~1 and~3 fire, and at $x=1.8$
the output is $z_1+z_3=1.56$ ($101$). The coloured dots are the three walls, where the envelope bends. The
four dashed lines never reach the envelope. Pattern $100$, for example, would need unit~1 open and unit~3
closed, that is $x>1.2$ and $x<0.2$ at once.
\textbf{(c) The same layer in exponent space.} Each subset $J$ becomes the point
$(\sum_{i\in J}w_i,\sum_{i\in J}b_i)$: $000$ is the origin, $001=(0.6,-0.12)$, $101=(1.6,-1.32)$, and so on.
The upper hull of the eight points, the roof, has as its vertices exactly the four patterns that win in (b).
Each roof edge is one unit's row. The navy edge from $000$ to $001$ is unit~3's row $(0.6,-0.12)$, and its
slope $-0.2$ is minus unit~3's wall at $x=0.2$. In the same way the amber edge has slope $1$, minus unit~2's
wall at $-1.0$, and the purple edge has slope $-1.2$, minus unit~1's wall at $1.2$. The hollow points lie
below the roof and never win: $100=(1.0,-1.2)$, for example, sits $0.6$ below the roof point above it. The
dashed edges bound the polygon from below and carry no winner. Dropping the height (dotted) projects the roof
onto the weight axis and splits it into three pieces, one per unit, which is the one-input counterpart of the
subdivision in Figure~\ref{fig:duality}b. The arrow marked \emph{dual} records that each roof vertex is a
region of (b) and each roof edge a wall.}
\label{fig:app:tropical}
\end{figure}

\begin{lemma}\label{lem:hull}
Let $V\subseteq\mathbb R^n$ be finite and nonempty, $P=\operatorname{conv}V$, $\ell$ a linear functional
on $\mathbb R^n$ and $M=\max_{v\in V}\ell(v)$. Then $\max_P\ell=M$ and
$F_\ell(P)=\operatorname{conv}\{v\in V:\ell(v)=M\}$.
\end{lemma}

\begin{proof}
Let $p\in P$. Since $P=\operatorname{conv}V$, there are weights $\lambda_v\ge0$, $v\in V$, with
$\sum_v\lambda_v=1$ and $p=\sum_v\lambda_vv$. By linearity of $\ell$,
\[
M-\ell(p)=\sum_{v\in V}\lambda_v\bigl(M-\ell(v)\bigr),
\]
a sum of non-negative terms. Hence $\ell(p)\le M$, with equality if and only if $\lambda_v=0$ for every
$v$ with $\ell(v)<M$. The points $v\in V$ with $\ell(v)=M$ lie in $P$, so $\max_P\ell=M$. If $\ell(p)=M$,
the representation of $p$ uses only points of $V_M=\{v\in V:\ell(v)=M\}$, so
$p\in\operatorname{conv}V_M$; conversely, every point of $\operatorname{conv}V_M$ lies in $P$ and, by
linearity, has value $M$. Therefore $F_\ell(P)=\operatorname{conv}V_M$.
\end{proof}

\begin{lemma}\label{lem:zono}
Let $a_1,\dots,a_m\in\mathbb R^n$ and, for $J\subseteq[m]$, $a_J=\sum_{i\in J}a_i$, so that
$a_\varnothing=0$. Then the Minkowski sum $\sum_{i=1}^m[0,a_i]$ equals
$\operatorname{conv}\{a_J:J\subseteq[m]\}$.
\end{lemma}

\begin{proof}
Write $Z=\sum_{i=1}^m[0,a_i]$. A Minkowski sum of convex sets is convex, and $Z$ contains every $a_J$
(take the $i$-th summand at $a_i$ for $i\in J$ and at $0$ otherwise), so
$\operatorname{conv}\{a_J\}\subseteq Z$. Conversely, let $p\in Z$, so that $p=\sum_i\lambda_ia_i$ with
every $\lambda_i\in[0,1]$. Order the indices as $(1),\dots,(m)$ so that
$\lambda_{(1)}\ge\dots\ge\lambda_{(m)}$, put $\lambda_{(0)}=1$ and $\lambda_{(m+1)}=0$, and let
$J_k=\{(1),\dots,(k)\}$ for $0\le k\le m$. We claim that
\[
p=\sum_{k=0}^m\bigl(\lambda_{(k)}-\lambda_{(k+1)}\bigr)\,a_{J_k}.
\]
Indeed, $a_{(j)}$ occurs in $a_{J_k}$ exactly when $k\ge j$, so its coefficient on the right is
$\sum_{k=j}^m(\lambda_{(k)}-\lambda_{(k+1)})=\lambda_{(j)}-\lambda_{(m+1)}=\lambda_{(j)}$. The weights
$\lambda_{(k)}-\lambda_{(k+1)}$ are non-negative, because $1\ge\lambda_{(1)}$, the $\lambda_{(k)}$ are
sorted and $\lambda_{(m)}\ge0$, and they sum to $\lambda_{(0)}-\lambda_{(m+1)}=1$. Hence
$p\in\operatorname{conv}\{a_J\}$.
\end{proof}

\begin{lemma}\label{lem:faces}
Let $K$ be a positive integer, let $g_1,\dots,g_K:\mathbb R^n\to\mathbb R$ be affine and non-constant, let
$P=\{y\in\mathbb R^n:g_t(y)\le0,\ t=1,\dots,K\}$, and assume that $P$ has nonempty interior.
\par\textbf{(i)} The boundary of $P$ is $\partial P=\bigcup_t(P\cap\{g_t=0\})$, and the interior points of $P$
are the points with $g_t<0$ for every $t$.
\par\textbf{(ii)} Every facet of $P$ equals $P\cap\{g_t=0\}$ for some $t$.
\par\textbf{(iii)} Every $y\in\partial P$ lies on a facet, and every facet containing $y$ equals
$P\cap\{g_t=0\}$ for some $t$ with $g_t(y)=0$.
\end{lemma}

\begin{proof}
Write $g_t(y)=c_t^\top y-\beta_t$ with $c_t\in\mathbb R^n\setminus\{0\}$, which is what non-constant
means, and $\beta_t\in\mathbb R$. The set $P$ is convex and, as a finite intersection of closed
half-spaces, closed.

\emph{(i)} Let $y\in P$. If $g_t(y)<0$ for every $t$, then by continuity of the finitely many $g_t$ a
neighbourhood of $y$ lies in $P$, so $y$ is an interior point. If $g_t(y)=0$ for some $t$, then for every
$\varepsilon>0$ the point $y+\varepsilon c_t$ satisfies $g_t(y+\varepsilon c_t)=\varepsilon\|c_t\|^2>0$
and lies outside $P$; hence $y$ is a limit of points outside $P$, that is, a boundary point. Since $P$
is closed, $\partial P\subseteq P$, and the two cases prove both claims.

\emph{(ii)} Let $F=P\cap\{c^\top y=\beta\}$ be a facet, where $c^\top y\le\beta$ is valid on $P$. Since
$\dim F=n-1<n=\dim P$, the face $F$ is neither $P$ nor empty, which forces $c\neq0$. Let $y_0$ be a
point of the relative interior of $F$, which is nonempty because $F$ is a nonempty convex set. The point
$y_0$ maximises $c^\top y$ over $P$, and $c\neq0$, so $y_0$ is not an interior point of $P$ (otherwise
$y_0+\varepsilon c\in P$ for small $\varepsilon>0$ would have a larger value); by (i), $g_t(y_0)=0$ for
some $t$. Fix such a $t$ and let $y\in F$. Because $y_0$ lies in the relative interior of $F$ and
$y_0+\varepsilon(y_0-y)\in\operatorname{aff}F$, the point $y_0+\varepsilon(y_0-y)$ lies in
$F\subseteq P$ for all small $\varepsilon>0$. Since $g_t$ is affine and $g_t(y_0)=0$,
\[
0\ge g_t\bigl(y_0+\varepsilon(y_0-y)\bigr)=(1+\varepsilon)g_t(y_0)-\varepsilon g_t(y)=-\varepsilon g_t(y),
\]
so $g_t(y)\ge0$, and $y\in P$ gives $g_t(y)\le0$; hence $g_t(y)=0$. Thus
$F\subseteq G:=P\cap\{g_t=0\}$. The set $G$ is convex and lies in the hyperplane $\{g_t=0\}$, so
$\dim G\le n-1=\dim F$; together with $F\subseteq G$ this gives $\operatorname{aff}F=\operatorname{aff}G$.
It follows that $y_0$, which has a neighbourhood in $\operatorname{aff}F$ contained in $F\subseteq G$,
lies in the relative interior of $G$. Now let $y\in G$. For all small $\varepsilon>0$ the point
$y_0+\varepsilon(y_0-y)$ lies in $G\subseteq P$, and validity of $c^\top y\le\beta$ together with
$c^\top y_0=\beta$ gives
\[
\beta\ge c^\top\bigl(y_0+\varepsilon(y_0-y)\bigr)=(1+\varepsilon)\beta-\varepsilon\,c^\top y ,
\]
so $c^\top y\ge\beta$. With validity this gives $c^\top y=\beta$, that is, $y\in F$. Therefore $F=G$.

\emph{(iii)} Let $I\subseteq[K]$ be minimal under inclusion with $P=\{y:g_t(y)\le0\ \text{for all }t\in I\}$;
it exists because $[K]$ is finite. Let $y\in\partial P$. Part (i), applied to the system
$\{g_t\}_{t\in I}$, which describes the same set $P$, gives an $s\in I$ with $g_s(y)=0$. We show that
$F_s=P\cap\{g_s=0\}$ is a facet. By minimality of $I$ there is a point $p\in\mathbb R^n$ with $g_s(p)>0$
and $g_t(p)\le0$ for every $t\in I\setminus\{s\}$, since otherwise the constraint $s$ could be dropped.
Let $q$ be an interior point of $P$; by (i), $g_t(q)<0$ for every $t$. The function
$\lambda\mapsto g_s((1-\lambda)q+\lambda p)$ is affine on $[0,1]$, negative at $0$ and positive at $1$, so
it vanishes at some $\lambda^\ast\in(0,1)$. At $y_0=(1-\lambda^\ast)q+\lambda^\ast p$, every
$t\in I\setminus\{s\}$ satisfies $g_t(y_0)=(1-\lambda^\ast)g_t(q)+\lambda^\ast g_t(p)<0$. By continuity
there is a neighbourhood $N$ of $y_0$ on which these finitely many inequalities stay strict, and since the
system $I$ describes $P$, this gives $N\cap\{g_s=0\}\subseteq F_s$. The set $N\cap\{g_s=0\}$ is a nonempty
relatively open subset of the hyperplane $\{g_s=0\}$, so $\dim F_s\ge n-1$, and $F_s$ lies in that
hyperplane, so $\dim F_s=n-1$. Since $g_s\le0$ is valid on $P$, the set $F_s$ is a face, hence a facet,
and it contains $y$. Finally, by (ii) a facet containing $y$ equals $P\cap\{g_t=0\}$ for some $t\in[K]$,
and $y$ lying in it forces $g_t(y)=0$.
\end{proof}

\paragraph{Proposition~\ref{prop:tropical}: the tropical form and the dual subdivision.}
Here $a_i=(w_i,b_i-\theta)\in\mathbb R^{d+1}$ for $i\in[m]$, so that
$\tilde w^\theta_J=\sum_{i\in J}a_i$ and $Z_\theta=\sum_{i=1}^m[0,a_i]$. For $y\in\mathbb R^d$ the linear
functional $\ell_y=\langle\cdot,(y,1)\rangle$ on $\mathbb R^{d+1}$ is the functional \emph{given by
$(y,1)$}. A set $J\subseteq[m]$ \emph{maximises at $x$} if it attains the maximum in (i) below, and a
gate pattern is \emph{realised off the walls} if it equals $G(x)$ for a point $x$ that lies on no wall.

\begin{proposition*}[Full form of Proposition~\ref{prop:tropical}]
\textbf{(i)} For every $x\in\mathbb R^d$,
$\sum_{i=1}^m\max(0,z_i(x)-\theta)=\max_{J\subseteq[m]}\langle\tilde w^\theta_J,\tilde x\rangle$; at
$\theta=0$ this is $\sum_i\act_0(z_i(x))=\bigodot_i\,(0\oplus z_i(x))$. A set $J\subseteq[m]$ maximises
at $x$ if and only if $G(x)\subseteq J\subseteq\{i:z_i(x)\ge\theta\}$. Hence the set $G(x)$ always
maximises at $x$, and it is the only maximiser if and only if $x$ lies on no wall.
\par\textbf{(ii)} The map $J\mapsto\tilde w^\theta_J$ sends the gate patterns realised off the walls
injectively into the set of vertices of $Z_\theta$ exposed by a functional $\ell_y$, $y\in\mathbb R^d$,
and onto that set when no unit is constant at the threshold, that is, when $(w_i,b_i)\neq(0,\theta)$ for
every $i$.
\par\textbf{(iii)} Let $x,x'\in\mathbb R^d$ lie on no wall, and let $G(x')=G(x)\sqcup\{i\}$ for a unit
$i\notin G(x)$. Then $[\tilde w^\theta_{G(x)},\tilde w^\theta_{G(x')}]$ is an edge of $Z_\theta$ with
direction $a_i=(w_i,b_i-\theta)$. The set
\[
\Omega=\bigl\{y\in\mathcal W_i:\ z_j(y)>\theta\ \text{for } j\in G(x),\ \ z_j(y)<\theta\ \text{for }
j\notin G(x)\cup\{i\}\bigr\}
\]
of points of $\mathcal W_i$ that lie on no other wall and at which the other units have pattern $G(x)$ is
a nonempty relatively open subset of the hyperplane $\mathcal W_i$, whose normal $w_i$ is the projection
of the edge's direction to the first $d$ coordinates; and for every $y\in\Omega$ the face of $Z_\theta$
exposed by $\ell_y$ is this edge.
\end{proposition*}

\begin{proof}
Fix $x\in\mathbb R^d$ and put $c_i=z_i(x)-\theta$ for $i\in[m]$. Since
$\langle a_i,\tilde x\rangle=\langle w_i,x\rangle+b_i-\theta=c_i$, every $J\subseteq[m]$ has
$\langle\tilde w^\theta_J,\tilde x\rangle=\sum_{i\in J}c_i$, the empty sum being $0$.

\emph{(i), the identity.} For every $J$, $\sum_{i\in J}c_i\le\sum_{i\in J}\max(0,c_i)\le\sum_{i\in[m]}\max(0,c_i)$,
because $c_i\le\max(0,c_i)$ and $\max(0,c_i)\ge0$. The set $\{i:c_i>0\}$ attains the right-hand side, so
$\max_J\langle\tilde w^\theta_J,\tilde x\rangle=\sum_i\max(0,c_i)$, which is the identity. At $\theta=0$,
$\max(0,z_i(x))=\act_0(z_i(x))=0\oplus z_i(x)$, and the ordinary sum of the $m$ units is their tropical
product $\bigodot_i(0\oplus z_i(x))$. Expanding this product by distributivity of $\odot$ over $\oplus$
gives $\bigoplus_J\bigodot_{i\in J}z_i(x)=\max_J\sum_{i\in J}z_i(x)$, which is the maximum above at
$\theta=0$.

\emph{(i), the maximisers.} For every $J$,
\[
\sum_{i\in[m]}\max(0,c_i)-\sum_{i\in J}c_i=\sum_{i\notin J}\max(0,c_i)+\sum_{i\in J}\max(0,-c_i),
\]
because $\max(0,c_i)-c_i=\max(0,-c_i)$. Every summand on the right is non-negative, so $J$ maximises at
$x$ if and only if every summand vanishes, that is, if and only if $c_i\le0$ for every $i\notin J$ and
$c_i\ge0$ for every $i\in J$. The first condition says $\{i:c_i>0\}\subseteq J$, that is,
$G(x)\subseteq J$, and the second says $J\subseteq\{i:c_i\ge0\}=\{i:z_i(x)\ge\theta\}$. In particular
$J=G(x)$ maximises. If $x$ lies on no wall, no $c_i$ is $0$, so $G(x)=\{i:z_i(x)\ge\theta\}$ and only one
set lies between them. If $x$ lies on the wall of a unit $j$, then $c_j=0$, and $G(x)$ and
$G(x)\cup\{j\}$ are two different maximisers.

\emph{(ii)} By Lemma~\ref{lem:zono}, $Z_\theta=\operatorname{conv}\{\tilde w^\theta_J:J\subseteq[m]\}$.
For $y\in\mathbb R^d$, the value of $\ell_y$ at $\tilde w^\theta_J$ is the term of $J$ in the maximum at
$y$, so by Lemma~\ref{lem:hull} the face $F_{\ell_y}(Z_\theta)$ is the convex hull of the points
$\tilde w^\theta_J$ of the sets $J$ that maximise at $y$.

\emph{Into.} Let $x$ lie on no wall. By (i) the set $G(x)$ is the only maximiser at $x$, so
$F_{\ell_x}(Z_\theta)=\{\tilde w^\theta_{G(x)}\}$, a vertex exposed by $\ell_x$.

\emph{Injective.} Let $x$ and $x'$ lie on no wall, with $\tilde w^\theta_{G(x')}=\tilde w^\theta_{G(x)}$.
The term of $G(x')$ at $x$ is then $\langle\tilde w^\theta_{G(x)},\tilde x\rangle$, the maximum, so
$G(x')$ maximises at $x$; since the maximiser at $x$ is unique by (i), $G(x')=G(x)$.

\emph{Onto.} Assume that $a_j\neq0$ for every $j$, which is the condition $(w_j,b_j)\neq(0,\theta)$. Let
$v$ be a vertex of $Z_\theta$ exposed by $\ell_y$ for some $y\in\mathbb R^d$, so that
$F_{\ell_y}(Z_\theta)=\{v\}$ and every $J$ that maximises at $y$ has $\tilde w^\theta_J=v$. Suppose, for
a contradiction, that $y$ lies on the wall of some unit $j$. Then $j\notin G(y)$ and
$z_j(y)\ge\theta$, so by (i) both $G(y)$ and $G(y)\cup\{j\}$ maximise at $y$; their points differ by
$a_j\neq0$, yet both equal $v$. Hence $y$ lies on no wall, $G(y)$ is a pattern realised off the walls, and
$v=\tilde w^\theta_{G(y)}$.

\emph{(iii)} Let $J=G(x)$, so that $G(x')=J\sqcup\{i\}$. For $j\neq i$, $j\in G(x)$ if and only if
$j\in G(x')$, and since neither $x$ nor $x'$ lies on $\mathcal W_j$, the numbers $z_j(x)-\theta$ and
$z_j(x')-\theta$ are nonzero and of the same sign. For the unit $i$, $i\notin G(x)$ and
$x\notin\mathcal W_i$ give $z_i(x)<\theta$, and $i\in G(x')$ gives $z_i(x')>\theta$. In particular
$z_i$ is not constant, so $w_i\neq0$ and $\mathcal W_i=\{y:\langle w_i,y\rangle=\theta-b_i\}$ is a
hyperplane with normal $w_i$.

\emph{The set $\Omega$ is nonempty.} For $s\in[0,1]$ let $x_s=(1-s)x+sx'$. Since each $z_j$ is affine,
$z_j(x_s)-\theta=(1-s)(z_j(x)-\theta)+s(z_j(x')-\theta)$. For $j\neq i$ this is a convex combination of
two nonzero numbers of the same sign, hence nonzero with that sign. For $j=i$ it is an affine function of
$s$ that is negative at $s=0$ and positive at $s=1$, so it vanishes at exactly one $s^\ast\in(0,1)$. The
point $x_{s^\ast}$ therefore lies on $\mathcal W_i$, has $z_j(x_{s^\ast})>\theta$ for $j\in J$ and
$z_j(x_{s^\ast})<\theta$ for $j\notin J\cup\{i\}$, and so belongs to $\Omega$.

\emph{The set $\Omega$ is relatively open.} It is the intersection of $\mathcal W_i$ with the finitely
many sets $\{z_j>\theta\}$, $j\in J$, and $\{z_j<\theta\}$, $j\notin J\cup\{i\}$, each of which is open in
$\mathbb R^d$ because $z_j$ is continuous; hence $\Omega$ is open in $\mathcal W_i$.

\emph{The exposed face.} Let $y\in\Omega$. Then $z_i(y)=\theta$, and for $j\neq i$, $z_j(y)\neq\theta$
with $z_j(y)>\theta$ if and only if $j\in J$. Hence $G(y)=J$ and $\{j:z_j(y)\ge\theta\}=J\cup\{i\}$, and by
(i) the maximisers at $y$ are exactly $J$ and $J\cup\{i\}$. By (ii) the face exposed by $\ell_y$ is
$\operatorname{conv}\{\tilde w^\theta_J,\tilde w^\theta_J+a_i\}=[\tilde w^\theta_{G(x)},\tilde
w^\theta_{G(x')}]$, a segment of length $\|a_i\|\ge\|w_i\|>0$. It is therefore an edge of $Z_\theta$, its
direction is $a_i$, and the first $d$ coordinates of $a_i$ form $w_i$, the normal of $\mathcal W_i$.
Since $\Omega$ is nonempty, this proves every claim of (iii).
\end{proof}

\paragraph{Proposition~\ref{prop:duality}: the data--weight duality.}
Here one unit is fixed, with base weights $\tilde w\in\mathbb R^{d+1}$, pre-activations
$z_t=\langle\tilde w,\tilde x_t\rangle$ and base open set $O=\{t:z_t>\theta\}$. Token $t$ is \emph{open at}
$\tilde v\in\mathbb R^{d+1}$ if $\langle\tilde v,\tilde x_t\rangle>\theta$ and closed otherwise, and it
\emph{changes cell} between two weight points if it is open at exactly one of them. The sets $H_t$,
$\PX$ and $\KX$ are those of Definition~\ref{def:poly}, so that
$\KX=\{\tilde v:\langle\tilde v,\tilde x_t\rangle>\theta\ (t\in O),\ \langle\tilde v,\tilde
x_t\rangle\le\theta\ (t\notin O)\}$, and $e_{d+1}$ is the last standard basis vector of $\mathbb R^{d+1}$.
Each $H_t$ is a hyperplane because $\tilde x_t\neq0$.

\begin{proposition*}[Full form of Proposition~\ref{prop:duality}]
\textbf{(a)} Token $t$ is open at $\tilde v\in\mathbb R^{d+1}$ if and only if $\tilde v$ lies in the open
half-space $H_t^>=\{\tilde v:\langle\tilde v,\tilde x_t\rangle>\theta\}$. For $\hat w\in\mathbb R^{d+1}$,
token $t$ changes cell between $\tilde w$ and $\hat w$ if and only if exactly one of them lies in
$H_t^>$, and it closes ($t\in S_{+-}$) if and only if that one is $\tilde w$. Along the segment
$[\tilde w,\hat w]$ it changes cell at most once: the gate bit
$s\mapsto\mathds 1[\langle(1-s)\tilde w+s\hat w,\tilde x_t\rangle>\theta]$, $s\in[0,1]$, is monotone.
\par\textbf{(b)} The pattern cell $\KX=\PX\setminus\bigcup_{t\in O}H_t$ is convex. Its interior is
$Q=\{\tilde v:\langle\tilde v,\tilde x_t\rangle>\theta\ (t\in O),\ \langle\tilde v,\tilde
x_t\rangle<\theta\ (t\notin O)\}$, which contains $\tilde w-\varepsilon e_{d+1}$ for every $\varepsilon$
with $0<\varepsilon<\min_{t\in O}(z_t-\theta)$ (every $\varepsilon>0$ if $O=\varnothing$) and is a chamber
of the arrangement $\{H_t\}_{t\in[T]}$; and $\overline{\KX}=\PX$.
\par\textbf{(c)} The set $\PX$ is a full-dimensional polyhedron, $\partial\PX=\bigcup_t(\PX\cap H_t)$, every
boundary point lies on a facet, and every facet of $\PX$ is $\PX\cap H_t$ for a token $t$: an active
facet if $t\in O$, an inactive one if not, a label that does not depend on which such $t$ is chosen.
\end{proposition*}

\begin{proof}
\emph{(a)} By definition token $t$ is open at $\tilde v$ if and only if
$\langle\tilde v,\tilde x_t\rangle>\theta$, that is, $\tilde v\in H_t^>$; likewise
$x_t\in\mathcal W_{\tilde v}$ if and only if $\tilde v\in H_t$. Token $t$ changes cell if and only if its
gate bits at $\tilde w$ and $\hat w$ differ, that is, if and only if exactly one of the two points lies in
$H_t^>$. It closes, $t\in\mathcal C_>\cap\hat{\mathcal C}_\le=S_{+-}$, if and only if it is open at
$\tilde w$ and closed at $\hat w$, that is, if and only if the point in $H_t^>$ is $\tilde w$. Along the
segment, $f(s)=\langle(1-s)\tilde w+s\hat w,\tilde x_t\rangle=(1-s)z_t+s\hat z_t$ is affine in $s$. If
$\hat z_t\ge z_t$, then $f$ is non-decreasing and so is the bit $b(s)=\mathds 1[f(s)>\theta]$; if
$\hat z_t<z_t$, both are non-increasing. A monotone function with values in $\{0,1\}$ changes value at
most once: there are no $s_1<s_2<s_3$ with $b(s_1)=b(s_3)\neq b(s_2)$.

\emph{(b)} The set $\KX$ is the intersection of the open half-spaces $H_t^>$, $t\in O$, and the closed
half-spaces $\{\langle\cdot,\tilde x_t\rangle\le\theta\}$, $t\notin O$, hence convex. The points of
$\PX\setminus\bigcup_{t\in O}H_t$ satisfy $\ge\theta$ and $\neq\theta$ on $O$ and $\le\theta$ off $O$,
which is the definition of $\KX$.

\emph{The point $\tilde w-\varepsilon e_{d+1}$.} The last coordinate of every $\tilde x_t$ is $1$, so
$\langle\tilde w-\varepsilon e_{d+1},\tilde x_t\rangle=z_t-\varepsilon$. For $t\in O$ and
$0<\varepsilon<\min_{s\in O}(z_s-\theta)$ this exceeds $\theta$, and for $t\notin O$ it is at most
$\theta-\varepsilon<\theta$. Hence $\tilde w-\varepsilon e_{d+1}\in Q$, and $Q\neq\varnothing$.

\emph{The interior of $\KX$ is $Q$.} The set $Q$ is open, as a finite intersection of open half-spaces,
and $Q\subseteq\KX$, so $Q$ lies in the interior of $\KX$. Conversely, let $\tilde v$ be an interior point
of $\KX$. Then $\langle\tilde v,\tilde x_t\rangle>\theta$ for $t\in O$. Suppose, for a contradiction, that
$\langle\tilde v,\tilde x_t\rangle=\theta$ for some $t\notin O$. For every $\varepsilon>0$,
$\langle\tilde v+\varepsilon\tilde x_t,\tilde x_t\rangle=\theta+\varepsilon\|\tilde x_t\|^2>\theta$
because $\tilde x_t\neq0$, so $\tilde v+\varepsilon\tilde x_t\notin\KX$, and $\tilde v$ is not interior.
Hence $\langle\tilde v,\tilde x_t\rangle<\theta$ for every $t\notin O$, and $\tilde v\in Q$.

\emph{The set $Q$ is a chamber.} Let $A=\mathbb R^{d+1}\setminus\bigcup_{t\in[T]}H_t$ and, for
$\sigma\in\{+1,-1\}^T$, $Q_\sigma=\{\tilde v:\sigma_t(\langle\tilde v,\tilde x_t\rangle-\theta)>0\ \text{for
all }t\}$. Every point of $A$ has $\langle\tilde v,\tilde x_t\rangle\neq\theta$ for all $t$, which fixes
one $\sigma$, so the sets $Q_\sigma$ are pairwise disjoint, open and convex, and their union is $A$. Let
$\Gamma$ be a connected component of $A$ and $\tilde v\in\Gamma\cap Q_\sigma$. The sets $\Gamma\cap
Q_{\sigma'}$ are disjoint, relatively open in $\Gamma$ and cover it, so connectedness of $\Gamma$ gives
$\Gamma\subseteq Q_\sigma$. Conversely, $Q_\sigma$ is convex, hence connected, contains $\tilde v$ and lies
in $A$, so $Q_\sigma\subseteq\Gamma$. The chambers are therefore exactly the nonempty sets $Q_\sigma$, and
$Q=Q_{\sigma^0}$ with $\sigma^0_t=+1$ for $t\in O$ and $\sigma^0_t=-1$ otherwise is nonempty, hence a
chamber.

\emph{The closure.} Since $Q\subseteq\KX\subseteq\PX$ and $\PX$ is closed, $\overline{\KX}\subseteq\PX$.
Let $y\in\PX$, $q\in Q$ and $s\in(0,1]$, and put $y_s=(1-s)y+sq$. For $t\in O$,
$\langle y_s,\tilde x_t\rangle=(1-s)\langle y,\tilde x_t\rangle+s\langle q,\tilde x_t\rangle>(1-s)\theta+s\theta=\theta$,
because $\langle y,\tilde x_t\rangle\ge\theta$, $\langle q,\tilde x_t\rangle>\theta$ and $s>0$; for
$t\notin O$ the same computation gives $\langle y_s,\tilde x_t\rangle<\theta$. Hence $y_s\in Q$, and
$y_s\to y$ as $s\to0$. It follows that $\PX\subseteq\overline Q\subseteq\overline{\KX}\subseteq\PX$.

\emph{(c)} Put $g_t(\tilde v)=\theta-\langle\tilde v,\tilde x_t\rangle$ for $t\in O$ and
$g_t(\tilde v)=\langle\tilde v,\tilde x_t\rangle-\theta$ for $t\notin O$. Then $\PX=\{g_t\le0\text{ for all
}t\}$, each $g_t$ is affine with linear part $\pm\tilde x_t\neq0$, hence non-constant, and
$\{g_t=0\}=H_t$. The polyhedron $\PX$ contains the nonempty open set $Q$ of (b), so it has nonempty
interior and is full-dimensional, and Lemma~\ref{lem:faces} gives the description of $\partial\PX$, the
facet through every boundary point, and the form $\PX\cap H_t$ of every facet. For the label, let $F$ be a
facet with $F=\PX\cap H_t=\PX\cap H_s$. The facet has dimension $d$ and lies in $H_t\cap H_s$; if
$H_t\neq H_s$, this intersection is empty or an affine subspace of dimension $d-1$, which cannot contain
$F$, so $H_t=H_s$. The affine functions $\langle\cdot,\tilde x_t\rangle-\theta$ and
$\langle\cdot,\tilde x_s\rangle-\theta$ then have the same zero hyperplane, so their linear parts are
proportional, $\tilde x_s=c\,\tilde x_t$ for a real $c\neq0$; the last coordinates, both equal to $1$,
give $c=1$. Hence $\tilde x_s=\tilde x_t$, so $z_s=z_t$, and $s$ and $t$ have the same base gate bit.
\end{proof}

\begin{figure}[tbp]
\centering
\includegraphics[width=0.75\textwidth]{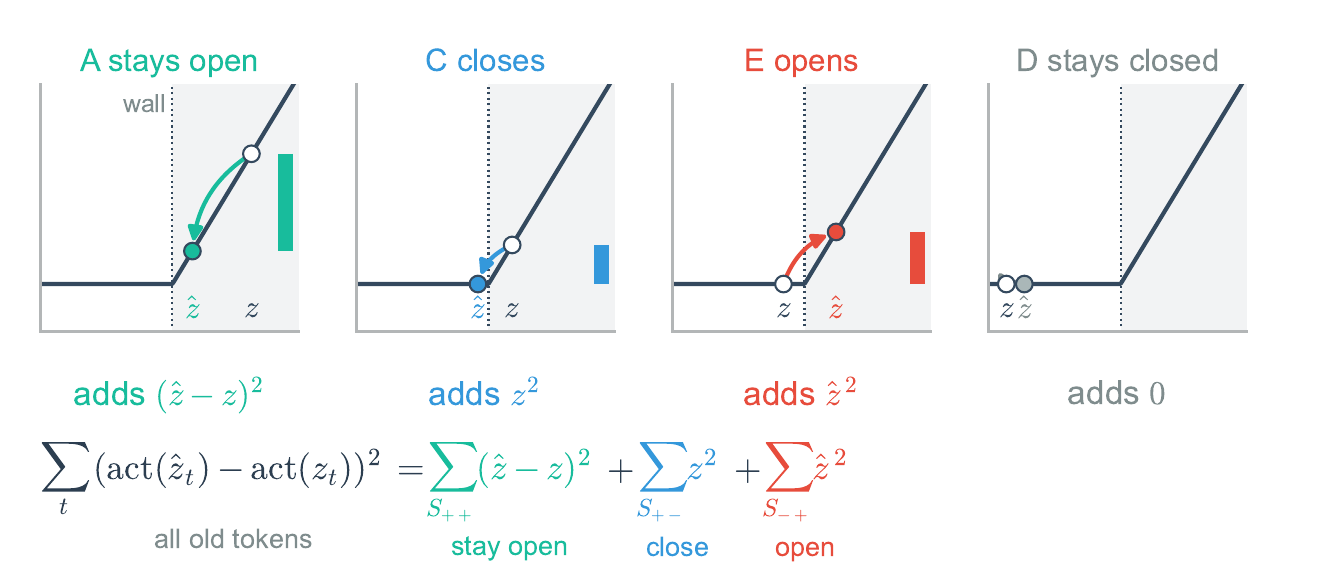}
\caption{\textbf{Theorem~\ref{thm:decomp}, one token at a time.} The four old tokens of
Figure~\ref{fig:duality}d under the update of unit~1 from $w$ to $\hat w$. Each panel draws the activation
$\act(z)=0\oplus z$ against the pre-activation; shading marks the open side of the wall (dotted). The hollow dot is
the token's pre-activation $z$ at $w$, the filled dot its pre-activation $\hat z$ at $\hat w$, and the bar on
the right is the change in the unit's output. The term under each panel is what that token adds to
Eq.~\eqref{eq:identity}, and the words under the formula name the tokens each sum runs over. $A$ stays open, moving from $z=1.50$ to $\hat z=0.38$; its output falls by $1.12$,
and it contributes the in-cell error $(\hat z-z)^2=1.26$. $C$ closes, moving from $0.45$ to $-0.20$; its
output falls from $0.45$ to $0$, whatever $\hat z$ is below the wall, so it contributes its old output
squared, $z^2=0.20$. $E$ opens, moving from $-0.40$ to $0.60$; its old output was $0$, so it contributes its
new output squared, $\hat z^{\,2}=0.36$. $D$ moves from $-2.17$ to $-1.83$ but stays closed, so its output is
$0$ before and after and it contributes nothing, however far it moves. The three terms add to $1.82$, which is
the total squared change of the unit's output over the four tokens, with no remainder.}
\label{fig:app:decomp}
\end{figure}

\paragraph{Corollary~\ref{cor:weight}: the weight-space reading.}
\begin{corollary*}[Full form of Corollary~\ref{cor:weight}]
Let $\hat w\in\mathbb R^{d+1}$, and let $O$, $\PX$ and $\KX$ be those of the base weights $\tilde w$.
\par\textbf{(a)} The set $S_{+-}\cup S_{-+}$ consists of the tokens whose hyperplanes separate $\tilde w$
from $\hat w$ in the sense of Proposition~\ref{prop:duality}a, that is, with exactly one of the two points
in $H_t^>$, and $S_{+-}$ consists of those with $\tilde w\in H_t^>$.
\par\textbf{(b)} On every set $W\subseteq\mathbb R^{d+1}$ of weights with a fixed gate pattern on
$\mathcal X$ (every $\hat w\in W$ gives each token the same gate bit), $\post$ is the restriction to $W$
of a polynomial of degree at most two in $\hat w$.
\par\textbf{(c)} On $\KX$, $\post$ is the in-cell quadratic $\|\tilde X_O(\hat w-\tilde w)\|^2$, and at
$\theta=0$ it equals that quadratic on all of $\PX$.
\end{corollary*}

\begin{proof}
\emph{(a)} By definition $t\in S_{+-}\cup S_{-+}$ if and only if $t$ changes cell between $\tilde w$ and
$\hat w$, and $t\in S_{+-}$ if and only if it closes; Proposition~\ref{prop:duality}a translates both
conditions into the stated ones.

\emph{(b)} Each of the four index sets of Theorem~\ref{thm:decomp} is determined by the base gate bits
and the new gate bits, so all four are the same for every $\hat w\in W$. For such $\hat w$ the in-cell
term $\sum_{t\in S_{++}}\langle\hat w-\tilde w,\tilde x_t\rangle^2$ and the closed$\to$open term
$\sum_{t\in S_{-+}}\langle\hat w,\tilde x_t\rangle^2$ are polynomials of degree at most two in $\hat w$,
and the open$\to$closed term $\sum_{t\in S_{+-}}z_t^2$ does not depend on $\hat w$. By
Theorem~\ref{thm:decomp} their sum is $\post$.

\emph{(c)} Let $\hat w\in\KX$. Every token then has its base gate bit, so $S_{+-}=S_{-+}=\varnothing$ and
$S_{++}=O$, and Theorem~\ref{thm:decomp} gives
$\post=\sum_{t\in O}\langle\hat w-\tilde w,\tilde x_t\rangle^2=\|\tilde X_O(\hat w-\tilde w)\|^2$. Now let
$\theta=0$ and $\hat w\in\PX\setminus\KX$. Since $\hat w\in\PX$, $\hat z_t\ge0$ for $t\in O$ and
$\hat z_t\le0$ for $t\notin O$, and since $\KX=\PX\setminus\bigcup_{t\in O}H_t$ by
Proposition~\ref{prop:duality}b, the set $O_0=\{t\in O:\hat z_t=0\}$ is nonempty. The tokens of
$O\setminus O_0$ lie in $S_{++}$ and contribute $(\hat z_t-z_t)^2$. The tokens of $O_0$ lie in $S_{+-}$
and contribute $z_t^2$, which equals $(\hat z_t-z_t)^2$ because $\hat z_t=0$. The tokens off $O$ have
$z_t\le0$ and $\hat z_t\le0$, so they lie in $S_{--}$ and contribute $0$. Hence
$\post=\sum_{t\in O}(\hat z_t-z_t)^2=\|\tilde X_O(\hat w-\tilde w)\|^2$.
\end{proof}

\paragraph{Proposition~\ref{prop:zero}: exactness and the ridge residual.}
Here unit $i$, the base inputs and an integer cap $C\ge0$ are fixed, and $S_i$, $k_i$, $A_i$
and $P_i(\rho)$ are those of Definition~\ref{def:proj}. We write $z_t=z_i(x_t)$; for $u\in\mathbb R^{d+1}$
and $\rho\ge0$ we put $u'=P_i(\rho)u$, the new row is $\tilde w_i+u'$, and $\hat z_t=z_t+\langle
u',\tilde x_t\rangle$. The sets $H_t$, $\PX$ and $\KX$ are those of unit $i$, that is, of
Definition~\ref{def:poly} with $O=O_i$. Since $\theta\ge0$, $h^{(i)}_t>0$ if and only if $t\in O_i$; hence
$S_i\subseteq O_i$ and $k_i=\min(C,|O_i|)$.

\begin{proposition*}[Full form of Proposition~\ref{prop:zero}]
\textbf{(a)} The matrix $P_i(0)$ is the orthogonal projector onto $\ker A_i$: at $\rho=0$, $A_iu'=0$, and
$u'$ is the point of $\ker A_i$ nearest to $u$. Let $k_i\ge1$ (then $A_i\neq0$, since every row ends in
$1$) and let $\sigma_+$ be the smallest positive singular value of $A_i$. For $\rho>0$ the residual
satisfies $\|A_iu'\|\le\frac{\rho}{\sigma_+^2+\rho}\|A_iu\|$, with equality for a right singular vector of
$A_i$ belonging to $\sigma_+$.
\par\textbf{(a$'$)} Let $k_i\ge1$ and $\rho>0$. Then $P_i(\rho)=\rho(A_i^\top A_i+\rho I)^{-1}$; the vector
$u'$ is the unique minimiser of $\|v-u\|^2+\rho^{-1}\|A_iv\|^2$ over $v\in\mathbb R^{d+1}$; and
$\|P_i(\rho)-P_i(0)\|_2=\rho/(\sigma_+^2+\rho)$, so $P_i(\rho)\to P_i(0)$ as $\rho\to0$. The bound of (a)
depends on duplicate or dependent tokens only through $\sigma_+$.
\par\textbf{(b)} At the base inputs and $\rho=0$, every pair in $S_i$ stays in $S_{++}$ and contributes zero
to all three terms, and every pair closed at base contributes only to closed$\to$open. Unit $i$'s
in-cell and open$\to$closed terms are sums over $O_i\setminus S_i$, which is empty when $C\ge|O_i|$.
\par\textbf{(c)} Let $C\ge|O_i|$ and $\rho=0$. The flat $F=\tilde w_i+\ker\tilde X_{O_i}$, which contains every
new row $\tilde w_i+u'$, is parallel to every active hyperplane $H_t$, $t\in O_i$, and lies in
$\bigcap_{t\in O_i}H_t^>$, so it meets $\KX$ exactly where it meets $\PX$. Every point where $F$ meets
$\partial\PX$ lies on a facet, and every facet through such a point is inactive
(Figure~\ref{fig:duality}f). In particular, every continuous path in $F$ that starts in $\KX$ and ends
outside $\KX$ meets $\partial\PX$.
\par\textbf{(d)} Let $S\subseteq[T]$ be any selected set, open or closed, and $u'=(I-\tilde X_S^+\tilde X_S)u$.
At the base inputs every token of $S$ keeps its pre-activation and contributes zero to $\post$, and a
closed token $t$ with $\theta-z_i(x_t)\ge\|u'\|\,\|\tilde x_t\|$ contributes zero whether it is in $S$ or
not.
\end{proposition*}

\begin{proof}
If $k_i=0$, then $A_i$ has no rows, $\ker A_i=\mathbb R^{d+1}$, and $P_i(\rho)=I=P_i(0)$ by
Definition~\ref{def:proj}; so $P_i(0)$ is the orthogonal projector onto $\ker A_i$ and $u'=u$ is the
nearest point, which is the first sentence of (a); the rest of (a), and (a$'$), assume $k_i\ge1$. The
proofs of (b), (c) and (d) below use only the first sentence of (a) and therefore also hold when
$k_i=0$.

Let $k_i\ge1$. Every row of $A_i$ ends in $1$, so $A_i\neq0$. Let $r=\operatorname{rank}A_i\ge1$ and let
$A_i=\sum_{j=1}^r\sigma_jp_jq_j^\top$ be a compact singular value decomposition, with $\sigma_j>0$ and
orthonormal families $\{p_j\}\subseteq\mathbb R^{k_i}$ and $\{q_j\}\subseteq\mathbb R^{d+1}$, so that
$\sigma_+=\min_j\sigma_j$. Let $\Pi=\sum_jq_jq_j^\top$, the orthogonal projector onto the row space
$\operatorname{span}\{q_j\}$ of $A_i$; then $I-\Pi$ is the orthogonal projector onto its orthogonal
complement, which is $\ker A_i$.

\emph{(a), the projector.} The pseudo-inverse is $A_i^+=\sum_j\sigma_j^{-1}q_jp_j^\top$, so
$A_i^+A_i=\sum_{j,l}\sigma_j^{-1}\sigma_lq_j(p_j^\top p_l)q_l^\top=\Pi$ and $P_i(0)=I-\Pi$. By the
projection theorem, $(I-\Pi)u$ is the unique point of $\ker A_i$ nearest to $u$, and $A_i(I-\Pi)u=0$
because the range of $I-\Pi$ is $\ker A_i$.

\emph{The rest of (a), and (a$'$).} Let $\rho>0$ and $E=\sum_jp_jp_j^\top$. Then
$A_iA_i^\top+\rho I=\sum_j(\sigma_j^2+\rho)p_jp_j^\top+\rho(I-E)$ acts as multiplication by
$\sigma_j^2+\rho>0$ on each $p_j$ and by $\rho>0$ on the range of $I-E$, so it is invertible, with inverse
$\sum_j(\sigma_j^2+\rho)^{-1}p_jp_j^\top+\rho^{-1}(I-E)$. Since $A_i^\top=\sum_j\sigma_jq_jp_j^\top$ vanishes
on the range of $I-E$,
\[
A_i^\top(A_iA_i^\top+\rho I)^{-1}A_i=\sum_j\frac{\sigma_j^2}{\sigma_j^2+\rho}\,q_jq_j^\top,
\qquad\text{hence}\qquad
P_i(\rho)=(I-\Pi)+\sum_j\frac{\rho}{\sigma_j^2+\rho}\,q_jq_j^\top .
\]
The same basis diagonalises $A_i^\top A_i+\rho I=\sum_j(\sigma_j^2+\rho)q_jq_j^\top+\rho(I-\Pi)$, and
$\rho$ times its inverse is $(I-\Pi)+\sum_j\rho(\sigma_j^2+\rho)^{-1}q_jq_j^\top=P_i(\rho)$. The objective
$\phi(v)=\|v-u\|^2+\rho^{-1}\|A_iv\|^2$ is a quadratic with Hessian $2I+2\rho^{-1}A_i^\top A_i$, which is
positive definite, so its unique minimiser is the unique zero of
$\nabla\phi(v)=2(v-u)+2\rho^{-1}A_i^\top A_iv$, namely
$v=(I+\rho^{-1}A_i^\top A_i)^{-1}u=\rho(A_i^\top A_i+\rho I)^{-1}u=u'$. Next, $A_iq_l=\sigma_lp_l$ and
$A_i(I-\Pi)=0$ give $A_iu'=\sum_j\frac{\rho\sigma_j}{\sigma_j^2+\rho}(q_j^\top u)\,p_j$, and orthonormality
of the $p_j$ gives
\[
\|A_iu'\|^2=\sum_j\Big(\frac{\rho}{\sigma_j^2+\rho}\Big)^2\sigma_j^2(q_j^\top u)^2
\le\Big(\frac{\rho}{\sigma_+^2+\rho}\Big)^2\sum_j\sigma_j^2(q_j^\top u)^2
=\Big(\frac{\rho}{\sigma_+^2+\rho}\Big)^2\|A_iu\|^2 ,
\]
where the inequality holds because $s\mapsto\rho/(s+\rho)$ is decreasing on $s>0$ and every
$\sigma_j\ge\sigma_+$, and the last equality because $A_iu=\sum_j\sigma_j(q_j^\top u)p_j$. For $u=q_{j^\ast}$
with $\sigma_{j^\ast}=\sigma_+$, a right singular vector belonging to $\sigma_+$, both sides equal
$(\rho/(\sigma_+^2+\rho))^2\sigma_+^2$. Finally,
$P_i(\rho)-P_i(0)=\sum_j\rho(\sigma_j^2+\rho)^{-1}q_jq_j^\top$ is symmetric with eigenvalues
$\rho/(\sigma_j^2+\rho)$ and $0$, so its operator norm is $\max_j\rho/(\sigma_j^2+\rho)=\rho/(\sigma_+^2+\rho)$,
which tends to $0$ as $\rho\to0$. Repeated or linearly dependent rows of $A_i$ change $r$ and the
$\sigma_j$, and the bound sees them only through $\sigma_+$.

\emph{(b)} At the base inputs the new pre-activation of token $t$ is $\hat z_t=z_t+\langle u',\tilde
x_t\rangle$. Let $\rho=0$ and $t\in S_i$. The number $\langle u',\tilde x_t\rangle$ is the entry of $A_iu'$
in the row of $t$, which is $0$ by (a); hence $\hat z_t=z_t>\theta$, as $S_i\subseteq O_i$. The pair is
therefore in $S_{++}$ with in-cell contribution $(\hat z_t-z_t)^2=0$, and in neither crossing set. A pair
closed at base has $z_t\le\theta$, so it lies in $S_{-+}\cup S_{--}$ and can contribute only to
closed$\to$open. The in-cell and open$\to$closed terms are sums over $S_{++}\cup S_{+-}=O_i$, and the
pairs of $S_i$ contribute zero to them, so they are sums over $O_i\setminus S_i$. Since $S_i\subseteq O_i$
and $|S_i|=k_i=\min(C,|O_i|)$, the set $O_i\setminus S_i$ is empty if and only if $C\ge|O_i|$.

\emph{(c)} Let $C\ge|O_i|$ and $\rho=0$. Then $S_i=O_i$ by (b), $A_i=\tilde X_{O_i}$, and
$u'\in\ker\tilde X_{O_i}$ by (a), so $\tilde w_i+u'\in F$. The direction space of $F$ is
$\ker\tilde X_{O_i}=\bigcap_{t\in O_i}\tilde x_t^\perp$, which is contained in the direction space
$\tilde x_t^\perp$ of each active hyperplane $H_t$: the flat is parallel to all of them. For
$v\in\ker\tilde X_{O_i}$ and $t\in O_i$, $\langle\tilde w_i+v,\tilde x_t\rangle=z_t>\theta$, so
$F\subseteq\bigcap_{t\in O_i}H_t^>$. A point of $F$ thus satisfies every active constraint strictly, so it
lies in $\KX$ if and only if it satisfies the inactive constraints, that is, if and only if it lies in
$\PX$. Let $y\in F\cap\partial\PX$. By Proposition~\ref{prop:duality}c, which rests on
Lemma~\ref{lem:faces}(iii), the point $y$ lies on a facet, and every facet through $y$ is $\PX\cap H_t$ for
a token $t$ with $y\in H_t$. Since $y\in H_s^>$ for every $s\in O_i$, $y$ lies on no active hyperplane, so
$t\notin O_i$ and the facet is inactive; by Proposition~\ref{prop:duality}c the label does not depend on
the choice of $t$. For the path, let $\alpha:[0,1]\to F$ be continuous with $\alpha(0)\in\KX$ and
$\alpha(1)\notin\KX$. By the first claim $\alpha(1)\notin\PX$, so
$s^\ast=\inf\{s\in[0,1]:\alpha(s)\notin\PX\}$ is well defined. Since $\PX$ is closed, $\alpha(0)\in\PX$ and
$\alpha(s)\in\PX$ for every $s<s^\ast$, continuity gives $\alpha(s^\ast)\in\PX$; hence $s^\ast$ is not in
the set, and by the definition of the infimum there are $s_n>s^\ast$ with $s_n\to s^\ast$ and
$\alpha(s_n)\notin\PX$. So $\alpha(s^\ast)$ is a point of $\PX$ that is a limit of points outside $\PX$,
that is, a point of $F\cap\partial\PX$.

\emph{(d)} If $S=\varnothing$, then $\tilde X_S^+\tilde X_S=0$ and $u'=u$, and the first claim is empty.
If $S\neq\varnothing$, the argument of (a), which uses no property of $A_i$ other than being a real
matrix, applies to $\tilde X_S$ and gives $\tilde X_Su'=0$. Hence $\hat z_t=z_t$ for $t\in S$, both
activations agree, and the token contributes $0$ to $\post$ (it lies in $S_{++}$ with zero difference if
it is open, and in $S_{--}$ if it is closed). For a closed $t$ with $\theta-z_t\ge\|u'\|\,\|\tilde x_t\|$,
the Cauchy--Schwarz inequality gives $\hat z_t\le z_t+\|u'\|\,\|\tilde x_t\|\le\theta$, so $t\in S_{--}$
and contributes $0$, whether or not $t\in S$.
\end{proof}

\paragraph{Theorem~\ref{thm:budget}: budget separation.}
Here the base inputs are fixed. For a unit $i$ and an update $u\in\mathbb R^{d+1}$ of row
$i$, the in-cell, open$\to$closed and closed$\to$open terms of unit $i$ are those of
Theorem~\ref{thm:decomp} with base weights $\tilde w_i$ and new weights $\tilde w_i+u$. A linear subspace
$U\subseteq\mathbb R^{d+1}$ is \emph{protecting} for unit $i$ if unit $i$'s in-cell term vanishes for every
$u\in U$, and row $i$ is \emph{confined to} $U$ if its update is required to lie in $U$. We write
$r_i=\operatorname{rank}\tilde X_{O_i}$ and $O_\cup=\bigcup_iO_i$.

\begin{theorem*}[Full form of Theorem~\ref{thm:budget}]
\textbf{(a)} A subspace $U$ is protecting for unit $i$ if and only if $U\subseteq\ker\tilde X_{O_i}$, and
then unit $i$'s open$\to$closed term also vanishes for every $u\in U$. A protecting subspace removes at
least $r_i\le|O_i|$ directions from row $i$. If $C\ge|O_i|$, the projector $P_i(0)$ confines row $i$ to
$\ker\tilde X_{O_i}$, which is protecting for unit $i$, and removes exactly $r_i$ directions.
\par\textbf{(b)} One subspace $U$ shared by every row is protecting for all units if and only if
$U\subseteq\ker\tilde X_{O_\cup}$, so it removes at least
$\operatorname{rank}\tilde X_{O_\cup}\ge\max_ir_i$ directions from every row.
\par\textbf{(c)} Let $B\ge0$ be an integer. If every row $i$ is confined to a subspace $U_i$, with
$\sum_i\operatorname{codim}U_i\le B$, the set $J$ of units $i$ for which $U_i$ is protecting satisfies
$\sum_{i\in J}r_i\le B$; conversely, for every $J\subseteq[m]$ with $\sum_{i\in J}r_i\le B$ there are such
$U_i$ under which every unit of $J$ is protected. Per-unit subspaces therefore protect up to
$\max\{|J|:\sum_{i\in J}r_i\le B\}$ units. If instead each row is either frozen ($U_i=\{0\}$, costing
$d+1$) or free ($U_i=\mathbb R^{d+1}$, costing $0$), with total cost at most $B$, the best choice protects
exactly $\#\{i:r_i=0\}+\min(\lfloor B/(d+1)\rfloor,\#\{i:r_i>0\})$ units, which is never more than the
per-unit count.
\end{theorem*}

\begin{proof}
Fix unit $i$, write $O=O_i$, and for an update $u$ let $a=\tilde X_Ou\in\mathbb R^{|O|}$, so that
$\hat z_t=z_t+a_t$ for $t\in O$.

\emph{(a)} If $O=\varnothing$, the in-cell and open$\to$closed terms are sums over the empty set and vanish
for every $u$, $\ker\tilde X_O=\mathbb R^{d+1}$ and $r_i=0$, so every claim holds. Let $O\neq\varnothing$.
If $\tilde X_Ou=0$, every $t\in O$ keeps $\hat z_t=z_t>\theta$, so $S_{++}=O$ with zero in-cell
contribution and $S_{+-}=\varnothing$; both terms vanish on $\ker\tilde X_O$, and every subspace of it is
protecting. Conversely, let $U$ be protecting and suppose, for a contradiction, that some $u\in U$ has
$a\neq0$. Put $\delta=\min_{t\in O}(z_t-\theta)$, which is positive because $O$ is finite and nonempty,
and $s=\delta/(2\|a\|_\infty)$. Then $su\in U$, and every $t\in O$ has
$|\hat z_t-z_t|=s|a_t|\le\delta/2$, so $\hat z_t\ge z_t-\delta/2>\theta$ and $t$ stays open. The in-cell
term of $su$ is therefore $\sum_{t\in O}s^2a_t^2=s^2\|a\|^2>0$, which contradicts protection. Hence
$U\subseteq\ker\tilde X_O$, and by rank--nullity $\dim U\le d+1-r_i$, that is,
$\operatorname{codim}U\ge r_i$. Finally $r_i\le|O|$, the number of rows of $\tilde X_O$.

\emph{(a), the projector.} If $C\ge|O|$, then $S_i=O$ and $A_i=\tilde X_O$ by Proposition~\ref{prop:zero}b,
and by Proposition~\ref{prop:zero}a the range of $P_i(0)$ is $\ker\tilde X_O$, which is protecting by the
first claim of (a) and has codimension $r_i$ by rank--nullity.

\emph{(b)} By (a), a subspace $U$ is protecting for every unit if and only if
$U\subseteq\bigcap_i\ker\tilde X_{O_i}$. The rows of $\tilde X_{O_\cup}$ are, up to repetition, the rows of
all the $\tilde X_{O_i}$, so $\bigcap_i\ker\tilde X_{O_i}=\ker\tilde X_{O_\cup}$. Hence
$\operatorname{codim}U\ge\operatorname{rank}\tilde X_{O_\cup}$, and
$\operatorname{rank}\tilde X_{O_\cup}\ge\operatorname{rank}\tilde X_{O_i}=r_i$ for every $i$, since the
rows of $\tilde X_{O_i}$ are among those of $\tilde X_{O_\cup}$. The subspace $U=\ker\tilde X_{O_\cup}$
attains the bound.

\emph{(c)} Let row $i$ be confined to $U_i$ for every $i$, with $\sum_i\operatorname{codim}U_i\le B$, and
let $J$ be the set of protected units. By (a), $\operatorname{codim}U_i\ge r_i$ for $i\in J$, so
$\sum_{i\in J}r_i\le\sum_{i\in J}\operatorname{codim}U_i\le B$. Conversely, let $\sum_{i\in J}r_i\le B$, and
choose $U_i=\ker\tilde X_{O_i}$ for $i\in J$ and $U_i=\mathbb R^{d+1}$ otherwise. The total codimension is
$\sum_{i\in J}r_i\le B$, and every unit of $J$ is protected by (a). Hence the largest number of units that
per-unit subspaces protect is $\max\{|J|:\sum_{i\in J}r_i\le B\}$; the protected set of the construction
also contains the units with $r_i=0$, and it is itself admissible, so the maximum is attained.

Now let each row be frozen or free. A frozen row is confined to $\{0\}\subseteq\ker\tilde X_{O_i}$ and is
protected; by (a) a free row is protected if and only if $\mathbb R^{d+1}\subseteq\ker\tilde X_{O_i}$, that
is, if and only if $r_i=0$. Let $Z=\{i:r_i=0\}$ and let $\mathcal F$ be the set of frozen rows. The budget
gives $(d+1)|\mathcal F|\le B$, that is, $|\mathcal F|\le\lfloor B/(d+1)\rfloor$, and the protected set is
$Z\cup\mathcal F$, of size
$|Z|+|\mathcal F\setminus Z|\le|Z|+\min(\lfloor B/(d+1)\rfloor,\#\{i:r_i>0\})$. Freezing
$\min(\lfloor B/(d+1)\rfloor,\#\{i:r_i>0\})$ rows with $r_i>0$ respects the budget and attains this
bound, which proves the formula. The protected set $Z\cup\mathcal F$ of any admissible choice satisfies
$\sum_{i\in Z\cup\mathcal F}r_i=\sum_{i\in\mathcal F\setminus Z}r_i\le(d+1)|\mathcal F|\le B$, because
$r_i\le d+1$, the number of columns of $\tilde X_{O_i}$; so it is admissible for the per-unit count, which
is therefore at least as large.
\end{proof}

\paragraph{Proposition~\ref{prop:kyfan}: the budget-limited optimum.}
Here $\theta=0$ and the inputs are the base inputs. For each row $i$, $U_i\subseteq\mathbb
R^{d+1}$ is a fixed linear subspace, $u_i$ is a random vector in $\mathbb R^{d+1}$ with
$\mathbb E[u_iu_i^\top]=\sigma^2I$ for a real $\sigma>0$ that is the same for every row, and the applied
update of row $i$ is $P_{U_i}u_i$. We write $\Xi_i=\tilde X_{O_i}^\top\tilde X_{O_i}$, which is $0$ if
$O_i=\varnothing$, and $\lambda_j(\Xi_i)$, $j\in[d+1]$, for its eigenvalues in non-increasing order.

\begin{proposition*}[Full form of Proposition~\ref{prop:kyfan}]
Unit $i$'s in-cell plus open$\to$closed terms are at most $\|\tilde X_{O_i}P_{U_i}u_i\|^2$, with equality
when no open token closes, and $D_i(U):=\mathbb E\|\tilde X_{O_i}P_Uu_i\|^2=\sigma^2\operatorname{tr}(\Xi_iP_U)$
for every subspace $U$.
\par\textbf{(a)} For every integer $q$ with $0\le q\le d+1$,
$\min_{\operatorname{codim}U=q}D_i(U)=\sigma^2\sum_{j>q}\lambda_j(\Xi_i)$, attained when $U^\perp$ is a
top-$q$ eigenspace of the unit's own open-token Gram $\Xi_i$.
\par\textbf{(b)} For one $U$ shared by all rows, $\sum_iD_i(U)=\sigma^2\operatorname{tr}(\bar\Xi P_U)$ with
$\bar\Xi=\sum_i\Xi_i=\sum_tn_t\tilde x_t\tilde x_t^\top$, where $n_t$ is the number of units open on token $t$;
over $\operatorname{codim}U=q$ it is minimised when $U^\perp$ is a top-$q$ eigenspace of $\bar\Xi$, with
minimum $\sigma^2\sum_{j>q}\lambda_j(\bar\Xi)$.
\par\textbf{(c)} For every integer $q$ with $0\le q\le d+1$, the minimum of $\sum_iD_i(U_i)$ over per-unit
subspaces with $\sum_i\operatorname{codim}U_i\le mq$ is at most $\sigma^2\sum_i\sum_{j>q}\lambda_j(\Xi_i)$,
the value of the allocation $\operatorname{codim}U_i=q$ for every $i$, which is at most the shared minimum
of (b).
\par\textbf{(d)} Let $U_i=\ker\tilde X_{S_i}$, the subspace to which the gate projector $P_i(0)$ confines row
$i$, so that $U_i^\perp=\operatorname{span}\{\tilde x_t:t\in S_i\}$. If these tokens span the row space of
$\tilde X_{O_i}$, in particular if $C\ge|O_i|$, then $\operatorname{codim}U_i=r_i$ and $D_i(U_i)=0$, the
minimum of $D_i$ over all subspaces.
\end{proposition*}

\begin{lemma}\label{lem:open-ub}
Let $\theta=0$, fix unit $i$ at the base inputs, and let $u\in\mathbb R^{d+1}$. Unit $i$'s in-cell plus
open$\to$closed terms for the update $u$ are at most $\|\tilde X_{O_i}u\|^2$, with equality when no token
of $O_i$ closes.
\end{lemma}

\begin{proof}
Let $t\in O_i$, so that $z_t>0$ and $\hat z_t=z_t+\langle u,\tilde x_t\rangle$. If $t$ stays open, it
contributes $(\hat z_t-z_t)^2=\langle u,\tilde x_t\rangle^2$ to the in-cell term. If $t$ closes, then
$\hat z_t\le0<z_t$, so $z_t-\hat z_t\ge z_t>0$, and its contribution $z_t^2$ to the open$\to$closed term is
at most $(z_t-\hat z_t)^2=\langle u,\tilde x_t\rangle^2$. Tokens off $O_i$ enter neither term. Summing over
$t\in O_i$ gives the bound $\sum_{t\in O_i}\langle u,\tilde x_t\rangle^2=\|\tilde X_{O_i}u\|^2$, and every
step is an equality when no token of $O_i$ closes.
\end{proof}

\begin{lemma}[Ky Fan]\label{lem:kyfan}
Let $\Xi\in\mathbb R^{n\times n}$ be symmetric positive semidefinite with eigenvalues
$\lambda_1\ge\dots\ge\lambda_n$, and let $Q\in\mathbb R^{n\times n}$ be an orthogonal projector of rank $q$,
$0\le q\le n$. Then $\operatorname{tr}(\Xi Q)\le\sum_{j\le q}\lambda_j$, with equality when $Q$ projects onto a
top-$q$ eigenspace of $\Xi$.
\end{lemma}

\begin{proof}
Let $e_1,\dots,e_n$ be orthonormal eigenvectors with $\Xi e_j=\lambda_je_j$, so that
$\Xi=\sum_j\lambda_je_je_j^\top$ and $\operatorname{tr}(\Xi Q)=\sum_j\lambda_jc_j$ with $c_j=e_j^\top Qe_j$.
Since $Q$ is symmetric and idempotent, $c_j=\|Qe_j\|^2\in[0,1]$, and $\sum_jc_j=\operatorname{tr}Q=q$. If
$q=0$, then $Q=0$ and both sides are $0$. If $q\ge1$,
\[
\sum_j\lambda_jc_j-\sum_{j\le q}\lambda_j=\sum_{j\le q}\lambda_j(c_j-1)+\sum_{j>q}\lambda_jc_j
\le\lambda_q\Bigl(\sum_{j\le q}(c_j-1)+\sum_{j>q}c_j\Bigr)=\lambda_q\Bigl(\sum_jc_j-q\Bigr)=0,
\]
where for $j\le q$ we used $c_j-1\le0$ and $\lambda_j\ge\lambda_q$, and for $j>q$ we used $c_j\ge0$ and
$\lambda_j\le\lambda_q$. If $Q$ projects onto the span of orthonormal eigenvectors $f_1,\dots,f_q$ with
eigenvalues $\lambda_1,\dots,\lambda_q$, then
$\operatorname{tr}(\Xi Q)=\sum_{l\le q}f_l^\top\Xi f_l=\sum_{l\le q}\lambda_l$.
\end{proof}

\begin{proof}[Proof of the proposition]
The bound and its equality case are Lemma~\ref{lem:open-ub} applied to the update $P_{U_i}u_i$. For the
expectation, let $U$ be a subspace. Since $P_U$ is symmetric and idempotent,
$\|\tilde X_{O_i}P_Uu_i\|^2=u_i^\top P_U\Xi_iP_Uu_i=\operatorname{tr}(P_U\Xi_iP_U\,u_iu_i^\top)$, and linearity of
the trace and of the expectation gives
\[
D_i(U)=\operatorname{tr}\bigl(P_U\Xi_iP_U\,\mathbb E[u_iu_i^\top]\bigr)=\sigma^2\operatorname{tr}(P_U\Xi_iP_U)
=\sigma^2\operatorname{tr}(\Xi_iP_U^2)=\sigma^2\operatorname{tr}(\Xi_iP_U),
\]
where the third equality is cyclicity of the trace.

\emph{(a)} Let $\operatorname{codim}U=q$ and $Q=I-P_U$, the orthogonal projector onto $U^\perp$, of rank
$q$. By Lemma~\ref{lem:kyfan},
$\operatorname{tr}(\Xi_iP_U)=\operatorname{tr}\Xi_i-\operatorname{tr}(\Xi_iQ)\ge\sum_j\lambda_j(\Xi_i)-\sum_{j\le
q}\lambda_j(\Xi_i)=\sum_{j>q}\lambda_j(\Xi_i)$, with equality when $U^\perp$, the range of $Q$, is a top-$q$
eigenspace of $\Xi_i$. Such a subspace exists, since $\Xi_i$ has an orthonormal eigenbasis, so the minimum is
attained and equals $\sigma^2\sum_{j>q}\lambda_j(\Xi_i)$.

\emph{(b)} For a shared $U$, linearity gives
$\sum_iD_i(U)=\sigma^2\sum_i\operatorname{tr}(\Xi_iP_U)=\sigma^2\operatorname{tr}(\bar\Xi P_U)$. Exchanging the
order of summation,
$\bar\Xi=\sum_i\sum_{t\in O_i}\tilde x_t\tilde x_t^\top=\sum_t\#\{i:t\in O_i\}\,\tilde x_t\tilde x_t^\top=\sum_tn_t\tilde x_t\tilde x_t^\top$.
The matrix $\bar\Xi$ is symmetric positive semidefinite, and the argument of (a) with $\bar\Xi$ in place of
$\Xi_i$ gives the minimiser and the minimum.

\emph{(c)} The allocation $\operatorname{codim}U_i=q$ for every $i$ has total codimension $mq$, so it is
admissible, and choosing each $U_i$ by (a) gives the total $\sigma^2\sum_i\sum_{j>q}\lambda_j(\Xi_i)$; the
minimum over admissible allocations is at most this value. For every subspace $U$ of codimension $q$,
$\sum_iD_i(U)\ge\sum_i\min_{\operatorname{codim}U_i=q}D_i(U_i)=\sigma^2\sum_i\sum_{j>q}\lambda_j(\Xi_i)$,
because each term is at least its own minimum; taking the minimum over $U$ gives the second inequality.
Since the allocation used has total codimension exactly $mq$, the comparison also holds among per-unit
subspaces of the same total codimension as the shared one.

\emph{(d)} Since $S_i\subseteq O_i$, $\operatorname{span}\{\tilde x_t:t\in S_i\}$ is contained in the row
space of $\tilde X_{O_i}$, and by hypothesis the two are equal. Hence
$U_i=\ker\tilde X_{S_i}=(\operatorname{span}\{\tilde x_t:t\in S_i\})^\perp$ equals the orthogonal complement
of the row space of $\tilde X_{O_i}$, which is $\ker\tilde X_{O_i}$, of codimension $r_i$ by rank--nullity.
Every $v\in U_i$ has $\tilde X_{O_i}v=0$, so $\tilde X_{O_i}P_{U_i}=0$, hence $\Xi_iP_{U_i}=0$ and
$D_i(U_i)=0$. Since $D_i(U)\ge0$ for every $U$, this is the minimum. If $C\ge|O_i|$, then $S_i=O_i$ by
Proposition~\ref{prop:zero}b, and the hypothesis holds.
\end{proof}

\paragraph{Proposition~\ref{prop:freed}: the freed term is a near-wall term.}
Here unit $i$, the base inputs and an update $u\in\mathbb R^{d+1}$ of row $i$ are fixed, so
that token $t$ has new pre-activation $\hat z_t=z_i(x_t)+a_t$ with $a_t=\langle u,\tilde x_t\rangle$. For
a closed token $t\notin O_i$, $\gamma_t=\theta-z_i(x_t)\ge0$ is its margin. A closed token \emph{opens} if
$\hat z_t>\theta$, and the closed$\to$open term is $\sum_{t\in S_{-+}}\hat z_t^2$.

\begin{proposition*}[Full form of Proposition~\ref{prop:freed}]
\textbf{(a)} A closed token $t$ opens if and only if $a_t>\gamma_t$, and the closed$\to$open term equals
$\sum_{t\notin O_i,\,a_t>\gamma_t}(\theta+a_t-\gamma_t)^2$.
\par\textbf{(b)} Only near-wall tokens, those of $N(u)=\{t\notin O_i:\gamma_t<\|u\|\,\|\tilde x_t\|\}$, can open.
\par\textbf{(c)} At $\theta=0$ the term is $\sum_{t\notin O_i}(a_t-\gamma_t)_+^2\le\|\tilde X_{N(u)}u\|^2$.
\end{proposition*}

\begin{proof}
\emph{(a)} For a closed token, $\hat z_t=z_i(x_t)+a_t=\theta-\gamma_t+a_t$, which exceeds $\theta$ if and only
if $a_t>\gamma_t$. The opened tokens form $S_{-+}$, and each contributes $\hat z_t^2=(\theta+a_t-\gamma_t)^2$.

\emph{(b)} By the Cauchy--Schwarz inequality $a_t\le\|u\|\,\|\tilde x_t\|$, so an opened token, which has
$a_t>\gamma_t$, satisfies $\gamma_t<\|u\|\,\|\tilde x_t\|$ and lies in $N(u)$.

\emph{(c)} At $\theta=0$, (a) gives the term $\sum_{a_t>\gamma_t}(a_t-\gamma_t)^2=\sum_{t\notin O_i}(a_t-\gamma_t)_+^2$.
An opened token has $0<a_t-\gamma_t\le a_t$, because $\gamma_t\ge0$, and lies in $N(u)$ by (b); so the term is at
most $\sum_{\mathrm{opened}}a_t^2\le\sum_{t\in N(u)}a_t^2=\|\tilde X_{N(u)}u\|^2$.
\end{proof}

\paragraph{Proposition~\ref{prop:firstorder}: first-order criteria inherit the partition.}
Here the layer's inputs $\tilde x_t$ are fixed; the matrix $Z=(z_{t,i})\in\mathbb R^{T\times m}$ has
$z_{t,i}=\langle\tilde w_i,\tilde x_t\rangle$, the post-activations are $h=(h_{t,i})=\act_\theta(Z)$
entrywise, and $L=g(h)$ for a function $g:\mathbb R^{T\times m}\to\mathbb R$, the downstream network per
token or per sequence, whose partial derivatives exist at $h$. With the convention $\act_\theta'(\theta):=0$ of automatic differentiation,
$\partial L/\partial z_{t,i}:=\act_\theta'(z_{t,i})\,\partial g/\partial h_{t,i}$ and
$\nabla_{\tilde w_i}L:=\sum_t(\partial L/\partial z_{t,i})\,\tilde x_t$ are the values that automatic
differentiation returns.

\begin{proposition*}[Full form of Proposition~\ref{prop:firstorder}]
\textbf{(a)} The derivative is $\partial L/\partial z_{t,i}=\mathds 1[z_{t,i}>\theta]\,\partial g/\partial
h_{t,i}$, the true partial derivative wherever $z_{t,i}\neq\theta$; it vanishes on every closed pair.
\par\textbf{(b)} The row gradient $\nabla_{\tilde w_i}L=\sum_{t\in O_i}(\partial L/\partial z_{t,i})\,\tilde x_t$
lies in the row space of $\tilde X_{O_i}$, and it is $0$ for a unit open on no token.
\par\textbf{(c)} The Fisher $F_i=\mathbb E[\nabla_{\tilde w_i}\ell\,\nabla_{\tilde w_i}\ell^\top]$ of row $i$,
over per-sample losses $\ell$ of this form, annihilates $\ker\tilde X_{O_i}$, with $O_i$ taken over the
Fisher data.
\par\textbf{(d)} EWC's diagonal $F_{i,jj}=\mathbb E[(\sum_{t\in O_i}\partial\ell/\partial z_{t,i}\,\tilde
x_{t,j})^2]$, $j\in[d+1]$, and its per-token form depend on the data only through open pairs: every term
they contain is indexed by an open pair of unit $i$. Both vanish for a unit open on no token.
\par\textbf{(e)} A selector that takes the $k_i$ largest $|\partial L/\partial z_{t,i}|$ of unit $i$ picks only
open pairs whenever at least $k_i$ open pairs of unit $i$ have $\partial g/\partial h_{t,i}\neq0$.
\end{proposition*}

\begin{proof}
The inputs $\tilde x_t$ of the layer do not depend on its own weights, and $L$ depends on $Z$ only through
$h=\act_\theta(Z)$.

\emph{(a)} On $\{v\in\mathbb R:v\neq\theta\}$ the scalar map $\act_\theta$ is locally affine with derivative
$\mathds 1[v>\theta]$; at $v=\theta$ it jumps if $\theta>0$ and has a kink if $\theta=0$. Let
$z_{t,i}\neq\theta$. For $|\eta|<|z_{t,i}-\theta|$, changing $z_{t,i}$ to $z_{t,i}+\eta$ changes only the
entry $h_{t,i}$ of $h$, and by $\eta\,\mathds 1[z_{t,i}>\theta]$, so the chain rule gives
$\partial L/\partial z_{t,i}=\mathds 1[z_{t,i}>\theta]\,\partial g/\partial h_{t,i}$, which is $0$ when
$z_{t,i}<\theta$. With $\act_\theta'(\theta):=0$ the same formula holds at pairs with $z_{t,i}=\theta$ and
gives $0$ there, so the derivative vanishes on the whole closed set $\{z_{t,i}\le\theta\}$.

\emph{(b)} By definition and (a),
$\nabla_{\tilde w_i}L=\sum_t(\partial L/\partial z_{t,i})\tilde x_t=\sum_{t\in O_i}(\partial L/\partial
z_{t,i})\tilde x_t$, a linear combination of the rows of $\tilde X_{O_i}$, and the empty combination $0$
when $O_i=\varnothing$.

\emph{(c)} Each sample has its own tokens and downstream map, the expectation defining $F_i$ is assumed
to exist, and $O_i$ is the set of tokens, over all samples, on which unit $i$ is open, so that $\ker\tilde X_{O_i}=\{v:\langle v,\tilde x_t\rangle=0\text{ for
all }t\in O_i\}$. For such $v$, (b) applied to each sample gives $\langle\nabla_{\tilde w_i}\ell,v\rangle=0$
for every sample, hence $F_iv=\mathbb E[\nabla_{\tilde w_i}\ell\,\langle\nabla_{\tilde w_i}\ell,v\rangle]=0$.

\emph{(d)} The $j$-th diagonal entry of $F_i$ is $\mathbb E[(\nabla_{\tilde w_i}\ell)_j^2]$, which by (b)
equals $\mathbb E[(\sum_{t\in O_i}\partial\ell/\partial z_{t,i}\,\tilde x_{t,j})^2]$, with $O_i$ the open
set of the sample; every term is indexed by an open pair, and the entry is $0$ when unit $i$ is open on
no token. In the per-token form $\mathbb E_t[(\partial\ell_t/\partial z_{t,i}\,\tilde x_{t,j})^2]$ the terms of
closed pairs are $0$ by (a).

\emph{(e)} By (a), every closed pair has $|\partial L/\partial z_{t,i}|=0$, and an open pair has
$|\partial L/\partial z_{t,i}|=|\partial g/\partial h_{t,i}|$. If at least $k_i$ open pairs have
$\partial g/\partial h_{t,i}\neq0$, then the $k_i$-th largest value is positive, so every one of the $k_i$
selected pairs, whatever the tie-breaking, has a positive value and is open.
\end{proof}

\begin{lemma}[Gated units]\label{lem:glu}
Let $\theta\ge0$, and let unit $i$ of a gated layer output $y_t=\act_\theta(g_t)\,v_t$ with
$g_t=\langle\tilde w^g,\tilde x_t\rangle$ and $v_t=\langle\tilde w^v,\tilde x_t\rangle$ for gate and up
weights $\tilde w^g,\tilde w^v\in\mathbb R^{d+1}$, and let $O=\{t:g_t>\theta\}$. If the updates
$u^g,u^v\in\mathbb R^{d+1}$ satisfy $\tilde X_Ou^g=\tilde X_Ou^v=0$, then at the base inputs
$\hat y_t=y_t$ for every $t\in O$, and for $t\notin O$, $\hat y_t\neq y_t=0$ only if $\hat g_t>\theta$.
\end{lemma}

\begin{proof}
At the base inputs $\hat g_t=g_t+\langle u^g,\tilde x_t\rangle$, $\hat v_t=v_t+\langle u^v,\tilde
x_t\rangle$ and $\hat y_t=\act_\theta(\hat g_t)\hat v_t$. For $t\in O$ the hypothesis gives $\hat g_t=g_t>\theta$
and $\hat v_t=v_t$, so $\hat y_t=g_tv_t=y_t$. For $t\notin O$, $g_t\le\theta$ gives $y_t=0$, and
$\hat y_t=\act_\theta(\hat g_t)\hat v_t$ is $0$ unless $\hat g_t>\theta$.
\end{proof}

%% file: app_cl_setup.tex
\section{Fine-tuning setup and implementation}\label{app:cl:setup}

\paragraph{Models and text.} Every run updates only the \texttt{fc1} matrix (weights and bias) of every decoder
block of the pre-LN OPT-125m or OPT-1.3b model \citep{zhang2022opt}, in \texttt{fp32}; its rows are the wall units
of Definition~\ref{def:wall}. On OPT-6.7b it updates the \texttt{fc1} matrices of the last four of the $32$ blocks,
in \texttt{fp32}, with the rest of the model in \texttt{bfloat16}; its units fire on $6.9\%$ of the support tokens
(a mean of $140.7$ of $2\,048$, against $72.1$ on OPT-1.3b), and cap $128$ gives $\bar k=68.7$ and
preserves $67.4\%$ of the active post-activation energy exactly. Text sets are disjoint character ranges. The new
task trains on characters $[0,\,400\,000)$ of its corpus file and is evaluated on $[600\,000,\,1\,200\,000)$
($64$ windows of $128$ tokens); the old support $X_{\mathrm{old}}$, on which every projector is built, is MiniPile
$[0,\,400\mathrm{k})$ \citep{kaddour2023minipile}, $16\times128=2\,048$ tokens, and the old probe, on which
forgetting is scored, is MiniPile $[600\mathrm{k},\,1.0\mathrm{M})$, $8\times128=1\,024$ tokens.

\paragraph{Training and measurements.} The new task is $512$ fresh $128$-token windows presented once, four per
AdamW step (accumulation $4$, $128$ steps), with $10$ warmup steps, cosine decay to zero and peak learning rate
$10^{-4}$ on every corpus and method. Data order is permuted under the seed, so two methods on one seed see the same
windows in the same order; seeds are $0$--$2$, and $0$--$5$ on the settings with $n=6$. Every number is read at step
$128$: \emph{forgetting} is the old probe's NLL minus its base value and \emph{learning} the held-out new-task NLL
minus its base value (more negative is better), and every run logs the layer form of Eq.~\eqref{eq:identity} at
every evaluation (Appendix~\ref{app:identity}). Base-model difficulty is $0.80$ nats
per character on WikiText-103 \citep{merity2017pointer} (\texttt{wt103}) and $1.02$ on \texttt{python}, Python
files from the validation split of CodeParrot-clean \citep{codeparrot2022clean}. On \texttt{wt103} unconstrained
fine-tuning of OPT-1.3b forgets $+0.2684\pm0.0192$ and learns $-0.5878\pm0.0141$ over three seeds, and OPT-125m
forgets $+0.1958$; on \texttt{python} it forgets $+0.1154$, while every constrained method ends below base on the old
probe, so there a more negative forgetting number is a larger improvement.

\paragraph{The projector.} It is built once from the base model. For each wall the base inputs on
$X_{\mathrm{old}}$, with a ones column appended so that the bias is projected with the weights, form
$X_a\in\mathbb R^{T\times(d+1)}$, $T=2\,048$. For unit $i$, with base post-activations
$h^{(i)}_t=\act_\theta(z_i(x_t))$ and open set $O_i=\{t: z_i(x_t)>\theta\}$, $S_i$ holds the tokens of the $C$
largest positive $h^{(i)}_t$, so $k_i=|S_i|=\min(C,|O_i|)$ as in Eq.~\eqref{eq:ki}, and $A_i=X_a[S_i,:]$. The Gram
$\Gamma_i=A_iA_i^{\top}$ is gathered from $X_aX_a^{\top}$ into a $C\times C$ array
$\widetilde\Gamma_i=\operatorname{diag}(\Gamma_i,I)$ padded by the identity, so that its factorisation is defined
for every unit, and the ridge is $\rho_i=10^{-4}\,C^{-1}\operatorname{tr}\widetilde\Gamma_i$, that is $10^{-4}$
times the mean of the padded diagonal. The Cholesky factor of $\widetilde\Gamma_i+\rho_iI$ is stored per wall in
\texttt{fp16}, $2mC^2$ bytes, and cast to \texttt{fp32} for each solve; caps $256$ and $512$, and OPT-6.7b, re-factorise inside
each step in blocks of $1\,024$ units instead, and at $C=128$ the two paths agree to $12$ decimal places. After
each optimiser step, with $D$ the cumulative displacement of the wall and its bias (row $u_i^{\top}=(dW_i,db_i)$),
unit $i$'s entries of $DX_a^{\top}$ at $S_i$, padded with zeros, form $r_i$; $y_i=(\widetilde\Gamma_i+\rho_iI)^{-1}r_i$
is scattered back to the positions $S_i$ of a matrix $Y$; and $D\leftarrow D-\alpha\,YX_a$. Because the padding
block is decoupled and carries a zero right-hand side, row $i$ of this update is
\begin{equation*}
u_i\;\leftarrow\;u_i-\alpha\,A_i^{\top}\bigl(\Gamma_i+\rho_iI\bigr)^{-1}A_i\,u_i ,
\end{equation*}
which is Eq.~\eqref{eq:proj}, with $\alpha=1$ for every gate-aware run. The residual the ridge leaves at the
stored base inputs, the summed $|\hat z-z|$ on the constrained pairs as a share of the total $|\hat z-z|$ mass, is
$5.4\times10^{-9}$ to $4.0\times10^{-8}$ on OPT-1.3b at caps $16$ to $128$, between $3.8\times10^{-9}$ and
$3.5\times10^{-7}$ over all $48$ cap-ladder runs on both models, and at most $1.3\times10^{-8}$ on every run of the
necessity comparison (\S\ref{sec:necessity}). At fixed inputs every pair in $S_i$ keeps its base post-activation,
so the cap fixes the share $\sum_i\sum_{t\in S_i}(h^{(i)}_t)^2\big/\sum_i\sum_t(h^{(i)}_t)^2$ of the active
post-activation energy preserved exactly (per wall, averaged over walls): $59.6\%$ at $C=64$ and $71.7\%$ at
$C=128$ on OPT-1.3b.

\paragraph{Baselines and run names.} Adam-NSCL \citep{wang2021adamnscl}, restricted to the walls, diagonalises
the uncentred second moment $T^{-1}X_a^{\top}X_a$ of the same augmented support inputs, forbids its top
$k_{\mathrm{sh}}$ eigendirections in every row, and after each optimiser step replaces the cumulative displacement
by $(1-\alpha)D+\alpha DP$, with $P$ the projector onto their complement. The random-subspace control replaces the
data-derived subspace by a random orthonormal one of exactly the same dimension. Tables name the unconstrained
run \texttt{B1}, a shared-subspace run by its dimension (\texttt{B16k24}: $k_{\mathrm{sh}}=24$) and a gate-aware run
by its cap (\texttt{B17c64}: $C=64$); the selector controls, which keep each unit's $k_i$ and change only which
support tokens fill it, carry a one-letter suffix (Appendix~\ref{app:cl:necessity}).

\subsection{The freed term on held-out text}\label{app:cl:crossmass}

On the old probe, which neither the projector nor the new task sees, every constrained method lowers the freed
closed$\to$open mass to $16$--$25\%$ of unconstrained fine-tuning's in the mean, and lowers its share of the
post-activation change, on all three seeds (Table~\ref{tab:app:crossmass}).

\begin{table}[ht]\centering\footnotesize
\setlength{\tabcolsep}{5pt}
\caption{\textbf{Closed$\to$open mass on the held-out old probe.} OPT-1.3b, WikiText-103 at lr $10^{-4}$,
\texttt{fc1}, seeds $0$--$2$. ``co mass'' is the run's mean level; the difference columns are mean paired per-seed
differences against the unconstrained run, negative on $3$ of $3$ seeds in every row. ``Share'' is of the
probe's total squared post-activation change.}
\label{tab:app:crossmass}
\begin{tabular}{llrrrr}
\toprule
run & constraint & dims/row & co mass & $\Delta$ mass & $\Delta$ share \\
\midrule
\texttt{B1} & none & $0$ & $9\,239$ & --- & --- \\
\midrule
\texttt{B17c16}  & gate, $C=16$  & $9.1$  & $2\,286$ & $-6\,953$ & $-0.0537$ \\
\texttt{B17c64}  & gate, $C=64$  & $23.5$ & $1\,785$ & $-7\,454$ & $-0.0433$ \\
\texttt{B17c128} & gate, $C=128$ & $34.4$ & $1\,695$ & $-7\,544$ & $-0.0366$ \\
\midrule
\texttt{B16k9}   & shared, $k_{\mathrm{sh}}=9$   & $9.0$   & $2\,196$ & $-7\,043$ & $-0.0772$ \\
\texttt{B16k24}  & shared, $k_{\mathrm{sh}}=24$  & $24.0$  & $1\,622$ & $-7\,617$ & $-0.0910$ \\
\texttt{B16k34}  & shared, $k_{\mathrm{sh}}=34$  & $34.0$  & $1\,506$ & $-7\,733$ & $-0.0927$ \\
\bottomrule
\end{tabular}
\end{table}

\subsection{Statistics}\label{app:cl:stats}

Two methods on one model and corpus are compared by a paired $t$-test on their per-seed differences (mean, $95\%$
interval, $p$, number of seeds with the predicted sign), one-sided where the construction predicts the direction
(ours forgetting less than the method it is compared with, or scoring above it on the zero-shot tasks; the gate
forgetting less than a control selector at identical $k_i$) and two-sided otherwise, including every learning
comparison and the cap-$128$ capacity comparisons. Forgetting removed is the paired reduction against the
unconstrained run of the same model, divided by that run's forgetting. Equivalence is two one-sided tests (TOST)
with margin $\delta=0.0068$ nats, which is $2.5\%$ of the unconstrained forgetting of OPT-1.3b on WikiText-103
($0.2684$) and $5.4\%$ of that of OPT-6.7b ($0.1256$), and equals the probe noise floor of the repair
experiments (Appendix~\ref{app:noise}): the $90\%$ interval of the paired difference lies inside
$(-\delta,\delta)$. A table of several rows is also tested as a whole, by a
sign-flip permutation test over all its seed-pairs (each difference standardised by its row's paired standard
deviation, the sum referred to all $2^n$ sign assignments for $n\le22$ pairs and to $2\times10^6$ random ones
otherwise) and by a DerSimonian--Laird random-effects meta-analysis of the row means (pooled difference, $95\%$
interval, between-row variance $\tau^2$).

%% file: app_cl_budget.tex
\section{Matched budget and capacity in full}\label{app:cl:budget}

Table~\ref{tab:app:matched} pairs each cap, with mean per-unit dimension $\bar k$ (Eq.~\eqref{eq:ki}), with
Adam-NSCL at $k_{\mathrm{sh}}=\mathrm{round}(\bar k)$, so both remove the same number of directions per row
(training setup and tests of Appendix~\ref{app:cl:setup}). On OPT-1.3b the ratio of the two forgetting levels, means over each row's seeds, grows
from $1.1$ at $9$ dimensions per row to $4.3$ at $60$: between $15$ and $47$ dimensions per row Adam-NSCL stays at
$0.046$--$0.050$ nats (seeds $0$--$2$), while ours falls from $0.047$ to $0.011$ at $47$ and $0.010$ at $60$, and
learning differs by at most $0.0026$ nats. On OPT-6.7b, cap $128$ at $68.7$ dimensions per row forgets $0.0192$ against Adam-NSCL's $0.0184$ at $69$ (TOST at
$\delta=0.0068$: 90\% CI $[-0.0001,+0.0018]$, equivalent) and learns $0.0109$ nats more on all three seeds.

\begin{table}[ht]\centering\footnotesize
\setlength{\tabcolsep}{1.6pt}
\caption{\textbf{Matched constrained dimensions per row.} Ours at cap $C$ (mean per-unit dimension $\bar k$)
against Adam-NSCL at $k_{\mathrm{sh}}$, paired by seed; $\Delta$ is ours minus Adam-NSCL, so a negative $\Delta$
is ours ahead. Forgetting: one-sided paired $t$-test; learning: two-sided; $\Delta<0$ counts the seed-pairs on
which ours forgets less. Bold: the 95\% CI excludes zero with ours ahead. WikiText-103, \texttt{fc1}, lr
$10^{-4}$, $T=2\,048$; each block has its own no-constraint run, and the pooled lines combine the rows named.}
\label{tab:app:matched}
\begin{tabular}{rrrrrcrrrcr}
\toprule
 & & & & \multicolumn{4}{c}{forgetting} & \multicolumn{3}{c}{learning}\\
\cmidrule(lr){5-8}\cmidrule(lr){9-11}
$C$ & $\bar k$ & $k_{\mathrm{sh}}$ & $n$ & $\Delta$ & 95\% CI & $p$ & $\Delta<0$ & $\Delta$ & 95\% CI & $p$\\
\midrule
\multicolumn{11}{l}{\emph{OPT-1.3b} (gate density $0.0392$, ambient dim.\ $2\,049$; no constraint forgets $+0.2684$)}\\
$16$  & $9.067$  & $9$  & $6$ & $-0.0068$ & $[-0.0246,+0.0109]$ & $0.1839$ & $4/6$ & $+0.0011$ & $[-0.0011,+0.0033]$ & $0.2691$\\
$32$  & $15.048$ & $15$ & $3$ & $-0.0027$ & $[-0.0126,+0.0073]$ & $0.1845$ & $2/3$ & $+0.0011$ & $[-0.0082,+0.0104]$ & $0.6679$\\
$64$  & $23.519$ & $24$ & $6$ & $\mathbf{-0.0111}$ & $[-0.0147,-0.0074]$ & $0.0003$ & $6/6$ & $+0.0015$ & $[-0.0002,+0.0033]$ & $0.0769$\\
$128$ & $34.421$ & $34$ & $6$ & $\mathbf{-0.0195}$ & $[-0.0247,-0.0143]$ & $0.0001$ & $6/6$ & $+0.0025$ & $[-0.0002,+0.0051]$ & $0.0654$\\
$256$ & $47.257$ & $47$ & $3$ & $\mathbf{-0.0357}$ & $[-0.0481,-0.0232]$ & $0.0033$ & $3/3$ & $+0.0011$ & $[-0.0004,+0.0025]$ & $0.0844$\\
$512$ & $60.332$ & $60$ & $3$ & $\mathbf{-0.0325}$ & $[-0.0485,-0.0164]$ & $0.0065$ & $3/3$ & $+0.0026$ & $[+0.0012,+0.0040]$ & $0.0146$\\
\addlinespace[2pt]
\multicolumn{11}{l}{\emph{OPT-125m} (gate density $0.0606$, ambient dim.\ $769$; no constraint forgets $+0.1958$)}\\
$8$   & $5.700$  & $6$  & $3$ & $+0.0348$ & $[+0.0279,+0.0417]$ & $0.9990$ & $0/3$ & $+0.0066$ & $[+0.0008,+0.0124]$ & $0.0389$\\
$32$  & $17.022$ & $17$ & $6$ & $+0.0007$ & $[-0.0039,+0.0052]$ & $0.6373$ & $2/6$ & $\mathbf{-0.0226}$ & $[-0.0249,-0.0203]$ & $<10^{-4}$\\
$128$ & $42.080$ & $42$ & $6$ & $\mathbf{-0.0047}$ & $[-0.0073,-0.0021]$ & $0.0027$ & $6/6$ & $\mathbf{-0.0084}$ & $[-0.0100,-0.0069]$ & $<10^{-4}$\\
$256$ & $62.047$ & $62$ & $6$ & $\mathbf{-0.0102}$ & $[-0.0138,-0.0067]$ & $0.0003$ & $6/6$ & $-0.0003$ & $[-0.0024,+0.0019]$ & $0.7599$\\
\addlinespace[2pt]
\multicolumn{11}{l}{\emph{OPT-6.7b}, last four blocks (gate density $0.0677$, ambient dim.\ $4\,097$; no constraint forgets $+0.1256$)}\\
$128$ & $68.724$ & $69$ & $3$ & $+0.0008$ & $[-0.0006,+0.0023]$ & $0.9359$ & $0/3$ & $\mathbf{-0.0109}$ & $[-0.0115,-0.0103]$ & $0.0001$\\
\bottomrule
\end{tabular}

\smallskip
\parbox{\linewidth}{\footnotesize \emph{Pooled forgetting} (sign-flip permutation over seed-pairs,
one-sided; random effects, DerSimonian--Laird). OPT-1.3b, six rows: $24/27$ lower, $p=1.5\times10^{-6}$
($2\times10^6$ Monte Carlo); $\Delta=-0.0182$ $[-0.0279,-0.0086]$, $\tau^2=1.3\times10^{-4}$. OPT-125m
from $17$ dims/row, three rows: $14/18$ lower, exact $p=5.1\times10^{-4}$; $\Delta=-0.0049$
$[-0.0102,+0.0005]$.}
\end{table}

\subsection{Capacity}\label{app:cl:capacity}

On OPT-1.3b we set our caps against Adam-NSCL at a uniform $k_{\mathrm{sh}}=259$ in every row
(\texttt{B16k259}), the mean dimension GPM's energy criterion selects at $0.95$ on the same support ($259.2$), so
ours constrains $7.5\times$ fewer dimensions per row at cap $128$ ($\bar k=34.4$) and $5.5\times$ fewer at cap $256$
($\bar k=47.3$). At cap $256$ ours forgets $0.0110$ against Adam-NSCL's $0.0255$ on WikiText-103, a
paired difference of $-0.0145$ (95\% CI $[-0.0229,-0.0062]$, one-sided $p=0.0087$, $3$ of $3$), and learns
$0.0518$ nats more ($[-0.0576,-0.0460]$, $p=0.0007$). At cap $128$, on Python ours forgets $0.0216$ nats less
(95\% CI $[-0.0343,-0.0090]$, two-sided $p=0.0180$) and learns $0.0265$ nats more ($[-0.0380,-0.0151]$,
$p=0.0099$), each on $3$ of $3$ seed-pairs; on WikiText-103, over $6$ seed-pairs, it learns $0.0523$ nats more
($[-0.0546,-0.0499]$, $p<10^{-4}$, $6$ of $6$) at a forgetting difference of $-0.0065$ ($[-0.0136,+0.0006]$,
two-sided $p=0.0645$), negative on $5$ of $6$. Against no constraint, cap $256$ removes $95.9\%$ of OPT-1.3b's
forgetting, where Adam-NSCL removes $82.6\%$ at $47$ dimensions per row and $90.5\%$ at $259$, and $97.7\%$ of
OPT-125m's, each method on $3$ of $3$ seeds.

\subsection{The reachable budget}\label{app:cl:density}

\begin{table}[ht]\centering\footnotesize
\caption{\textbf{The reachable budget: base gate density and per-unit open sets.} $p$ is the fraction of
(token, unit) pairs open at the base weights on the $1\,024$-token old probe. $\mathrm{open}_i=|O_i|$ is
counted directly on the $T=2\,048$-token support that builds the projector; its mean
$\overline{\mathrm{open}}$ over units is the right-hand side of Eq.~\eqref{eq:cl:ceiling} and bounds
$\bar k$ ($p\,T$ is about $11\%$ higher because $p$ is measured on the probe). ``Ambient'' is the row
dimension $d+1$.}
\label{tab:app:density}
\begin{tabular}{lrrrrrrrrr}
\toprule
 & & & & & \multicolumn{5}{c}{$\mathrm{open}_i$ over units}\\
\cmidrule(lr){6-10}
model & $p$ & ambient & $\overline{\mathrm{open}}$ & $p\,T$ & median & p25 & p99 & never open & mean/median\\
\midrule
OPT-1.3b & $0.0392$ & $2\,049$ & $72.1$  & $80$  & $11$ & $1$ & $954$    & $19.5\%$ & $6.6$\\
OPT-125m & $0.0606$ & $769$    & $111.3$ & $124$ & $16$ & $2$ & $1\,320$ & $14.8\%$ & $7.0$\\
\bottomrule
\end{tabular}
\end{table}

Since $k_i\le|O_i|$ for every unit, the mean constrained dimension per row is bounded by the mean open-set size
on the support,
\begin{equation}\label{eq:cl:ceiling}
\bar k\;=\;\frac1m\sum_{i=1}^{m}k_i\;\le\;\frac1m\sum_{i=1}^{m}|O_i|\;=\;p_{\mathrm{sup}}\,T ,
\end{equation}
with $p_{\mathrm{sup}}$ the base gate density counted on the support: at $T=2\,048$ this reachable budget is
$72.1$ dimensions per row on OPT-1.3b and $111.3$ on OPT-125m (Table~\ref{tab:app:density}), while a shared
$k_{\mathrm{sh}}$ ranges up to $d+1$. The cap binds
on a heavy tail: on OPT-1.3b the $90$th percentile unit fires on $187$ support tokens and the $99$th on $954$, and
$7.5\%$ of units fire on more than $259$. The budget is fixed by the data alone: from a single source at
$T=2\,048$ the projector's setup returns $259.2$ dimensions per row for the shared subspace at energy fraction
$0.95$ and $34.4$ for our cap $128$, and our cap $64$ rises from $23.5$ dimensions per row at $2\,048$ tokens to
$55.8$ at $49\,152$.

%% file: app_cl_mechanism.tex
\section{What the criterion is: the selector controls in full}\label{app:cl:necessity}

\begin{figure}[tp]
\centering
\includegraphics[width=0.75\textwidth]{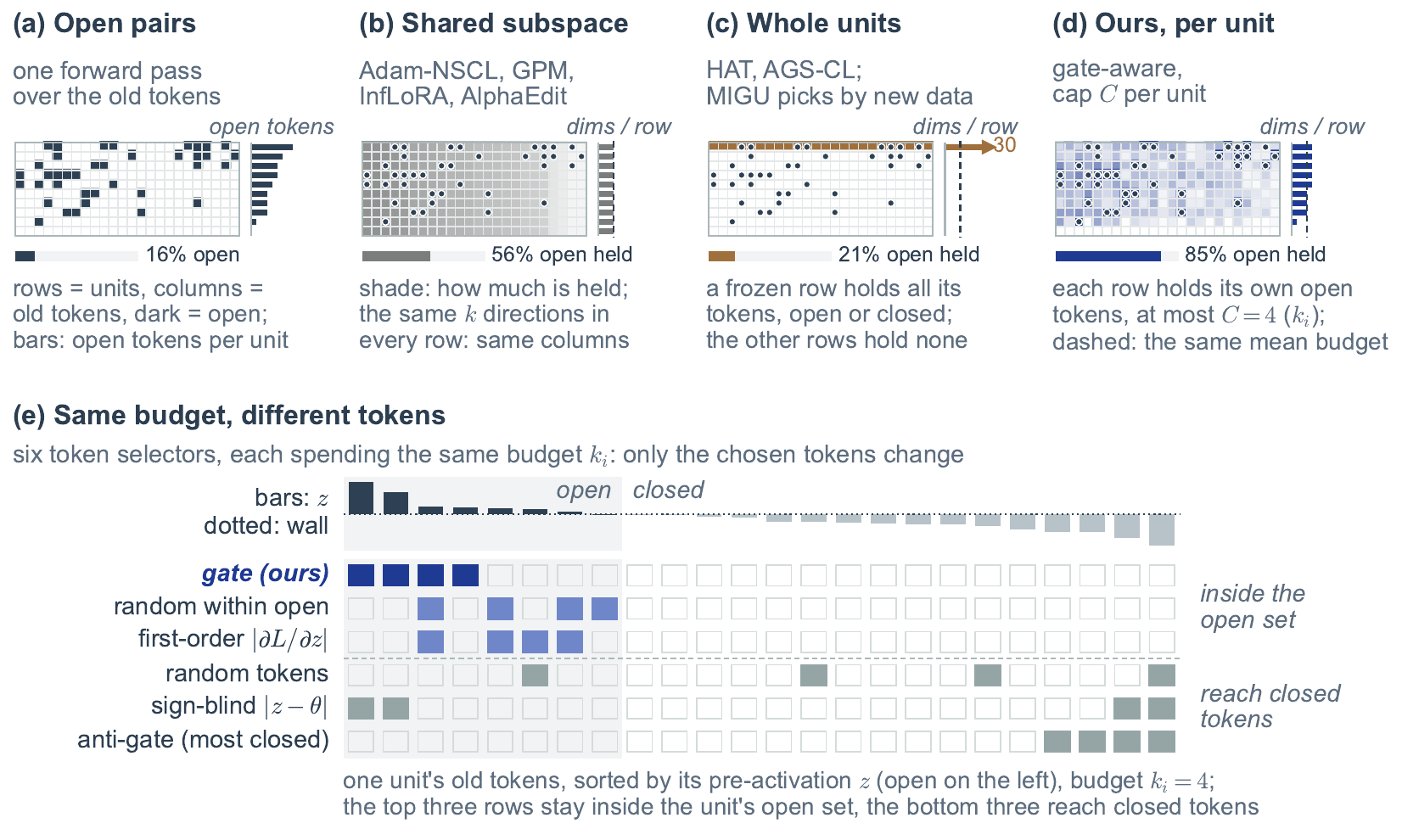}
\caption{\textbf{Which (token, unit) pairs a constraint holds, at one budget.} A random ReLU layer with
$d=30$ inputs, $m=10$ units and $T=24$ old tokens.
\textbf{How to read a grid.} Each row is one unit, that is one row of the layer's first weight matrix (\texttt{fc1}
in OPT), and each column is one old token, so the cell in row $i$ and column $t$ is the pair $(t,i)$. Rows are sorted by how many tokens the unit
fires on, most active on top. Columns are sorted by how much the shared subspace of (b) holds them, most held
on the left. The order is the same in every grid, so a given cell sits at the same place in (a)--(d).
\textbf{(a)} The gate pattern of the base model, read off one forward pass over the old tokens: a cell is dark
when unit $i$ is open on token $t$, that is $z_i(x_t)>\theta$, with $\theta$ the gate threshold ($0$ for
ReLU). The bar to the right of a row counts that
unit's open tokens. The top unit fires on $8$ of the $24$ tokens, the next on $6$, and the bottom unit on
none. Only $38$ of the $240$ pairs, $16\%$, are open; at OPT-1.3b the share is $4\%$.
\textbf{(b)--(d)} A cell's shade is the fraction of an isotropic weight update's effect on that pair's
pre-activation that the constraint removes: white means the pair is left free, and full colour means it is
held exactly. The dots repeat the open pairs of (a). The bar to the right of a row is the number of directions
the constraint forbids in that row, and the dashed line marks the mean, which is $3$ for every family. The bar
under a grid is the mean shade over the open pairs. Holding a closed pair guards only against its opening.
\textbf{(b)} The shared subspace of Adam-NSCL, GPM, InfLoRA and AlphaEdit forbids the top $3$ eigendirections
of the old inputs' second moment, and it forbids them in every row. Every row therefore spends $3$ directions,
and a cell's shade depends only on its column: the leftmost token is held at $83\%$ in every row and the
rightmost at $8\%$. The top unit's $8$ open tokens are held at between $9\%$ and $78\%$. Overall $56\%$ of the
open pairs' effect is held.
\textbf{(c)} Methods that protect whole units freeze entire rows; HAT and AGS-CL choose them by old-task use,
and MIGU by new-data magnitude. A frozen row spends all $30$ of its directions (arrow), so a mean of $3$ buys
one frozen row out of ten. The top row is frozen. All $24$ of its tokens are held, including the $16$ on which
it is closed, and no other row holds anything, so $21\%$ of the open pairs' effect is held.
\textbf{(d)} Our projector gives unit $i$ the budget $k_i=\min(C,\#\text{open})$ with cap $C=4$, and spends it
on that unit's own open tokens with the largest activations. The top unit therefore holds its $4$ most active
open tokens exactly, and its other $4$ open tokens only in part, at $54\%$, $22\%$, $14\%$ and $10\%$. A unit
with $3$ open tokens holds all $3$ exactly, and the bottom unit, which never fires, receives $k_i=0$ and holds
nothing. The budgets, top to bottom, are $4,4,4,4,4,3,3,3,1,0$, which also average $3$, and $85\%$ of the open
pairs' effect is held.
\textbf{(e)} The selector controls of \S\ref{sec:necessity}, drawn for the top unit. Its $24$ old tokens are
sorted by pre-activation $z$, shown by the bars with the open tokens on the left, and the dotted line is
the wall. Eight are open, and its
budget is $k_i=4$. Each row marks the $4$ tokens that one selector spends the budget on. The gate takes the
$4$ most active open tokens. Random-within-open takes $4$ of the $8$ open tokens at random, here sharing one
with the gate. The first-order criterion takes the $4$ tokens with the largest $|\partial L/\partial z|$;
since that derivative is zero on closed tokens, it also stays inside the open set. Random tokens draws $4$ of
all $24$, and here lands on $1$ open and $3$ closed tokens. The sign-blind criterion takes the $4$ largest
$|z-\theta|$, which are the $2$ most open and the $2$ most deeply closed tokens, and the anti-gate takes the
$4$ most deeply closed. Every row spends the same $k_i$; only the tokens change.}
\label{fig:app:projectors}
\end{figure}

\begin{table}[t]\centering\footnotesize
\setlength{\tabcolsep}{3pt}
\caption{\textbf{The selector controls at identical per-unit constraint dimension.} OPT-1.3b, wt103, $C=64$,
$\bar k=23.5$. \textbf{(a)} The selected set on the base model (\emph{overlap} with the gate set, share \emph{open},
share of the active post-activation \emph{energy} preserved exactly) and outcomes, means over $n$ seeds.
\textbf{(b)} Gate minus rank-matched random tokens on OPT-125m, on Python (OPT-1.3b) and on OPT-6.7b
($\bar k=68.7$), paired by seed, $95\%$ interval and one-sided $p$; the first row pools the three off-partition
contrasts of Table~\ref{tab:necessity} (random-effects difference, exact sign-flip $p$).}
\label{tab:app:necessity}
\vspace{2pt}
\textbf{(a)}\quad
\begin{tabular}{llrrrrrr}
\toprule
& & \multicolumn{3}{c}{selected set, base model} & & \multicolumn{2}{c}{outcome} \\
\cmidrule(lr){3-5}\cmidrule(lr){7-8}
run & selector & overlap & open & energy & $n$ & forget & learn \\
\midrule
\texttt{B1}      & no constraint     & ---      & ---      & ---      & $3$ & $+0.2684$ & $-0.5878$ \\
\midrule
\texttt{B17c64}  & \texttt{gate}     & $100\%$  & $100\%$  & $59.6\%$ & $6$ & $+0.0247$ & $-0.6240$ \\
\texttt{B17c64o} & \texttt{openrand} & $69.9\%$ & $100\%$  & $35.2\%$ & $3$ & $+0.0324$ & $-0.6223$ \\
\texttt{B17c64g} & \texttt{grad}     & $70.1\%$ & $99.6\%$ & $35.8\%$ & $6$ & $+0.0222$ & $-0.6216$ \\
\texttt{B17c64r} & random            & $2.3\%$  & $8.3\%$  & $2.8\%$  & $6$ & $+0.0563$ & $-0.6285$ \\
\texttt{B17c64n} & \texttt{absz}     & $3.0\%$  & $3.0\%$  & $10.4\%$ & $6$ & $+0.0821$ & $-0.6270$ \\
\texttt{B17c64x} & \texttt{closed}   & $0.0\%$  & $0.2\%$  & $0.1\%$  & $6$ & $+0.0906$ & $-0.6275$ \\
\bottomrule
\end{tabular}

\vspace{6pt}
\textbf{(b)}\quad\setlength{\tabcolsep}{2.2pt}
\begin{tabular}{lrrrrr}
\toprule
setting & $n$ & $\Delta$forget $[95\%\ \mathrm{CI}]$ & test & seeds & $\Delta$learn $[95\%\ \mathrm{CI}]$ \\
\midrule
1.3b, pooled & $18$ & $-0.0515$ $[-0.0740,-0.0289]$ & $p=3.81{\times}10^{-6}$ & $18/18$ & --- \\
125m, $C=32$ & $3$  & $-0.0114$ $[-0.0162,-0.0067]$ & $p=0.0046$ & $3/3$ & $+0.0254$ $[+0.0154,+0.0355]$ \\
python, $C=64$ & $6$  & $-0.0224$ $[-0.0320,-0.0129]$ & $p=0.0009$ & $6/6$ & $+0.0069$ $[+0.0034,+0.0104]$ \\
6.7b, $C=128$ & $3$  & $-0.0098$ $[-0.0150,-0.0047]$ & $p=0.0072$ & $3/3$ & $+0.0082$ $[+0.0070,+0.0094]$ \\
\bottomrule
\end{tabular}
\end{table}

Figure~\ref{fig:app:projectors} draws, on one random layer at one mean budget per row, the (token, unit) pairs on
which each constraint family of \S\ref{sec:families} spends its budget, and in panel~(e) the tokens that each
selector control of \S\ref{sec:necessity} picks for one unit. The controls hold each unit's $k_i$ at the gate's value
and vary only which of its support tokens fill it; they share one no-constraint run \texttt{B1} (seeds $0$--$2$),
run on seeds $0$--$5$ (\texttt{openrand} on $0$--$2$) and use the tests of Appendix~\ref{app:cl:stats}.
\texttt{gate} takes the $k_i$ largest base post-activations (Eq.~\eqref{eq:ki}); \texttt{openrand} draws $k_i$
tokens uniformly from the unit's open set; \texttt{grad} takes the largest $|\partial L_{\mathrm{old}}/\partial z|$
at the base model; \texttt{absz} the largest $|z-\theta|$ over all tokens, which lands almost entirely on closed
pairs because the negative tail of $z$ is far longer; \texttt{closed} the most closed tokens; and the random selector
$k_i$ tokens uniformly from the whole support.

\paragraph{Reading the table.} In Table~\ref{tab:app:necessity}, off the partition the gate forgets less than each
of the three selectors on all $18$ seed-pairs; on it, \texttt{openrand} and \texttt{grad} are TOST-equivalent to the gate
(Table~\ref{tab:necessity}) while preserving $35.2\%$ and $35.8\%$ of the active energy against $59.6\%$; and
the open share of the selected set separates the three selectors on the partition from the three off it and
orders the latter by forgetting, while preserved energy does not (\texttt{absz}
preserves more than the random selector and forgets more), so the partition matters, not a ranking inside it or the
energy it covers. Learning differences
across the five OPT-1.3b contrasts lie between $-0.0024$ and $+0.0045$ nats, and the gate learns more than
\texttt{grad} on all six seeds, by $0.0024$ nats ($95\%$ CI $[-0.0039,-0.0008]$, two-sided $p=0.0107$).
\texttt{grad} inherits the partition through the chain rule (Proposition~\ref{prop:firstorder}), placing $99.6\%$
of its pairs on open pairs, but needs a backward pass and a $T\times m$ gradient tensor per wall
($2\,048\times8\,192$ in \texttt{fp32}, $1.6$\,GB over the $24$ walls), where the gate reads the sign of $z-\theta$ off one
forward pass.

\subsection{Freezing whole units}\label{app:cl:rowfreeze}

\begin{table}[ht]\centering\scriptsize
\setlength{\tabcolsep}{2.8pt}
\caption{\textbf{Freezing whole units against the per-unit projector.} OPT-1.3b, WikiText-103,
\texttt{fc1}, lr $10^{-4}$. $\Delta$ is ours minus the freezing run, paired by seed; negative means ours forgets
less (one-sided $p$) or learns more (two-sided $p$); the fraction counts seed pairs with $\Delta<0$.}
\label{tab:app:rowfreeze}
\begin{tabular}{llrrrrr}
\toprule
freezing setting & matched to & dims/row & $n$ & forget, learn & $\Delta$ forget ($p$) & $\Delta$ learn ($p$) \\
\midrule
top $1.15\%$ of units & cap $64$ ($23.5$)   & $23.51$ & $6$ & $+0.2101$, $-0.5851$ & $-0.1855$ ($6/6$, $<10^{-4}$) & $-0.0389$ ($6/6$, $<10^{-4}$) \\
top $1.68\%$ of units & cap $128$ ($34.4$)  & $34.52$ & $3$ & $+0.1951$, $-0.5853$ & $-0.1713$ ($3/3$, $0.0014$)   & $-0.0357$ ($3/3$, $0.027$) \\
$59.6\%$ of active energy & cap $64$ (energy) & $69.34$ & $3$ & $+0.1719$, $-0.5871$ & $-0.1380$ ($3/3$, $0.0006$) & $-0.0367$ ($3/3$, $0.054$) \\
\bottomrule
\end{tabular}
\end{table}

We freeze the \texttt{fc1} rows (weights and bias) of the units with the largest old-support activity
$\sum_t h_i(t)^2$ and train the rest without constraint. A frozen row removes all $d+1$ of its directions, so
freezing $1.15\%$ of units ($94$ per wall) costs $23.51$ dimensions per row, the budget of cap $64$, $1.68\%$
($138$ units) costs $34.52$, that of cap $128$, and the smallest top set carrying the $59.6\%$ of active energy
that cap $64$ preserves exactly needs $2$ to $464$ units per wall, $69.34$ dimensions per row. At equal budget
freezing forgets about eight times as much as the per-unit projector on the shared seeds and learns
$0.036$--$0.039$ nats less, and at three times the budget it still forgets five times as much
(Table~\ref{tab:app:rowfreeze}).

\section{Comparison with other method families}\label{app:cl:families}

\begin{table}[t]\centering\scriptsize
\setlength{\tabcolsep}{4pt}
\caption{\textbf{Method families on \texttt{fc1}.} OPT-1.3b, WikiText-103, seeds $0$--$2$, forgetting and learning
as in Appendix~\ref{app:cl:setup}. $\Delta$ is our cap $128$ minus the row, paired by seed; negative means ours
forgets less (one-sided $p$) or learns more (two-sided $p$).}
\label{tab:app:families}
\input{tab_families}
\end{table}

\begin{table}[ht]\centering\footnotesize
\setlength{\tabcolsep}{4pt}
\caption{\textbf{WiSE-FT on each method.} OPT-1.3b, WikiText-103, seeds $0$--$2$. \emph{End of training} is
$\alpha=1$; \emph{zero forgetting} is the learning each method's WiSE-FT curve reaches where the old-probe NLL returns
to its base value, and its gain is paired by seed against WiSE-FT on unconstrained fine-tuning (negative = learns
more; seeds on which it is negative).}
\label{tab:app:wiseft}
\input{tab_wiseft}
\end{table}

Every family runs under the training setup of Appendix~\ref{app:cl:setup} on the \texttt{fc1} walls of OPT-1.3b with
WikiText-103 as the new task; under a single transition and without its intra-task distillation term, AdNS~\citep{kong2022adns} reduces to
Adam-NSCL with a rank cap and one scalar strength. A random orthonormal subspace of exactly the data-derived dimension per
wall (\texttt{B16e99r}), at matched $\|\Delta W\|_F$ ($12.1$ against $12.2$), forgets $+0.2176$, against $+0.2684$
for no constraint and $-0.0045$ for the real projector at the same rank, and at the $259$-dimension rank the
corresponding random control forgets $+0.2318$ against the projector's $+0.0227$. Two-sided activation anchoring on the
walls (LDIFS \citep{mukhoti2024ldifs}, $\lambda=100$) forgets $+0.1720$, fresh old-corpus replay
\citep{ibrahim2024continual} $+0.1721$ on a quarter of the micro-batches and $+0.0314$ on every one, and replaying
the $2\,048$ support tokens themselves $+0.7247$. Relaxing Adam-NSCL at $k_{\mathrm{sh}}=259$ to
$(1-\alpha)D+\alpha DP$ with $\alpha\in\{1,0.8,0.5\}$ moves its learning from $-0.5671$ to $-0.5697$ and $-0.5719$
(seeds $0$--$2$), recovering $0.005$ of the $0.054$ nats by which our cap $128$ learns more
(Table~\ref{tab:app:families}).

\subsection{EWC, MIGU, LoRA and LoRA-Null}\label{app:cl:families:more}

Table~\ref{tab:app:families} gives every operating point of Figure~\ref{fig:families} except Adam-NSCL at
$k_{\mathrm{sh}}\in\{9,24,128\}$ and our caps $16$ and $512$, whose forgetting is in Figure~\ref{fig:pareto}, each
paired by seed against our cap $128$. The figure shows EWC across $\lambda\in\{4,\dots,4\times10^5\}$, LoRA
at ranks $8$, $32$, $128$ and learning rates $2\times10^{-4}$ and $10^{-3}$ (and rank $32$ at
$5\times10^{-4}$), LoRA-Null at ranks $32$ and $128$ ($2\times10^{-4}$), MIGU at its default mask ratio
$0.7$, L2-to-init at $\lambda\in\{10^{-3},10^{-2},0.1,0.3\}$, Adam-NSCL across $k_{\mathrm{sh}}$ and with its dimension set by GPM's energy criterion at $0.95$, and
ours at caps $C\in\{16,64,128,256,512\}$. Five operating points lie outside its range, each learning less
than unconstrained fine-tuning or than ours by a wide margin: MIGU at mask ratio $0.9$ (old-probe NLL falls
by $0.027$; $0.17$ nats less learning than ours), Adam-NSCL at $k_{\mathrm{sh}}=585$ and at GPM energy
$0.99$ (forgetting $0.004$ and $-0.005$; $0.054$ and $0.046$ nats less learning than unconstrained), EWC
at $\lambda=4\times10^6$ ($-0.011$; $0.021$ nats less), and L2-to-init at $\lambda=1$ (old-probe NLL falls by
$0.021$; $0.148$ nats less learning than ours).

\paragraph{Implementations.} \emph{EWC} adds $\tfrac{\lambda}{2}\sum_j F_j(w_j-w_j^0)^2$ over the \texttt{fc1}
weights, the diagonal Fisher estimated on the $16$ old-support windows ($\lambda=0$ reproduces \texttt{B1}), and
\emph{L2-to-init} sets $F_j=1$; \emph{MIGU} masks, on each batch, the gradient of the fraction $\tau$ of units with
the smallest mean $|z_i|$; \emph{LoRA} trains $s\,BA$ with $s=\alpha/r$, $\alpha=2r$, $B$ initialised at zero and the
bias trained, which O-LoRA equals at the first task; \emph{LoRA-Null} initialises the adapter from the thin SVD of $W_0UU^\top$, $U$ the
$r$ smallest-eigenvalue eigenvectors of $X_{\mathrm{old}}^\top X_{\mathrm{old}}$, on a frozen residual that makes the
effective weight $W_0$ at initialisation.

\paragraph{Reading the table.} EWC's forgetting falls with $\lambda$, from $+0.27$ at $\lambda\le40$ to
$+0.094$ at $\lambda=4\times10^4$; up to there ours forgets less at $p\le0.019$ and learns at least as
well. At $\lambda=4\times10^5$ EWC reaches our frontier ($+0.0175$, $-0.6241$; neither paired difference
is significant against cap $128$ or cap $256$), as expected of a criterion whose diagonal Fisher inherits
the open/closed partition (Proposition~\ref{prop:firstorder}). MIGU at $\tau=0.7$ forgets $2\times$ more than
ours and learns $0.051$ nats less ($p=0.0005$). Every LoRA configuration forgets $5.7$--$22\times$ as much as
ours; the four that learn significantly more do so by $0.035$--$0.042$ nats, at $12$--$16\times$ the forgetting,
and LoRA at $r=8$, lr $2\times10^{-4}$ and at $r=128$, lr $10^{-3}$ is dominated outright. LoRA-Null at lr $2\times10^{-4}$
forgets $3.0\times$ less than LoRA at rank $32$ ($0.085$ against $0.256$) and $2.3\times$ less at rank
$128$ ($0.122$ against $0.284$); against our cap $128$ it forgets $3.6\times$ and $5.1\times$ more, is
dominated at rank $32$ ($0.061$ nats less learning, $p=0.003$), and at rank $128$ learns $0.012$ nats
more than ours ($p=0.055$). L2-to-init at $\lambda=0.3$ forgets the same as our cap $256$ (paired
difference $+0.0007$, TOST-equivalent at $\delta=0.0068$) and learns $0.049$ nats less ($3/3$ seeds,
$p=0.020$). Adam-NSCL with GPM's energy criterion at $0.95$ (mean $259$ dimensions per row) forgets about as
much as our cap $128$ ($+0.0011$, $p=0.57$) and learns $0.014$ nats more; our cap $256$, at $47.3$
dimensions per row, forgets $0.0117$ less than it ($3/3$ seeds, one-sided $p=0.03$) and learns $0.016$
nats less ($p=0.013$).

\subsection{WiSE-FT and zero-shot accuracy}\label{app:cl:wiseft}

\begin{table}[ht]\centering\footnotesize
\caption{\textbf{Zero-shot accuracy after one fine-tuning transition.} OPT-1.3b fine-tuned on WikiText-103
(\texttt{fc1}, lr $10^{-4}$), mean over seeds $0$--$2$, in \%. Accuracy for LAMBADA (OpenAI variant) and
Winogrande, length-normalised accuracy for the others. \emph{Mean drop} is the base model's mean accuracy over the
five tasks minus the run's. Cap $128$ against no constraint: $+1.92$ points of mean accuracy, positive on
$3$ of $3$ seeds, one-sided paired $p=0.0125$ ($n=3$).}
\label{tab:zs}
\setlength{\tabcolsep}{3pt}
\input{tab_zs}
\end{table}

WiSE-FT \citep{wortsman2022wiseft} interpolates a finished run's \texttt{fc1} weights toward the base,
$W_0+\alpha(W-W_0)$. On the checkpoints of seeds $0$--$2$ at $\alpha\in\{0.4,0.5,0.6,0.7,0.75,0.8,0.85,0.9,0.95,1\}$
each curve is read at a fixed forgetting by linear interpolation between neighbouring $\alpha$; $\alpha$ is swept
on the evaluation sets, so this compares trade-off curves, not deployable choices of $\alpha$. Every constrained run
composes with it (Table~\ref{tab:app:wiseft}), and the matched-budget ordering holds: at zero forgetting
composed cap $64$ learns $0.0074$ nats more than composed $k_{\mathrm{sh}}=24$ and composed cap $128$ $0.0084$
more than composed $k_{\mathrm{sh}}=34$, on $3/3$ seeds each ($p=0.08$), while Adam-NSCL with GPM's energy
criterion, at $7.5$--$11\times$ their dimensions, learns $0.005$--$0.008$ nats more than either composed cap
($p=0.19$, $0.09$). Without interpolation, each projector learns more on the seed mean than WiSE-FT on the
unconstrained run read at the same forgetting: by $0.006$ and $0.008$ nats at caps $64$ and $128$, by $0.001$--$0.003$ for
Adam-NSCL at the matched budgets and by $0.023$ for its energy criterion ($3$ seeds, $p=0.13$--$0.88$).
Zero-shot accuracy (Table~\ref{tab:zs}) is scored on the full evaluation splits of the five tasks.

%% file: tab_families.tex
\begin{tabular}{llcrrrr}
\toprule
family & setting & lr & forget & learn & $\Delta$ forget ($p$) & $\Delta$ learn ($p$)\\
\midrule
none & --- & $10^{-4}$ & $+0.2684$ & $-0.5878$ & $-0.2447$ ($0.0018$) & $-0.0332$ ($0.0436$) \\
\addlinespace[2pt]
EWC & $\lambda=4$ & $10^{-4}$ & $+0.2695$ & $-0.5850$ & $-0.2457$ ($0.0007$) & $-0.0360$ ($0.0261$) \\
EWC & $\lambda=40$ & $10^{-4}$ & $+0.2743$ & $-0.5863$ & $-0.2505$ ($0.0001$) & $-0.0347$ ($0.0118$) \\
EWC & $\lambda=400$ & $10^{-4}$ & $+0.2578$ & $-0.5863$ & $-0.2341$ ($0.0004$) & $-0.0347$ ($0.0051$) \\
EWC & $\lambda=4\times10^3$ & $10^{-4}$ & $+0.2092$ & $-0.5962$ & $-0.1854$ ($<10^{-4}$) & $-0.0248$ ($0.0047$) \\
EWC & $\lambda=4\times10^4$ & $10^{-4}$ & $+0.0937$ & $-0.6144$ & $-0.0700$ ($0.0190$) & $-0.0066$ ($0.3399$) \\
EWC & $\lambda=4\times10^5$ & $10^{-4}$ & $+0.0175$ & $-0.6241$ & $+0.0063$ ($0.7618$) & $+0.0032$ ($0.2396$) \\
EWC & $\lambda=4\times10^6$ & $10^{-4}$ & $-0.0114$ & $-0.5668$ & $+0.0352$ ($0.9866$) & $-0.0541$ ($0.0069$) \\
\addlinespace[2pt]
MIGU & mask $0.7$ & $10^{-4}$ & $+0.0486$ & $-0.5698$ & $-0.0248$ ($0.0039$) & $-0.0512$ ($0.0005$) \\
MIGU & mask $0.9$ & $10^{-4}$ & $-0.0267$ & $-0.4484$ & $+0.0505$ ($0.9896$) & $-0.1726$ ($0.0003$) \\
\addlinespace[2pt]
LoRA & $r=8$ & $2\times10^{-4}$ & $+0.1360$ & $-0.5553$ & $-0.1122$ ($0.0003$) & $-0.0657$ ($0.0080$) \\
LoRA & $r=8$ & $10^{-3}$ & $+0.3601$ & $-0.6567$ & $-0.3363$ ($0.0007$) & $+0.0357$ ($0.0360$) \\
LoRA & $r=32$ & $2\times10^{-4}$ & $+0.2558$ & $-0.6320$ & $-0.2320$ ($0.0007$) & $+0.0110$ ($0.1448$) \\
LoRA & $r=32$ & $5\times10^{-4}$ & $+0.3319$ & $-0.6611$ & $-0.3081$ ($0.0004$) & $+0.0401$ ($0.0101$) \\
LoRA & $r=32$ & $10^{-3}$ & $+0.3875$ & $-0.6563$ & $-0.3637$ ($0.0039$) & $+0.0353$ ($0.0015$) \\
LoRA & $r=128$ & $2\times10^{-4}$ & $+0.2841$ & $-0.6626$ & $-0.2604$ ($0.0002$) & $+0.0416$ ($0.0168$) \\
LoRA & $r=128$ & $10^{-3}$ & $+0.5328$ & $-0.5927$ & $-0.5090$ ($0.0027$) & $-0.0283$ ($0.0351$) \\
\addlinespace[2pt]
LoRA-Null & $r=32$ & $2\times10^{-4}$ & $+0.0853$ & $-0.5602$ & $-0.0615$ ($0.0028$) & $-0.0608$ ($0.0026$) \\
LoRA-Null & $r=128$ & $2\times10^{-4}$ & $+0.1221$ & $-0.6325$ & $-0.0983$ ($0.0055$) & $+0.0115$ ($0.0551$) \\
\addlinespace[2pt]
L2-to-init & $\lambda=10^{-3}$ & $10^{-4}$ & $+0.2648$ & $-0.5944$ & $-0.2410$ ($0.0003$) & $-0.0266$ ($0.0309$) \\
L2-to-init & $\lambda=10^{-2}$ & $10^{-4}$ & $+0.1893$ & $-0.6328$ & $-0.1655$ ($0.0022$) & $+0.0118$ ($0.0425$) \\
L2-to-init & $\lambda=0.1$ & $10^{-4}$ & $+0.0558$ & $-0.6252$ & $-0.0321$ ($0.0517$) & $+0.0042$ ($0.5019$) \\
L2-to-init & $\lambda=0.3$ & $10^{-4}$ & $+0.0103$ & $-0.5694$ & $+0.0134$ ($0.9557$) & $-0.0516$ ($0.0169$) \\
L2-to-init & $\lambda=1$ & $10^{-4}$ & $-0.0207$ & $-0.4733$ & $+0.0445$ ($0.9947$) & $-0.1477$ ($0.0005$) \\
\addlinespace[2pt]
Adam-NSCL & $k_{\mathrm{sh}}=34$ & $10^{-4}$ & $+0.0462$ & $-0.6229$ & $-0.0224$ ($0.0075$) & $+0.0019$ ($0.2644$) \\
Adam-NSCL & $k_{\mathrm{sh}}=259$ & $10^{-4}$ & $+0.0255$ & $-0.5671$ & $-0.0018$ ($0.2146$) & $-0.0539$ ($0.0002$) \\
Adam-NSCL & GPM energy $0.95$ & $10^{-4}$ & $+0.0227$ & $-0.6346$ & $+0.0011$ ($0.5669$) & $+0.0136$ ($0.0217$) \\
\addlinespace[2pt]
ours & $C=64$ & $10^{-4}$ & $+0.0339$ & $-0.6237$ & $-0.0101$ ($0.0181$) & $+0.0027$ ($0.1251$) \\
ours & $C=128$ & $10^{-4}$ & $+0.0238$ & $-0.6210$ & \multicolumn{2}{c}{reference} \\
ours & $C=256$ & $10^{-4}$ & $+0.0110$ & $-0.6189$ & $+0.0128$ ($0.9731$) & $-0.0021$ ($0.1189$) \\
\bottomrule
\end{tabular}

%% file: tab_wiseft.tex
\begin{tabular}{lrrrrc}
\toprule
 & & \multicolumn{2}{c}{end of training} & \multicolumn{2}{c}{WiSE-FT, zero forgetting}\\
\cmidrule(lr){3-4}\cmidrule(lr){5-6}
method & dims/row & forget & learn & learn & gain over none (seeds) \\
\midrule
no constraint & --- & $+0.2684$ & $-0.5878$ & $-0.5904$ & --- \\
ours, $C=64$ & 23.5 & $+0.0339$ & $-0.6237$ & $-0.6289$ & $-0.0385$ ($3/3$) \\
ours, $C=128$ & 34.4 & $+0.0238$ & $-0.6210$ & $-0.6262$ & $-0.0359$ ($3/3$) \\
Adam-NSCL, $k_{\mathrm{sh}}=24$ & 24 & $+0.0468$ & $-0.6252$ & $-0.6215$ & $-0.0311$ ($3/3$) \\
Adam-NSCL, $k_{\mathrm{sh}}=34$ & 34 & $+0.0462$ & $-0.6229$ & $-0.6178$ & $-0.0275$ ($3/3$) \\
Adam-NSCL, GPM energy $0.95$ & 259.2 & $+0.0227$ & $-0.6346$ & $-0.6337$ & $-0.0434$ ($3/3$) \\
\bottomrule
\end{tabular}

%% file: tab_zs.tex
\begin{tabular}{lcccccc}
\toprule
method & LAMBADA & HellaSwag & ARC-e & Winogrande & SciQ & mean drop \\
\midrule
base & 58.9 & 53.7 & 51.3 & 59.1 & 78.2 & --- \\
no constraint & 53.7 & 51.1 & 48.8 & 58.6 & 74.4 & 2.94 \\
Adam-NSCL, $k_{\mathrm{sh}}=34$ & 56.1 & 52.6 & 51.3 & 58.5 & 77.3 & 1.08 \\
ours, $C=128$ ($\bar k=34.4$) & 57.3 & 52.6 & 50.8 & 58.1 & 77.4 & 1.02 \\
Adam-NSCL, GPM energy $0.95$ ($259$) & 57.0 & 53.2 & 51.4 & 58.4 & 76.9 & 0.85 \\
\bottomrule
\end{tabular}

%% file: app_repair.tex
\theoremstyle{plain}
\ifcsname lemma*\endcsname\else\newtheorem*{lemma*}{Lemma}\fi
\ifcsname theorem*\endcsname\else\newtheorem*{theorem*}{Theorem}\fi
\ifcsname proposition*\endcsname\else\newtheorem*{proposition*}{Proposition}\fi
\theoremstyle{remark}
\providecommand{\proofidea}{\par\smallskip\noindent\textit{Proof idea.}\ }
\providecommand{\KX}{\mathcal K_{\mathcal X}}
\providecommand{\Esc}{\operatorname{Esc}}
\makeatletter\providecommand{\tabcaptype}{\def\@captype{table}}\makeatother

\section{Repair after pruning: the full development}\label{app:repair}

\begin{figure}[tb]
\centering
\includegraphics[width=0.75\textwidth]{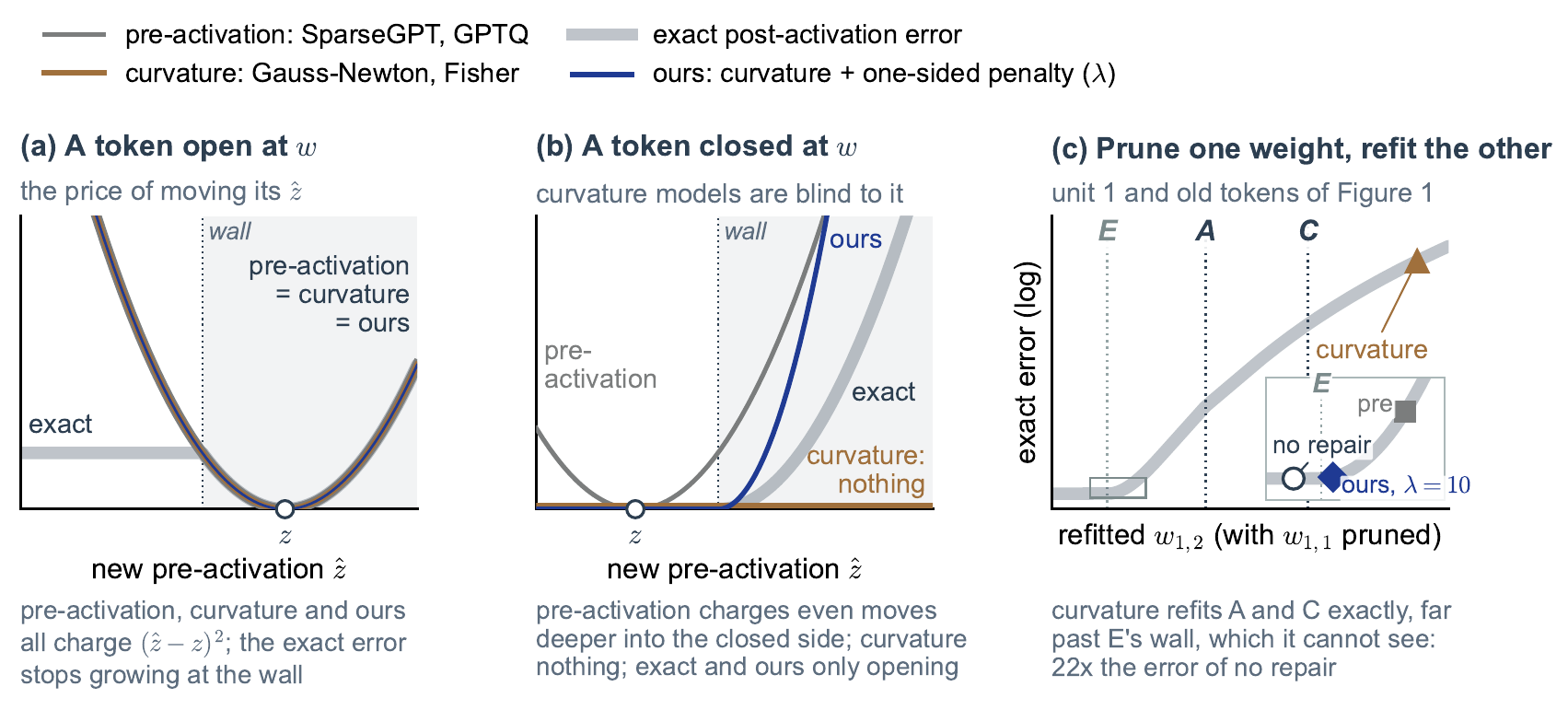}
\caption{\textbf{What each repair objective charges for, on open and closed tokens.} Four objectives are
compared, as in the legend: the pre-activation error $\sum_x(\hat z-z)^2$, which SparseGPT and GPTQ fit; the
Gauss--Newton and Fisher models of the post-activation error, which reduce to the in-cell error on the open
tokens, $\gated$ (Lemma~\ref{lem:blind}); the exact post-activation error $\post$; and ours, $\gatedh$, which
adds to that in-cell error a penalty of weight $\lambda$ on closed tokens that open, the one-sided penalty of
Theorem~\ref{thm:confine}.
\textbf{(a)} The price each objective charges one token that is open at $w$, here with $z=1$, as its new
pre-activation $\hat z$ moves; shading marks the open side of the wall. The pre-activation, curvature and our
objectives charge the same parabola $(\hat z-z)^2$, drawn as three coincident curves. The exact error agrees
with them while the token stays open, and stops growing once it closes. A move to $\hat z=-1$ costs $4$ under
the three quadratics but only $1$ under the exact error, because the unit's output cannot fall below $0$.
\textbf{(b)} The same for a token closed at $w$, here with $z=-1$. The pre-activation objective still charges
$(\hat z-z)^2$, so a move to $\hat z=-2$, deeper into the closed side and with no effect on the output, costs
$1$. The curvature models charge nothing anywhere, even after the token opens. The exact error charges only
opening, $\max(\hat z,0)^2$, and ours charges $\lambda$ times that, with $\lambda=3$ here. A move to
$\hat z=1$ thus costs $4$ under pre-activation, $0$ under curvature, $1$ under the exact error and $3$ under
ours.
\textbf{(c)} What each objective does to a pruned row. Take unit~1 of Figure~\ref{fig:duality}, with
$w=(1,0.4)$, old tokens $A$ and $C$ open and $E$ and $D$ closed. Prune $w_{1,1}$ to $0$ and let each
objective refit $w_{1,2}$ alone. The horizontal axis is the refitted $w_{1,2}$, and the curve is the exact
error along it (log scale). The dotted lines mark where $E$ opens ($w_{1,2}=0.50$) and where $A$ and $C$ open
again ($1.65$ and $2.85$). No point of this line keeps every gate: $A$ and $C$ need $w_{1,2}>2.85$, while $E$
needs $w_{1,2}\le0.50$. In this toy no refit therefore improves on leaving the pruned row alone at
$w_{1,2}=0.4$, where $A$, $C$ and $E$ are all closed and the exact error is $2.45$; the objectives differ in
how much they lose. The curvature fit sees only $A$ and $C$ and restores both exactly, by moving to
$w_{1,2}=4.13$. That opens $E$ to $z_E=7.26$, and the exact error there is $52.8$, $22\times$ that of no
repair. The pre-activation fit, at $0.81$, keeps $A$ and $C$ closed and opens $E$ to $0.61$, for an exact
error of $2.83$. Ours at $\lambda=10$ stops at $0.54$, just past $E$'s wall with $z_E=0.09$, and its exact
error of $2.46$ equals that of no repair to within $1\%$. The inset, on a linear scale, enlarges the
three fits next to $E$'s wall.}
\label{fig:app:repair}
\end{figure}

This appendix proves Lemma~\ref{lem:blind} and Theorem~\ref{thm:confine}, whose body statements condense
parts (b)--(c) and (a)--(f) of the full forms below, and derives the signatures they predict;
Theorem~\ref{thm:escape} is proved in Appendix~\ref{app:proof}, the experiments are in
Appendix~\ref{app:repair:exp}, and Figure~\ref{fig:app:repair} reads the objectives of \S\ref{sec:repair} one
token at a time and on one pruned row of the toy layer of Figure~\ref{fig:duality}.

\subsection{Notation and curvature blindness}\label{sec:partIV}

\paragraph{Notation for one row.} Throughout this appendix and Appendix~\ref{app:proof} we fix one unit
of a wall matrix. Its dense weights are $w\in\mathbb R^d$, its bias is $b\in\mathbb R$ and its threshold
is $\theta\ge0$, so that $z_{\hat w}(x)=\langle\hat w,x\rangle+b$ for all $\hat w,x\in\mathbb R^d$. The
mask $M\in\{0,1\}^d$ keeps the index set $K\subseteq\{1,\dots,d\}$, with $k=|K|$ and complement $K^c$; for
$v\in\mathbb R^d$, the vectors $v_K\in\mathbb R^k$ and $v_{K^c}$ are its restrictions to $K$ and $K^c$. The
repair moves only the kept weights, and the bias is held at $b$. A repaired row is a vector
$\hat w\in\mathbb R^d$ with $\hat w=\hat w\odot M$, which we call \emph{masked}. We parametrise it by
$\delta:=\hat w_K-w_K\in\mathbb R^k$, so that the mask-only weights $w\odot M$ are $\delta=0$, and we write a
function of a masked $\hat w$ as a function of $\delta$ where convenient. Because the bias is fixed, we
write $\hat w\in\PX$ for $(\hat w,b)\in\PX$ and $\hat w\in\KX$ for $(\hat w,b)\in\KX$
(Definition~\ref{def:poly}); the cells $\mathcal C_>$ and $\mathcal C_\le$ and the sets $\PX$ and $\KX$ are
those of the dense unit $(w,b)$ on the calibration tokens. These tokens are $x_1,\dots,x_n\in\mathbb R^d$,
forming the multiset $\mathcal X$, and for every $\hat w\in\mathbb R^d$ we write
\[
z_t=z_w(x_t),\qquad z^M_t=z_{w\odot M}(x_t),\qquad \hat z_t=z_{\hat w}(x_t),
\]
so that $\hat z_t=z^M_t+\langle\delta,x_{t,K}\rangle$ for a masked $\hat w$, where $x_{t,K}\in\mathbb R^k$ is
the restriction of $x_t$ to $K$. Let $I_+=\{t:z_t>\theta\}$, the tokens in $\mathcal C_>$, and
$I_-=\{t:z_t\le\theta\}$, with $n_\pm=|I_\pm|$. Let $X_K\in\mathbb R^{n\times k}$ have rows $x_{t,K}^\top$,
and let $A\in\mathbb R^{n_+\times k}$ and $\tilde A\in\mathbb R^{n_-\times k}$ consist of its rows indexed by
$I_+$ and by $I_-$; if $n_+=0$ or $n_-=0$ the corresponding matrix is empty, with kernel $\mathbb R^k$. Put
\[
\zeta:=(z_t-z^M_t)_{t\in I_+}=\big(\langle w_{K^c},x_{t,K^c}\rangle\big)_{t\in I_+}\in\mathbb R^{n_+},\qquad
c:=(\theta-z^M_t)_{t\in I_-}\in\mathbb R^{n_-},
\]
and let $y_+$ denote the componentwise positive part of a vector $y$. For a masked $\hat w$,
$\hat z_t-z_t=(A\delta-\zeta)_t$ when $t\in I_+$ and $\hat z_t-\theta=(\tilde A\delta-c)_t$ when $t\in I_-$, so
the objectives of \S\ref{sec:repair} are
\[
\gated=\|A\delta-\zeta\|^2,\qquad \Esc=\|(\tilde A\delta-c)_+\|^2,
\]
and in particular $\|\zeta\|^2=\gated(w\odot M)$ and $E_0=\Esc(w\odot M)=\|(-c)_+\|^2$. The masked vectors that satisfy the inactive block of $\PX$ form $\QM=\{\delta:\tilde A\delta\le c\}$, and
$\PX\cap\{\hat w=\hat w\odot M\}=\{\delta:A\delta\ge\theta\mathbf 1-z^M_{I_+},\ \tilde A\delta\le c\}$, so
$\QM\supseteq\PX\cap\{\hat w=\hat w\odot M\}$. The escape is at most the closed$\to$open term of
Eq.~\eqref{eq:identity}, because both sum over the closed tokens with $\hat z_t>\theta$ and
$(\hat z_t-\theta)^2\le\hat z_t^{\,2}$ there, with equality for every such token exactly when $\theta=0$.

\begin{lemma*}[Full form of Lemma~\ref{lem:blind}]
Let $\theta\ge0$, let $(w,b)\in\mathbb R^d\times\mathbb R$ be one unit's dense weights and bias, and let
$\mathcal X\subseteq\mathbb R^d$ be a finite multiset of $n$ calibration tokens. Assume that no token lies on
the wall, that is, $z_w(x)\neq\theta$ for every $x\in\mathcal X$. Put
$m_{\mathcal X}:=\min_{x\in\mathcal X}|z_w(x)-\theta|$, which is positive, and
$R:=\max_{x\in\mathcal X}\|x\|$, with $m_{\mathcal X}/R:=+\infty$ if $R=0$. Throughout, $\hat w$
ranges over $\mathbb R^d$ with the bias held at $b$, and
$\gated(\hat w)=\sum_{x\in\mathcal X\cap\mathcal C_>}(z_{\hat w}(x)-z_w(x))^2$.
\par\textbf{(a)} $\post=\gated$ on $\KX$.
\par\textbf{(b)} Every $\hat w$ with $\|\hat w-w\|<m_{\mathcal X}/R$ lies in $\KX$. Hence $\post=\gated$ on
the open ball $B(w,m_{\mathcal X}/R)$, and both crossing terms of Eq.~\eqref{eq:identity} vanish on it.
\par\textbf{(c)} $\post$ and $\gated$ have the same derivatives of every order at $w$; in particular
$\post(w)=0$, $\nabla\post(w)=0$ and $\nabla^2\post(w)=2\sum_{x\in\mathcal X\cap\mathcal C_>}xx^\top$. The
Gauss--Newton matrix of $\post$ at $w$ is $J^\top J=\sum_{x\in\mathcal X\cap\mathcal C_>}xx^\top=
\frac12\nabla^2\post(w)$, where $J$ has one row $J_x=\mathds 1[x\in\mathcal C_>]\,x^\top$ per token, and for
every $\sigma>0$ the Fisher matrix at $\hat w=w$ of the Gaussian model
$y\sim\mathcal N(\act_\theta(z_{\hat w}(x)),\sigma^2)$, summed over the tokens, is $\sigma^{-2}J^\top J$. K-FAC
\citep{martens2015kfac} replaces $J^\top J$ by
$\big(\frac1n\sum_{x\in\mathcal X}\mathds 1[x\in\mathcal C_>]\big)\sum_{x\in\mathcal X}xx^\top$, which
approximates $J^\top J$ and differs from it in general.
\end{lemma*}

\begin{proof}
\emph{(a)} Let $\hat w\in\KX$. Every token then keeps its dense gate bit, so the sets $S_{+-}$ and $S_{-+}$
of Theorem~\ref{thm:decomp} are empty and $S_{++}=\mathcal X\cap\mathcal C_>$; Eq.~\eqref{eq:identity}
reduces to its in-cell term, which is $\gated(\hat w)$.

\emph{(b)} Let $\|\hat w-w\|<m_{\mathcal X}/R$ and $x\in\mathcal X$. By the Cauchy--Schwarz inequality,
$|z_{\hat w}(x)-z_w(x)|=|\langle\hat w-w,x\rangle|\le\|\hat w-w\|\,\|x\|$. If $x=0$ this is $0$; otherwise
$0<\|x\|\le R$, and it is less than $(m_{\mathcal X}/R)\,R=m_{\mathcal X}$. In both cases
$|z_{\hat w}(x)-z_w(x)|<m_{\mathcal X}\le|z_w(x)-\theta|$, so $z_{\hat w}(x)-\theta$ is non-zero and has the
sign of $z_w(x)-\theta$: the token keeps its gate bit. Hence $\hat w\in\KX$, part (a) gives
$\post(\hat w)=\gated(\hat w)$, and the crossing sets, hence both crossing terms, are empty.

\emph{(c)} Two functions that agree on an open set containing $w$ have the same derivatives of every order
at $w$, because a derivative at a point depends only on the values on a neighbourhood of that point. By
(b) this applies to $\post$ and $\gated$. The function $\gated(\hat w)=\sum_{x\in\mathcal X\cap\mathcal
C_>}\langle\hat w-w,x\rangle^2$ is a quadratic with value $0$, gradient $0$ and Hessian
$2\sum_{x\in\mathcal X\cap\mathcal C_>}xx^\top$ at $w$, which gives the three stated derivatives of $\post$.
For the Gauss--Newton matrix, write $\post(\hat w)=\sum_{x}\rho_x(\hat w)^2$ with residuals
$\rho_x(\hat w)=\act_\theta(z_{\hat w}(x))-\act_\theta(z_w(x))$. The map $\act_\theta$ is differentiable at
every $u\neq\theta$, with derivative $\mathds 1[u>\theta]$, and $z_w(x)\neq\theta$ by hypothesis, so $\rho_x$
is differentiable at $w$ with gradient $\mathds 1[x\in\mathcal C_>]\,x$. These gradients are the rows
$J_x$, and $J^\top J=\sum_x\mathds 1[x\in\mathcal C_>]\,xx^\top=\frac12\nabla^2\post(w)$. For the Gaussian
model, the score of one observation is
$\sigma^{-2}\big(y-\act_\theta(z_{\hat w}(x))\big)\nabla_{\hat w}\act_\theta(z_{\hat w}(x))$; at $\hat w=w$ its
second moment under the model is $\sigma^{-4}\cdot\sigma^2\cdot J_x^\top J_x=\sigma^{-2}J_x^\top J_x$, and
summing over the tokens gives $\sigma^{-2}J^\top J$. Finally, the K-FAC matrix differs from $J^\top J$ in
general. Take $d=2$, $w=e_1$, $b=0$, $\theta=0$ and the two tokens $e_1$ and $e_2-e_1$, whose
pre-activations are $1$ and $-1$. Then $J^\top J=e_1e_1^\top$ has $(2,2)$ entry $0$, while the K-FAC matrix
$\frac12\big(e_1e_1^\top+(e_2-e_1)(e_2-e_1)^\top\big)$ has $(2,2)$ entry $\frac12$.
\end{proof}

\subsection{Pricing the faces back}\label{sec:partV}

\begin{theorem*}[Full form of Theorem~\ref{thm:confine}]
Use the notation for one row of \S\ref{sec:partIV}, let $\lambda>0$ be real, and write
\[
F_\lambda(\delta):=\gatedh(\hat w,\lambda)=\|A\delta-\zeta\|^2+\lambda\,\|(\tilde A\delta-c)_+\|^2,
\qquad\delta\in\mathbb R^k .
\]
\par\textbf{(a)} $F_\lambda$ is convex and continuously differentiable, with Lipschitz gradient.
\par\textbf{(b)} Its penalty is the sum of the squared violations of the inactive constraints. It is zero
exactly on $\QM$, which contains $\PX\cap\{\hat w=\hat w\odot M\}$, so it is an exterior penalty for the
inactive block, finite everywhere.
\par\textbf{(c)} $F_\lambda$ attains its minimum, and $A\delta$ and $(\tilde A\delta-c)_+$ take the same
value at every minimiser; in particular $\gated$ and $\Esc$ take single values $g(\lambda)$ and
$e(\lambda)$ on the minimiser set.
\par\textbf{(d)} Every $\delta$ with $F_\lambda(\delta)\le F_\lambda(0)$ (every minimiser, and the output
of any method whose output has objective at most that of its starting point $w\odot M$) satisfies
\begin{equation}\label{eq:massbound}
\sum_{x\in\mathcal X\cap\mathcal C_{\le}}\max\big(z_{\hat w}(x)-\theta,0\big)^2
\ \le\ \frac{\|\zeta\|^2}{\lambda}+E_0 ,\qquad
E_0:=\!\!\sum_{x\in\mathcal X\cap\mathcal C_{\le}}\!\!\max\big(z_{w\odot M}(x)-\theta,0\big)^2 .
\end{equation}
If $\QM\neq\varnothing$, then $G^\star:=\min_{\delta\in\QM}\gated(\delta)$ is attained, and every $\delta$
with $F_\lambda(\delta)\le G^\star$, every minimiser in particular, satisfies
\begin{equation}\label{eq:massbound2}
\sum_{x\in\mathcal X\cap\mathcal C_{\le}}\max\big(z_{\hat w}(x)-\theta,0\big)^2
\ \le\ \frac{G^\star}{\lambda}\ \xrightarrow[\lambda\to\infty]{}\ 0 .
\end{equation}
\par\textbf{(e)} On $(0,\infty)$, $\lambda\mapsto e(\lambda)$ is non-increasing and $\lambda\mapsto
g(\lambda)$ is non-decreasing.
\par\textbf{(f)} $e(\lambda)\to0$ as $\lambda\to\infty$ if and only if $\QM\neq\varnothing$. If
$\QM=\varnothing$, then $\min_\delta\Esc(\delta)$ is attained and positive, and $e(\lambda)\ge
\min_\delta\Esc(\delta)$ for every $\lambda>0$.
\par For $\theta=0$ the escape is exactly the closed$\to$open term of Eq.~\eqref{eq:identity}, and $E_0$ is
that term at the mask-only weights.
\end{theorem*}

The proof of (e) uses one elementary lemma.

\begin{lemma}[Exchange]\label{lem:exchange}
Let $\mathcal D$ be a set, let $g,P:\mathcal D\to\mathbb R$ be functions, and let $\lambda_1<\lambda_2$ be
real numbers. If $x_i\in\mathcal D$ minimises $g+\lambda_iP$ over $\mathcal D$ for $i=1,2$, then
$P(x_2)\le P(x_1)$. If in addition $\lambda_1\ge0$, then $g(x_2)\ge g(x_1)$.
\end{lemma}
\begin{proof}
Optimality of $x_1$ at $\lambda_1$ and of $x_2$ at $\lambda_2$ gives
$g(x_1)+\lambda_1P(x_1)\le g(x_2)+\lambda_1P(x_2)$ and $g(x_2)+\lambda_2P(x_2)\le g(x_1)+\lambda_2P(x_1)$.
Adding the two inequalities gives $(\lambda_2-\lambda_1)\big(P(x_2)-P(x_1)\big)\le0$, and since
$\lambda_2-\lambda_1>0$, $P(x_2)\le P(x_1)$. The first inequality can be rewritten as
$g(x_1)-g(x_2)\le\lambda_1\big(P(x_2)-P(x_1)\big)$. When $\lambda_1\ge0$ the right-hand side is the product
of a non-negative and a non-positive number, hence at most $0$, so $g(x_2)\ge g(x_1)$.
\end{proof}

\begin{proof}[Proof of Theorem~\ref{thm:confine}]
Throughout, $\psi(u):=u_+^2$ for $u\in\mathbb R$, so that the penalty is
$\sum_{j\in I_-}\psi\big((\tilde A\delta-c)_j\big)$.

\emph{(a)} The function $\psi$ is convex and non-decreasing, being the composition of the convex map
$u\mapsto u_+$ with the map $s\mapsto s^2$, which is convex and non-decreasing on $[0,\infty)$. It is
continuously differentiable with $\psi'(u)=2u_+$, and $|\psi'(u)-\psi'(u')|\le2|u-u'|$, so $\psi'$ is
$2$-Lipschitz. Each penalty summand is $\psi$ composed with an affine function of $\delta$, hence convex and
continuously differentiable with Lipschitz gradient, and $\|A\delta-\zeta\|^2$ is a convex quadratic whose
gradient $2A^\top(A\delta-\zeta)$ is Lipschitz. A non-negative combination of such functions, here with weights
$1$ and $\lambda>0$, has the same three properties.

\emph{(b)} Since $\psi\ge0$ and $\psi(u)=0$ if and only if $u\le0$, the penalty vanishes if and only if
$(\tilde A\delta-c)_j\le0$ for every $j$, that is, on $\QM$; elsewhere it is the sum of the squared positive
violations. The containment $\QM\supseteq\PX\cap\{\hat w=\hat w\odot M\}$ holds because $\QM$ imposes only
the inactive block of $\PX$ (\S\ref{sec:partIV}).

\emph{(c), existence.} Consider the quadratic programme
\[
\min\big\{\|A\delta-\zeta\|^2+\lambda\|s\|^2:\ (\delta,s)\in\mathbb R^k\times\mathbb R^{n_-},\ s\ge\tilde
A\delta-c,\ s\ge0\big\}.
\]
Its feasible set is a polyhedron, and it is non-empty because it contains $(0,(-c)_+)$; its objective is a
quadratic bounded below by $0$ on it. By the Frank--Wolfe theorem \citep{frank1956algorithm}, a quadratic
function bounded below on a non-empty polyhedron attains its infimum there, so the programme has a minimiser
$(\bar\delta,\bar s)$. For
fixed $\delta$ the constraints on $s$ are the componentwise bounds $s_j\ge\max((\tilde A\delta-c)_j,0)$, and
$s_j^2$ is increasing on $[0,\infty)$, so the minimum over $s$ is attained at $s=(\tilde A\delta-c)_+$ and
equals $F_\lambda(\delta)$. Hence, for every $\delta$,
$F_\lambda(\delta)\ge\|A\bar\delta-\zeta\|^2+\lambda\|\bar s\|^2\ge F_\lambda(\bar\delta)$, where the second
inequality holds because $\bar s$ is feasible for $\bar\delta$. So $\bar\delta$ minimises $F_\lambda$.

\emph{(c), invariance.} Let $\delta_1,\delta_2$ be minimisers, let $F^\star$ be the minimum and let
$\delta_m=\frac12(\delta_1+\delta_2)$. Put $y_i=A\delta_i-\zeta$ and $u_i=\tilde A\delta_i-c$ for $i=1,2$. The
parallelogram identity gives
$\|\frac12(y_1+y_2)\|^2=\frac12\|y_1\|^2+\frac12\|y_2\|^2-\frac14\|y_1-y_2\|^2$, and for each $j$ the midpoint
gap $\nu_j:=\frac12\psi(u_{1j})+\frac12\psi(u_{2j})-\psi\big(\frac12(u_{1j}+u_{2j})\big)$ is non-negative
because $\psi$ is convex. Adding these,
\[
F^\star\le F_\lambda(\delta_m)=F^\star-\tfrac14\|y_1-y_2\|^2-\lambda\sum_j\nu_j .
\]
Since $\lambda>0$ and every subtracted term is non-negative, $A\delta_1=A\delta_2$ and $\nu_j=0$ for every
$j$. We show that $\nu_j=0$ forces $(u_{1j})_+=(u_{2j})_+$. Write $a=\min(u_{1j},u_{2j})$ and
$a'=\max(u_{1j},u_{2j})$, and consider three exhaustive cases. If $a'\le0$, both positive parts are $0$. If
$a\ge0$, then $\psi(u)=u^2$ on $[a,a']$ and $\nu_j=\frac14(a'-a)^2$, so $a=a'$. If $a<0<a'$, the midpoint
$\bar u=\frac12(a+a')$ satisfies $\bar u<\frac12a'$, so $\psi(\bar u)\le\frac14a'^2$ and
$\nu_j\ge\frac12a'^2-\frac14a'^2>0$, which is excluded. Hence $(\tilde A\delta_1-c)_+=(\tilde A\delta_2-c)_+$,
and $\gated$ and $\Esc$, which are functions of $A\delta$ and of $(\tilde A\delta-c)_+$, take one value each
on the minimiser set.

\emph{(d)} At $\delta=0$, $F_\lambda(0)=\|\zeta\|^2+\lambda E_0$ (\S\ref{sec:partIV}). Let
$F_\lambda(\delta)\le F_\lambda(0)$. Since $\gated(\delta)\ge0$,
$\lambda\Esc(\delta)\le F_\lambda(\delta)\le\|\zeta\|^2+\lambda E_0$, and dividing by $\lambda>0$ gives
\eqref{eq:massbound}. Now let $\QM\neq\varnothing$. Then $\QM$ is a non-empty polyhedron on which the
quadratic $\gated$ is bounded below by $0$, so by the Frank--Wolfe theorem $G^\star$ is attained, at some
$\delta_Q\in\QM$; by (b) the penalty vanishes there, so $F_\lambda(\delta_Q)=G^\star$. Let
$F_\lambda(\delta)\le G^\star$. Then $\lambda\Esc(\delta)\le G^\star-\gated(\delta)\le G^\star$, which is
\eqref{eq:massbound2}, and a minimiser $\delta_\lambda$ qualifies because
$F_\lambda(\delta_\lambda)\le F_\lambda(\delta_Q)=G^\star$.

\emph{(e)} Let $0<\lambda_1<\lambda_2$. By (c) both $F_{\lambda_1}$ and $F_{\lambda_2}$ have minimisers, and
Lemma~\ref{lem:exchange} with $\mathcal D=\mathbb R^k$, $g=\gated$ and $P=\Esc$ gives
$e(\lambda_2)\le e(\lambda_1)$ and $g(\lambda_2)\ge g(\lambda_1)$; by (c) these values do not depend on which
minimisers are chosen.

\emph{(f)} If $\QM\neq\varnothing$, (d) gives $0\le e(\lambda)\le G^\star/\lambda\to0$. Conversely, let
$\QM=\varnothing$. The programme $\min\{\|s\|^2:(\delta,s)\in\mathbb R^k\times\mathbb R^{n_-},\ s\ge\tilde
A\delta-c,\ s\ge0\}$ attains its minimum by the Frank--Wolfe theorem, as in (c), and by the same partial
minimisation over $s$ that minimum equals $\min_\delta\Esc(\delta)$, which is therefore attained. It is
positive: a minimiser $\delta'$ with $\Esc(\delta')=0$ would satisfy $\tilde A\delta'\le c$, that is,
$\delta'\in\QM=\varnothing$. Hence $e(\lambda)\ge\min_\delta\Esc>0$ for every $\lambda>0$, and $e(\lambda)$
does not tend to $0$.

Finally, for $\theta=0$ the escape $\sum_{t\in I_-}(\hat z_t)_+^2$ is the sum of $\hat z_t^{\,2}$ over the
closed tokens with $\hat z_t>0$, which is the closed$\to$open term, and at $\delta=0$ it is $E_0$.
\end{proof}

\begin{remark}[Certifying the feasibility hypothesis]\label{rem:E0}
The bound that tends to zero is \eqref{eq:massbound2}, and its hypothesis is cheap to certify: $\QM\neq\varnothing$
is the feasibility of the linear system $\tilde A\delta\le c$, and the vector $\hat w=0$ is a witness whenever
the unit's bias satisfies $b\le\theta$, since every token then has $z_{\hat w}(x)=b\le\theta$. At $\theta=0$
this holds for $81.2\%$ of OPT-1.3b and $95.7\%$ of OPT-125m wall units (a property of the biases alone). This
witness certifies membership of $\QM$, which asks only the inactive block;
when both blocks are non-empty, $\hat w=0$ lies in $\PX$ itself only when $b=\theta$, which is why the
theorem is stated over $\QM$. By Theorem~\ref{thm:confine}(f), a masked weight vector respecting the
closed side is exactly what confinement requires.
\end{remark}

\begin{remark}[The $\theta>0$ case, and why the penalty is written with $\max(u-\theta,0)$]
\label{rem:fatconvex}
Writing the penalty as $\act_\theta(z_{\hat w})^2$ instead, the form our implementation uses,
gives $u^2\,\mathds 1[u>\theta]$, which for $\theta>0$ jumps from $0$ to $\theta^2$ at $u=\theta$ and is
neither continuous nor convex. For $\theta=0$ the two forms coincide, so every ReLU result optimises the
analysed objective. For $\theta>0$ the implemented penalty
$P(\hat w)=\sum_{\mathcal X\cap\mathcal C_\le}\act_\theta(z_{\hat w})^2$ is exactly the
closed$\to$open term of Eq.~\eqref{eq:identity}, it dominates the escape pointwise and it is zero on
$\QM$. The argument of Theorem~\ref{thm:confine}(d) uses only these facts and non-negativity, so for the
implemented objective $\gated+\lambda P$: every $\hat w$ with objective at most that of $w\odot M$ has
closed$\to$open mass at most $\|\zeta\|^2/\lambda+P(w\odot M)$, and, if $\QM\neq\varnothing$, every $\hat w$
with objective at most $G^\star$ has closed$\to$open mass at most $G^\star/\lambda$. These statements use
$\act_\theta(u)=u\,\mathds 1[u>\theta]$ of Definition~\ref{def:wall}. Under the implementation's
convention $u\ge\theta$, the penalty at a point of $\QM$ equals $\theta^2$ times the number of inactive
tokens on the wall, so it vanishes on $\{\tilde A\delta<c\}$; when that set is non-empty it is dense in
$\QM$, so for every $\epsilon>0$ some $\hat w$ has objective at most $G^\star+\epsilon$, and every such
$\hat w$ has closed$\to$open mass at most $(G^\star+\epsilon)/\lambda$.
\end{remark}

The objectives $\postw(\lambda)$ price the exact objective in the same way, for the model's activation $\act$:
$\postw(\lambda)(\hat w)=\sum_{t\in I_+}(\act(\hat z_t)-\act(z_t))^2+\lambda\sum_{t\in I_-}(\act(\hat z_t)-\act(z_t))^2$,
so that $\postw(1)$ is the exact post-activation error; for $\act=\act_\theta$ it is
$\sum_{t\in I_+}(\act_\theta(\hat z_t)-z_t)^2+\lambda\sum_{t\in I_-}\act_\theta(\hat z_t)^2$ and
$\postw(1)=\post$. On a SwiGLU block the suffix $v$ ($\gatedhv$, $\postwv$) weights every token term by
$v_t^2$, where $v_t$ is the unit's up-projection output on token $t$, which the repair does not change.

\subsection{Predicted activation-pattern signatures}\label{sec:signature}

With the mask fixed and only the functional varied, Theorem~\ref{thm:escape} leaves the inactive faces
unpriced by $\gated$, which should therefore show a closed$\to$open share far above the other objectives, while
Theorem~\ref{thm:confine} and Remark~\ref{rem:fatconvex} bound that mass for $\gatedh$, and every $\hat w$
with $\post(\hat w)\le\post(w\odot M)$ has closed$\to$open mass at most $\post(w\odot M)$, since each term of
Eq.~\eqref{eq:identity} is non-negative.
Measured against the dense pattern (Table~\ref{tab:bucketsfull}, OPT-1.3b, $s=0.7$), $\gated$'s
crossings are closed$\to$open by $1.15\%$ to $0.18\%$ of pairs, and $\pre$'s run the other way,
open$\to$closed by $0.52\%$ to $0.34\%$; the closed$\to$open term holds $0.372$ of $\gated$'s residual,
against $0.171$ for $\gatedh(1)$, $0.111$ for $\post$ and $0.049$ for $\gatedh(30)$.

%% file: app_theory.tex
\section{Proof of Theorem~\ref{thm:escape}}\label{app:proof}

We use the notation for one row of \S\ref{sec:partIV}: the bias is fixed at $b$,
$\delta=\hat w_K-w_K\in\mathbb R^k$, $\gated=\|A\delta-\zeta\|^2$ and $\Esc=\|(\tilde A\delta-c)_+\|^2$ with
$c=(\theta-z^M_t)_{t\in I_-}$. For a matrix $B$, $B^+$ denotes its Moore--Penrose pseudo-inverse.

\begin{theorem*}[Full form of Theorem~\ref{thm:escape}]
Let $\theta\ge0$, and let one unit, its mask and its calibration tokens be as in \S\ref{sec:partIV}.
\par\textbf{(i)} The objective $\gated(\delta)=\|A\delta-\zeta\|^2$ depends on the calibration tokens only
through $A$ and $\zeta$. Changing the inactive rows $\tilde A$ or the inactive pre-activations $z^M_{I_-}$
changes neither $\gated$ nor its minimiser set, which is $A^+\zeta+\ker A$.
\par\textbf{(ii)} Assume $\ker A\not\subseteq\ker\tilde A$. This holds if $\operatorname{rank}X_K=k$ and
$\operatorname{rank}A<k$, and $\operatorname{rank}A<k$ holds whenever $n_+<k$. Then there is $v\in\ker A$
such that, for every minimiser $\delta$ of $\gated$ and every $\tau\ge0$, the point $\delta+\tau v$ is a
minimiser, and $\Esc(\delta+\tau v)\to\infty$ as $\tau\to\infty$.
\par\textbf{(iii)} Let the calibration tokens $x_1,\dots,x_n$ be i.i.d.\ draws from a distribution $D$ on
$\mathbb R^d$ with $\mathbb E\|x\|^2<\infty$, let the mask be fixed, and assume
$\Sigma_+:=\mathbb E\big[\mathds 1[z_w(x)>\theta]\,x_Kx_K^\top\big]\succ0$. Put
$h:=\mathbb E\big[\mathds 1[z_w(x)>\theta]\,x_K\langle w_{K^c},x_{K^c}\rangle\big]\in\mathbb R^k$ and
$\delta^\star:=\Sigma_+^{-1}h$, let $w^\star$ be the masked vector with $w^\star_K=w_K+\delta^\star$, and put
\[
e^\star:=\mathbb E\Big[\mathds 1[z_w(x)\le\theta]\,\big(z_{w^\star}(x)-\theta\big)_+^2\Big].
\]
Then, almost surely, $A$ has full column rank for all large $n$, so that the minimiser
$\delta_n=(A^\top A)^{-1}A^\top \zeta$ of $\gated$ is unique; $\delta_n\to\delta^\star$; and
$\frac1n\Esc(\delta_n)\to e^\star$. The point $w^\star$ depends on $D$ only through its restriction to
$\mathcal C_>$, while $e^\star$ depends on its restriction to $\mathcal C_\le$ as well. Finally, $e^\star\ge0$,
and $e^\star>0$ if and only if $D\big(z_w(x)\le\theta<z_{w^\star}(x)\big)>0$.
\end{theorem*}

The body's affine subspace in (ii) is the minimiser set of (i), and in body (iii) ``a point $w^\star$ the
active tokens determine'' means that $w^\star$ depends on $D$ only through its restriction to $\mathcal C_>$.

\begin{proof}
\emph{(i)} The formula $\gated(\delta)=\|A\delta-\zeta\|^2$ involves $A$ and $\zeta$ only. Its minimisers are the
solutions of the normal equations $A^\top A\delta=A^\top \zeta$. The vector $A^+\zeta$ is one of them, because
$A^\top AA^+=A^\top$, and two solutions differ by an element of $\ker(A^\top A)=\ker A$, since $A^\top Au=0$
implies $\|Au\|^2=u^\top A^\top Au=0$; so the minimiser set is $A^+\zeta+\ker A$. Neither $\tilde A$ nor
$z^M_{I_-}$ occurs in any of this: they enter only the inactive constraints
$(\tilde A\delta)_j\le\theta-z^M_j$, $j\in I_-$. Finally $\ker A\cap\ker\tilde A=\ker X_K$, because the rows
of $X_K$ are those of $A$ together with those of $\tilde A$.

\emph{(ii)} Since $\ker A\not\subseteq\ker\tilde A$, there is $v\in\ker A$ with $\tilde Av\neq0$, and so an
index $j\in I_-$ with $(\tilde Av)_j\neq0$; replacing $v$ by $-v\in\ker A$ if necessary, $(\tilde Av)_j>0$.
For the sufficient condition: if $\operatorname{rank}X_K=k$ and $\operatorname{rank}A<k$, then $\ker A$
contains some $u\neq0$, and $u\notin\ker\tilde A$ because $\ker A\cap\ker\tilde A=\ker X_K=\{0\}$; and
$\operatorname{rank}A\le n_+$, so $n_+<k$ gives $\operatorname{rank}A<k$. Now let $\delta$ be a minimiser
and $\tau\ge0$. Since $A(\delta+\tau v)=A\delta$, the point $\delta+\tau v$ has the minimal value of
$\gated$, so it is a minimiser. Token $j$'s pre-activation along the ray is
$\hat z_j(\tau)=z^M_j+(\tilde A\delta)_j+\tau(\tilde Av)_j$, affine and strictly increasing in $\tau$.
Hence it exceeds $\theta$ for every $\tau$ beyond a finite threshold, so beyond it the ray violates the
$j$-th inactive constraint and lies outside $\QM$ and $\PX$; and
\[
\Esc(\delta+\tau v)\ \ge\ \big(\hat z_j(\tau)-\theta\big)_+^2\ \xrightarrow[\tau\to\infty]{}\ \infty .
\]

\emph{(iii), the minimiser.} By the strong law of large numbers,
$\frac1nA^\top A=\frac1n\sum_t\mathds 1[z_t>\theta]\,x_{t,K}x_{t,K}^\top\to\Sigma_+$ and
$\frac1nA^\top \zeta=\frac1n\sum_t\mathds 1[z_t>\theta]\,x_{t,K}\langle w_{K^c},x_{t,K^c}\rangle\to h$ almost
surely; both summands are integrable, since their norms are at most $\|x\|^2$ and $\|w\|\,\|x\|^2$ and
$\mathbb E\|x\|^2<\infty$. The smallest eigenvalue is a continuous function of a symmetric matrix and
$\Sigma_+\succ0$, so $\frac1nA^\top A\succ0$ for all large $n$. Then $A$ has full column rank, $\gated$ is
strictly convex, and its unique minimiser is $\delta_n=(\frac1nA^\top A)^{-1}\frac1nA^\top \zeta$, which
converges to $\delta^\star$ because inversion is continuous on invertible matrices. The population
objective $\mathbb E\big[\mathds 1[z_w(x)>\theta]\,(\langle\delta,x_K\rangle-\langle
w_{K^c},x_{K^c}\rangle)^2\big]$ is a convex quadratic in $\delta$ with Hessian $2\Sigma_+\succ0$ and unique
minimiser $\delta^\star$; its indicator is evaluated at the dense pre-activation and does not depend on
$\delta$. Both $\Sigma_+$ and $h$ are expectations of functions that vanish off $\mathcal C_>$, so
$w^\star$ depends on $D$ only through its restriction to $\mathcal C_>$; the indicator
$\mathds 1[z_w(x)\le\theta]$ in $e^\star$ brings in the restriction to $\mathcal C_\le$.

\emph{(iii), the escape.} For $x\in\mathbb R^d$ and $\delta\in\mathbb R^k$ let
$f(x,\delta)=\mathds 1[z_w(x)\le\theta]\,\big(z_{w\odot M}(x)+\langle\delta,x_K\rangle-\theta\big)_+^2$, so
that $\Esc(\delta)=\sum_tf(x_t,\delta)$. For reals $a$ and $a'$,
$|a_+^2-(a')_+^2|\le|a-a'|\,(|a|+|a'|)$, and $|z_{w\odot M}(x)-\theta|\le|b-\theta|+\|w\|\,\|x\|$. Hence, for
$\delta,\delta'\in\mathbb R^k$,
\[
|f(x,\delta)-f(x,\delta')|\le\|\delta-\delta'\|\,\|x\|\,\big(2|b-\theta|+(2\|w\|+\|\delta\|+\|\delta'\|)\|x\|\big)
\le C(\delta,\delta')\,\|\delta-\delta'\|\,(1+\|x\|^2),
\]
with $C(\delta,\delta')=2|b-\theta|+2\|w\|+\|\delta\|+\|\delta'\|$, using $\|x\|\le1+\|x\|^2$ and
$\|x\|^2\le1+\|x\|^2$. Since $\delta_n\to\delta^\star$ almost surely, the sequence $\delta_n$ is bounded, so
$C(\delta_n,\delta^\star)\le\bar C$ for a finite $\bar C$ and all large $n$, and
\[
\Big|\frac1n\sum_tf(x_t,\delta_n)-\frac1n\sum_tf(x_t,\delta^\star)\Big|\le
\bar C\,\|\delta_n-\delta^\star\|\,\frac1n\sum_t(1+\|x_t\|^2)\ \longrightarrow\ 0
\]
almost surely, because the average converges to $1+\mathbb E\|x\|^2<\infty$. The bound holds for every
realisation of the data, so the dependence of $\delta_n$ on the same tokens does not matter. Finally
$\frac1n\sum_tf(x_t,\delta^\star)\to\mathbb E f(x,\delta^\star)=e^\star$ by the strong law, since
$0\le f(x,\delta^\star)\le\big(|b-\theta|+(\|w\|+\|\delta^\star\|)\|x\|\big)^2$ is integrable.

\emph{(iii), positivity.} The integrand of $e^\star$ is non-negative, so $e^\star\ge0$, and the positivity
clause is Lemma~\ref{lem:positive} below with $w^\circ=w^\star$.
\end{proof}

\begin{lemma}[When the limit escape is positive]\label{lem:positive}
In the setting of (iii), let $w^\circ\in\mathbb R^d$ be any masked weight vector and
$e(w^\circ):=\mathbb E\big[\mathds 1[z_w(x)\le\theta]\,(z_{w^\circ}(x)-\theta)_+^2\big]$, so that
$e^\star=e(w^\star)$. Then $e(w^\circ)>0$ if and only if $D\big(z_w(x)\le\theta<z_{w^\circ}(x)\big)>0$.
\end{lemma}

\begin{proof}
The integrand of $e(w^\circ)$ is non-negative and is positive exactly on the set
$\{z_w\le\theta<z_{w^\circ}\}$; an expectation of a non-negative function is positive if and only if the
function is positive on a set of positive measure, which gives the equivalence.
\end{proof}

\paragraph{The regime of our measurements.} At OPT-1.3b, $s=0.7$, the closed-form solver of
Appendix~\ref{app:repair:exp} fits $57\,344$ tokens; a row of gate density $0.035$ has $n_+\approx2\,007$
active tokens against $k\approx614$ kept columns, yet $39.2\%$ of rows have $n_+<k$, so both (ii) and (iii) occur.

\section{Experimental setup and noise floors}\label{app:setup}

\begin{table}[!htb]\centering\small
\caption{Models. The \emph{wall matrix} is the one whose output passes through the activation and which the
repair modifies; $\theta$ is recovered at load time by bisection on the resolved \texttt{act\_fn}; the density
column is the mask-only gate density at sparsity $s$ (Pythia-160m has none). Activations: ReLU, GELU
\citep{hendrycks2016gelu} and the thresholded FATReLU \citep{kurtz2020fatrelu}; SwiGLU blocks
\citep{shazeer2020glu}.}
\label{tab:models}
{\setlength{\tabcolsep}{4pt}%
\begin{tabular}{lllcrrr}
\toprule
model & activation & block & $\theta$ & $n_{\mathrm{wall}}$ & $s$ & mask-only density \\
\midrule
OPT-125m \citep{zhang2022opt}         & ReLU    & plain FFN & $0$      & 12 & $0.7$ & $0.0589$ \\
OPT-350m         & ReLU    & plain FFN & $0$      & 24 & $0.7$ & $0.0848$ \\
OPT-1.3b         & ReLU    & plain FFN & $0$      & 24 & $0.7$ & $0.0352$ \\
Pythia-160m \citep{biderman2023pythia} & GELU  & plain FFN & $0$      & 12 & $0.7$ & --- \\
Pythia-410m      & GELU    & plain FFN & $0$      & 24 & $0.5$ & $0.1428$ \\
MiniCPM-S-1B-sft \citep{song2025prosparse} & FATReLU & SwiGLU & $0.03$ & 52 & $0.5$ & $0.1899$ \\
ProSparse-7B \citep{song2025prosparse}    & FATReLU & SwiGLU    & $0.01$   & 32 & $0.6$ & $0.163$ \\
ReluLLaMA-7B \citep{sparsellm2023relullamafix} & ReLU  & SwiGLU    & $0$      & 32 & $0.6$ & $0.339$ \\
\bottomrule
\end{tabular}}
\end{table}

\paragraph{Activation and masks.} The activation is resolved from the model config and verified against the
loaded \texttt{act\_fn} on $8192$ random inputs, and the LLaMA-family models require a BOS token (ProSparse
scores $6.07$ nats on our wiki probe without it and $3.54$ with it). The mask is a Wanda-style saliency,
$\mathrm{sal}_{ij}=|W_{ij}|\cdot\|X_j\|_2$, ranked within each output row and thresholded at the target
sparsity. The $\|X_j\|_2$ statistics are accumulated in a single forward pass over the calibration set on the
\emph{dense} model, and the mask is frozen and shared by every objective in a cell. All \texttt{Linear} modules
inside transformer blocks are masked, and a $2{:}4$ mask keeps the top $2$ of every $4$ consecutive input
channels.

\paragraph{Repair.} Repair is applied to the wall matrices alone (\texttt{fc1} on OPT,
\texttt{dense\_h\_to\_4h} on GPT-NeoX, \texttt{gate\_proj} on LLaMA-family blocks), the matrices whose output
passes through the activation; attention and the output matrix are pruned and left mask-only for every objective.
Every objective is solved identically: Adam, learning rate
$3\times10^{-4}$, batch $1024$ reconstruction tokens, step cap $3000$, early stopping on a held-out slice of the
calibration tokens ($20\%$, at most $8\,192$ tokens) scored with its \emph{own} objective, the mask re-applied after every
step so that all objectives optimise the same free parameters, and the loss normalised by the target's own mean
square so that a shared learning rate does not favour one objective. The input $X$ to layer $\ell$ is taken from
the current, already pruned and repaired model (with dense inputs, $\pre$ and $\post$ are worse than no repair by
$0.08$ and $0.04$ nats on OPT-125m at $s=0.5$), and the target is the dense weights on those same inputs, the
convention of Eq.~\eqref{eq:identity}.

\paragraph{Data.} Calibration is drawn from a slice of MiniPile \citep{kaddour2023minipile}, a subset of The
Pile \citep{gao2020pile}; the held-out in-domain probe comes from the same file past an offset, with a guard
that the calibration span does not reach it, and the out-of-domain probe from a WikiText-103 slice.
WikiText-2 is the raw test set at context $2048$, cut into non-overlapping windows with no BOS per window,
on which our evaluation code reproduces the published dense perplexities, $27.663$ vs $27.65$ (OPT-125m) and
$14.624$ vs $14.62$ (OPT-1.3b).

\paragraph{Noise floors and evaluation.}\label{app:noise} At calibration seeds $4000$ and $91000$ the
no-repair baseline reproduces exactly, because the seed controls the reconstruction minibatch order, not the
saliency; the largest spread of a repair objective on the $256$-token probe, $\gatedh(30)$ on OPT-1.3b at
$s=0.7$, is the probe floor of \NoiseFloor{} nats used throughout and the equivalence margin $\delta$ of
Table~\ref{tab:necessity}. On WikiText-2 at context $2048$, calibration seeds $4000$, $77\,000$ and $91\,000$
give a spread over repair objectives with median $0.0079$ and maximum \WTFloorMax{} nats, whose 90th percentile,
\WTFloorVal{} nats, is the WikiText-2 floor. Evaluated inside the same run on identical repaired weights, the
$256$-token probe understates the repair effect of the published protocol on MiniCPM-S-1B by $2.02\times$ at
$s=0.4$ and $4.18\times$ at $s=0.5$, and on OPT-1.3b at $s=0.6$--$0.8$ the two evaluations rank $\pre$, $\gated$,
$\post$, $\gatedh$ and $\postw$ identically; WikiText-2 figures use the published protocol.

\section{Verification of the identity}\label{app:identity}

Eq.~\eqref{eq:identity} is an identity, so a check verifies the implementation. Every fine-tuning run forms both
sides of Eq.~\eqref{eq:identity} on every wall at every evaluation and stops if their relative gap reaches
$10^{-4}$; on OPT the logged gap is exactly zero. On pruned layers (the Wanda-style mask alone, no repair), each
pair's first-order residual $\act_\theta(z_{\hat w})-\act_\theta(z_w)-\mathds 1[z_w>\theta](z_{\hat w}-z_w)$,
formed on real activations, is compared with $|z_{\hat w}|$ on the pairs that change cell and with $0$ elsewhere,
the case split behind Eq.~\eqref{eq:identity}; the relative gap is at floating-point level in every cell
(Table~\ref{tab:identitygap}).

\begin{center}\tabcaptype\small
\caption{Relative gap of the case split behind Eq.~\eqref{eq:identity} on pruned layers (mask alone): the summed
absolute first-order residual against the summed $|z_{\hat w}|$ of the pairs that change cell, maximum over
the cells of each family.}
\label{tab:identitygap}
\input{tab_identitygap_compact}
\end{center}

\section{Repair experiments}\label{app:repair:exp}

\begin{table}[!htb]\centering\small
\caption{WikiText-2 perplexity at context 2048, OPT-1.3b, dense $=14.62$. One frozen mask per row; only
the functional changes. $\post<\pre$ at every sparsity, and at $3$ of the $4$ by more than the
WikiText-2 noise floor \WTFloorVal{} nats.}
\label{tab:main}
\begin{tabular}{lrrrrrr}
\toprule
$s$ & no repair & $\pre$ & $\gated$ & $\post$ & $\gatedh(30)$ & $\postw(10)$ \\
\midrule
$0.5$ & $19.76$ & $20.76$ & $20.41$ & $20.54$ & $\mathbf{20.34}$ & $20.54$ \\
$0.6$ & $29.37$ & $34.33$ & $31.56$ & $31.46$ & $\mathbf{30.04}$ & $30.91$ \\
$0.7$ & $97.25$ & $153.64$ & $100.66$ & $100.07$ & $\mathbf{84.50}$ & $94.12$ \\
$0.8$ & $2643.77$ & $2334.70$ & $879.56$ & $\mathbf{740.52}$ & $902.77$ & $776.50$ \\
\bottomrule
\end{tabular}
\end{table}

\begin{table}[!htb]\centering\small
\caption{Shares of the post-activation residual by term, and the fraction of $(t,i)$ pairs crossing in each
direction, OPT-1.3b $s=0.7$, averaged over the 24 wall matrices.}
\label{tab:bucketsfull}
\begin{tabular}{lrrrrr}
\toprule
& \multicolumn{3}{c}{residual share} & \multicolumn{2}{c}{flips} \\
\cmidrule(lr){2-4}\cmidrule(lr){5-6}
objective & in-cell & o$\to$c & c$\to$o & o$\to$c & c$\to$o \\
\midrule
$\gated$        & $.602$ & $.027$ & $\mathbf{.372}$ & $.0018$ & $.0115$ \\
$\gatedh(1)$    & $.767$ & $.061$ & $.171$ & $.0027$ & $.0070$ \\
$\post$         & $.747$ & $.143$ & $.111$ & $.0044$ & $.0046$ \\
$\pre$          & $.805$ & $.137$ & $.058$ & $.0052$ & $.0034$ \\
$\postw(10)$    & $.724$ & $.224$ & $\mathbf{.053}$ & $.0071$ & $.0025$ \\
$\gatedh(30)$   & $.802$ & $.150$ & $\mathbf{.049}$ & $.0065$ & $.0028$ \\
\bottomrule
\end{tabular}
\end{table}

\paragraph{Design.} The objectives compared in a cell share one frozen Wanda-style mask and everything else of
Appendix~\ref{app:setup}, so only the functional changes; probe gaps are compared with the floor of
\NoiseFloor{} nats and WikiText-2 gaps with WikiText-2's own (Appendix~\ref{app:noise}), and
Table~\ref{tab:main} gives the WikiText-2 perplexities on OPT-1.3b.

\paragraph{The degeneracy is in the objective, not the solver.} A degenerate minimiser is a claim about the
objective, so it is tested on the minimiser. Solving $\gated$ in closed form in \texttt{fp64} (ridge
$10^{-7}$ of the mean Gram diagonal), per row, with no step limit, costs $+1.9416$ nats against not
repairing on OPT-1.3b at $s=0.7$, and adds $237\%$ of the error it was asked to remove. Inside the same
solver, the identical \texttt{fp64} normal equations reproduce the iteratively solved pre-activation fit to
$-0.0029$ nats. The early-stopped $\gated$ fit is at $-0.01$ nats. Results reported for this family in the
pruning regime are therefore mediated by early stopping: the early-stopped fit has already left $\PX$, at
$1.271\times$ the mask-only gate density, but halts long before the
escape the closed-form solve reaches. $\pre$ stays accurate under the closed-form solve because it is a least-squares
problem in every calibration token, not because it respects $\PX$; $\gated$ discards the inactive rows.

\paragraph{Movement and stopping controls.} At the $3000$-step budget on OPT-1.3b, $s=0.7$, $\post$ reaches
$-0.1098$ nats at relative weight movement $0.255$, while \texttt{prematch}, which descends $\pre$ but halts
at $\post$'s movement, reaches $+0.2894$, and \texttt{pre\_pval}, which descends $\pre$ but selects its
checkpoint on the post-activation validation loss, reaches $+0.2258$.

\paragraph{Crossing directions.} Table~\ref{tab:bucketsfull} is the source of the crossing rates of
\S\ref{sec:signature}. Because Eq.~\eqref{eq:identity} is exact, each objective's residual also splits by term:
$\gated$ carries $37.2\%$ of its residual in closed$\to$open, $2.2\times$ the next-largest objective ($\gatedh$ at
$\lambda=1$, $17.1\%$), and the \emph{least} open$\to$closed mass ($2.7\%$). Both follow from what it charges:
nothing for a closed$\to$open token, since it sums only over $\mathcal C_{>}$, and, at $\theta=0$, at least
the true cost $z_w^2$ for an open$\to$closed one, since it charges $(z_{\hat w}-z_w)^2\ge z_w^2$ when
$z_{\hat w}\le0<z_w$.

\section{The \texorpdfstring{$\lambda$}{lambda} sweep and the \texorpdfstring{$\post<\pre$}{post < pre} counts}\label{app:lambda}

\begin{table}[!htb]\centering\small
\caption{Held-out probe loss of the best $\lambda>1$ member of an objective family against its $\lambda=1$
member, sharing configuration, no-repair anchor, calibration seed and budgets. \textbf{Bold}: the gap exceeds
$3\times$ the seed-noise floor (Appendix~\ref{app:noise}) and a second calibration seed reproduces its sign;
``seed 2 gap'' is the gap of the same $\lambda$ at that seed, ``other groups'' the gap in the other
comparability groups of the configuration.}
\label{tab:lambda}
\resizebox{\linewidth}{!}{\input{tab_lambda}}
\end{table}

Of the configurations in which a $\lambda=1$ and a $\lambda>1$ member of one family are comparable, a
$\lambda>1$ member attains lower held-out loss in \LamWin{}, in \LamWinFloor{} by more than the noise floor,
on plain-FFN ReLU, GELU and a SwiGLU block with a thresholded ReLU (Table~\ref{tab:lambda}); \LamRows{} rows
are bold, at $5$ to $27$ times the floor, the second seed reproduces the sign in \LamSeedTwo{} cases, and the
best $\lambda$ runs from $3$ to $30$. Orderings are read within runs, because across runs that measure one
configuration under different protocols $\post-\pre$ spreads by up to $1.04$ nats.

\begin{center}\small
\begin{tabular}{lrl}
\toprule
population & $\post<\pre$ & exception \\
\midrule
probe cells, $256$-token evaluation           & \PostPreProbe & OPT-125m $s{=}0.3$, $+0.0058$ \\
configurations, under \emph{every} protocol    & \PostPreConfig & the same one \\
WikiText-2 cells, $2048$-token evaluation    & \PostPreWT & --- \\
7B ReLU-family cells                         & \SevenB & --- \\
\bottomrule
\end{tabular}
\end{center}

The probe cells span OPT-125m, OPT-350m, OPT-1.3b and two GELU models, the WikiText-2 cells the three OPT
models and a SwiGLU block with a thresholded ReLU, and the 7B cells ReluLLaMA-7B and ProSparse-7B over four
unstructured sparsities and a $2{:}4$ mask, at two calibration seeds. The 7B gaps run $0.0086$--$0.1594$ nats, all beyond that
family's seed-noise floor of \SevenBFloorVal{} nats; collapsing the two seeds, which share a mask, leaves $10$
independent configurations, all negative, a sign test at $p=2^{-10}=9.8\times10^{-4}$.

%% file: tab_identitygap_compact.tex
\begin{tabular}{llrrr}
\toprule
family & models & $s$ & cells & max.\ identity gap \\
\midrule
OPT & OPT-125m, OPT-350m, OPT-1.3b & $0.1$--$0.9$ & $27$ & $5.80\times10^{-5}$ \\
7B ReLU family & ReluLLaMA-7B, ProSparse-7B & $0.1$--$0.9$ & $18$ & $9.02\times10^{-5}$ \\
\bottomrule
\end{tabular}

%% file: tab_lambda.tex
\begin{tabular}{llrrlrr}
\toprule
model & activation & $\lambda=1$ & best $\lambda>1$ & attained at & seed 2 gap & other groups \\
\midrule
MiniCPM-S-1B, $s{=}0.5$ & SwiGLU$+$FATReLU & $3.4636$ & $3.4594$ & $\postwv(10)$ & --- & $+0.0043$ \\
OPT-1.3b, $s{=}0.5$ & ReLU & $3.6847$ & $3.6805$ & $\postw(10)$ & --- & --- \\
OPT-1.3b, $s{=}0.6$ & ReLU & $4.0404$ & $\mathbf{3.9915}$ & $\gatedh(30)$ & $+0.0437$ & $+0.0283$ to $+0.0489$ \\
OPT-1.3b, $s{=}0.7$ & ReLU & $4.9442$ & $\mathbf{4.7631}$ & $\gatedh(30)$ & $+0.1715$ & $+0.1072$ to $+0.1811$ \\
OPT-1.3b, $s{=}0.8$ & ReLU & $6.4546$ & $6.4381$ & $\postw(3)$ & --- & $+0.0062$ to $+0.0165$ \\
OPT-125m, $s{=}0.7$ & ReLU & $5.6261$ & $5.6121$ & $\postw(10)$ & $+0.0086$ & --- \\
OPT-350m, $s{=}0.7$ & ReLU & $5.9919$ & $\mathbf{5.8906}$ & $\gatedh(30)$ & $+0.0950$ & $+0.0812$ to $+0.1014$ \\
Pythia-410m, $s{=}0.5$ & GELU & $4.6791$ & $\mathbf{4.6407}$ & $\postw(10)$ & $+0.0401$ & --- \\
Pythia-410m, $s{=}0.6$ & GELU & $5.8639$ & $\mathbf{5.8273}$ & $\postw(10)$ & $+0.0727$ & --- \\
\bottomrule
\end{tabular}

%% file: main.bbl
\begin{thebibliography}{63}
\providecommand{\natexlab}[1]{#1}
\providecommand{\url}[1]{\texttt{#1}}
\expandafter\ifx\csname urlstyle\endcsname\relax
  \providecommand{\doi}[1]{doi: #1}\else
  \providecommand{\doi}{doi: \begingroup \urlstyle{rm}\Url}\fi

\bibitem[Abbasi et~al.(2022)Abbasi, Nooralinejad, Braverman, Pirsiavash, and
  Kolouri]{abbasi2022shd}
Ali Abbasi, Parsa Nooralinejad, Vladimir Braverman, Hamed Pirsiavash, and
  Soheil Kolouri.
\newblock Sparsity and heterogeneous dropout for continual learning in the null
  space of neural activations.
\newblock In \emph{Conference on Lifelong Learning Agents (CoLLAs)}, volume 199
  of \emph{PMLR}, pp.\  617--628, 2022.

\bibitem[Aghasi et~al.(2017)Aghasi, Abdi, Nguyen, and
  Romberg]{aghasi2017nettrim}
Alireza Aghasi, Afshin Abdi, Nam Nguyen, and Justin Romberg.
\newblock {Net-Trim}: Convex pruning of deep neural networks with performance
  guarantee.
\newblock In \emph{Advances in Neural Information Processing Systems},
  volume~30, 2017.

\bibitem[Biderman et~al.(2024)Biderman, Portes, Gonzalez~Ortiz, Paul,
  Greengard, Jennings, King, Havens, Chiley, Frankle, Blakeney, and
  Cunningham]{biderman2024loraforgets}
Dan Biderman, Jacob Portes, Jose~Javier Gonzalez~Ortiz, Mansheej Paul, Philip
  Greengard, Connor Jennings, Daniel King, Sam Havens, Vitaliy Chiley, Jonathan
  Frankle, Cody Blakeney, and John~P. Cunningham.
\newblock {LoRA} learns less and forgets less.
\newblock \emph{Transactions on Machine Learning Research}, 2024.

\bibitem[Biderman et~al.(2023)Biderman, Schoelkopf, Anthony, Bradley, O'Brien,
  Hallahan, Khan, Purohit, Prashanth, Raff, Skowron, Sutawika, and van~der
  Wal]{biderman2023pythia}
Stella Biderman, Hailey Schoelkopf, Quentin Anthony, Herbie Bradley, Kyle
  O'Brien, Eric Hallahan, Mohammad~Aflah Khan, Shivanshu Purohit, USVSN~Sai
  Prashanth, Edward Raff, Aviya Skowron, Lintang Sutawika, and Oskar van~der
  Wal.
\newblock Pythia: A suite for analyzing large language models across training
  and scaling.
\newblock In \emph{Proceedings of the 40th International Conference on Machine
  Learning}, volume 202 of \emph{Proceedings of Machine Learning Research}.
  PMLR, 2023.

\bibitem[Botev et~al.(2017)Botev, Ritter, and Barber]{botev2017practical}
Aleksandar Botev, Hippolyt Ritter, and David Barber.
\newblock Practical {Gauss-Newton} optimisation for deep learning.
\newblock In \emph{International Conference on Machine Learning (ICML)}, 2017.
\newblock arXiv:1706.03662.

\bibitem[Clark et~al.(2018)Clark, Cowhey, Etzioni, Khot, Sabharwal, Schoenick,
  and Tafjord]{clark2018arc}
Peter Clark, Isaac Cowhey, Oren Etzioni, Tushar Khot, Ashish Sabharwal, Carissa
  Schoenick, and Oyvind Tafjord.
\newblock Think you have solved question answering? {Try} {ARC}, the {AI2}
  reasoning challenge, 2018.
\newblock arXiv:1803.05457.

\bibitem[{CodeParrot}(2022)]{codeparrot2022clean}
{CodeParrot}.
\newblock {CodeParrot} dataset cleaned, validation split.
\newblock
  \url{https://huggingface.co/datasets/codeparrot/codeparrot-clean-valid},
  2022.
\newblock Python files from public GitHub repositories, deduplicated and
  filtered.

\bibitem[DerSimonian \& Laird(1986)DerSimonian and Laird]{dersimonian1986meta}
Rebecca DerSimonian and Nan Laird.
\newblock Meta-analysis in clinical trials.
\newblock \emph{Controlled Clinical Trials}, 7\penalty0 (3):\penalty0 177--188,
  1986.

\bibitem[Du et~al.(2019)Du, Zhai, P{\'o}czos, and Singh]{du2019gradient}
Simon~S. Du, Xiyu Zhai, Barnab{\'a}s P{\'o}czos, and Aarti Singh.
\newblock Gradient descent provably optimizes over-parameterized neural
  networks.
\newblock In \emph{International Conference on Learning Representations
  (ICLR)}, 2019.
\newblock arXiv:1810.02054.

\bibitem[Du et~al.(2024)Du, Cheng, Luo, Qiu, Huang, Cheung, Cheng, and
  Fu]{du2024migu}
Wenyu Du, Shuang Cheng, Tongxu Luo, Zihan Qiu, Zeyu Huang, Ka~Chun Cheung,
  Reynold Cheng, and Jie Fu.
\newblock Unlocking continual learning abilities in language models.
\newblock In \emph{Findings of the Association for Computational Linguistics:
  EMNLP}, 2024.

\bibitem[Fang et~al.(2025)Fang, Jiang, Wang, Ma, Jie, Wang, He, and
  Chua]{fang2025alphaedit}
Junfeng Fang, Houcheng Jiang, Kun Wang, Yunshan Ma, Shi Jie, Xiang Wang,
  Xiangnan He, and Tat-seng Chua.
\newblock {AlphaEdit}: Null-space constrained knowledge editing for language
  models.
\newblock In \emph{International Conference on Learning Representations
  (ICLR)}, 2025.
\newblock arXiv:2410.02355.

\bibitem[Farajtabar et~al.(2020)Farajtabar, Azizan, Mott, and
  Li]{farajtabar2020ogd}
Mehrdad Farajtabar, Navid Azizan, Alex Mott, and Ang Li.
\newblock Orthogonal gradient descent for continual learning.
\newblock In \emph{International Conference on Artificial Intelligence and
  Statistics (AISTATS)}, 2020.
\newblock arXiv:1910.07104.

\bibitem[Frank \& Wolfe(1956)Frank and Wolfe]{frank1956algorithm}
Marguerite Frank and Philip Wolfe.
\newblock An algorithm for quadratic programming.
\newblock \emph{Naval Research Logistics Quarterly}, 3\penalty0
  (1--2):\penalty0 95--110, 1956.

\bibitem[Frantar \& Alistarh(2023)Frantar and Alistarh]{frantar2023sparsegpt}
Elias Frantar and Dan Alistarh.
\newblock {SparseGPT}: Massive language models can be accurately pruned in
  one-shot.
\newblock In \emph{Proceedings of the 40th International Conference on Machine
  Learning}, volume 202 of \emph{Proceedings of Machine Learning Research}.
  PMLR, 2023.

\bibitem[Frantar et~al.(2023)Frantar, Ashkboos, Hoefler, and
  Alistarh]{frantar2023gptq}
Elias Frantar, Saleh Ashkboos, Torsten Hoefler, and Dan Alistarh.
\newblock {GPTQ}: Accurate post-training quantization for generative
  pre-trained transformers.
\newblock In \emph{International Conference on Learning Representations}, 2023.

\bibitem[Gao et~al.(2020)Gao, Biderman, Black, Golding, Hoppe, Foster, Phang,
  He, Thite, Nabeshima, Presser, and Leahy]{gao2020pile}
Leo Gao, Stella Biderman, Sid Black, Laurence Golding, Travis Hoppe, Charles
  Foster, Jason Phang, Horace He, Anish Thite, Noa Nabeshima, Shawn Presser,
  and Connor Leahy.
\newblock The {Pile}: An 800{GB} dataset of diverse text for language modeling.
\newblock \emph{arXiv preprint arXiv:2101.00027}, 2020.

\bibitem[Gao et~al.(2024)Gao, Tow, Abbasi, Biderman, Black, DiPofi, Foster,
  Golding, Hsu, Le~Noac'h, Li, McDonell, Muennighoff, Ociepa, Phang, Reynolds,
  Schoelkopf, Skowron, Sutawika, Tang, Thite, Wang, Wang, and
  Zou]{gao2023lmeval}
Leo Gao, Jonathan Tow, Baber Abbasi, Stella Biderman, Sid Black, Anthony
  DiPofi, Charles Foster, Laurence Golding, Jeffrey Hsu, Alain Le~Noac'h,
  Haonan Li, Kyle McDonell, Niklas Muennighoff, Chris Ociepa, Jason Phang,
  Laria Reynolds, Hailey Schoelkopf, Aviya Skowron, Lintang Sutawika, Eric
  Tang, Anish Thite, Ben Wang, Kevin Wang, and Andy Zou.
\newblock The language model evaluation harness, 2024.
\newblock Zenodo record of v0.4.3; we use v0.4.12.

\bibitem[Hendrycks \& Gimpel(2016)Hendrycks and Gimpel]{hendrycks2016gelu}
Dan Hendrycks and Kevin Gimpel.
\newblock Gaussian error linear units ({GELU}s).
\newblock \emph{arXiv preprint arXiv:1606.08415}, 2016.

\bibitem[Heo et~al.(2019)Heo, Lee, Yun, and Choi]{heo2019boundary}
Byeongho Heo, Minsik Lee, Sangdoo Yun, and Jin~Young Choi.
\newblock Knowledge transfer via distillation of activation boundaries formed
  by hidden neurons.
\newblock In \emph{AAAI Conference on Artificial Intelligence}, 2019.
\newblock arXiv:1811.03233.

\bibitem[Hu et~al.(2022)Hu, Shen, Wallis, Allen-Zhu, Li, Wang, Wang, and
  Chen]{hu2022lora}
Edward~J. Hu, Yelong Shen, Phillip Wallis, Zeyuan Allen-Zhu, Yuanzhi Li, Shean
  Wang, Lu~Wang, and Weizhu Chen.
\newblock {LoRA}: Low-rank adaptation of large language models.
\newblock In \emph{International Conference on Learning Representations
  (ICLR)}, 2022.

\bibitem[Ibrahim et~al.(2024)Ibrahim, Th{\'e}rien, Gupta, Richter, Anthony,
  Lesort, Belilovsky, and Rish]{ibrahim2024continual}
Adam Ibrahim, Benjamin Th{\'e}rien, Kshitij Gupta, Mats~L. Richter, Quentin
  Anthony, Timoth{\'e}e Lesort, Eugene Belilovsky, and Irina Rish.
\newblock Simple and scalable strategies to continually pre-train large
  language models.
\newblock \emph{Transactions on Machine Learning Research (TMLR)}, 2024.

\bibitem[Jiang et~al.(2025)Jiang, Fang, Zhang, Bi, Zhang, Wang, Liang, and
  Wang]{jiang2025nse}
Houcheng Jiang, Junfeng Fang, Tianyu Zhang, Baolong Bi, An~Zhang, Ruipeng Wang,
  Tao Liang, and Xiang Wang.
\newblock Neuron-level sequential editing for large language models.
\newblock In \emph{Annual Meeting of the Association for Computational
  Linguistics (ACL)}, pp.\  16678--16702, 2025.

\bibitem[Jung et~al.(2020)Jung, Ahn, Cha, and Moon]{jung2020agscl}
Sangwon Jung, Hongjoon Ahn, Sungmin Cha, and Taesup Moon.
\newblock Continual learning with node-importance based adaptive group sparse
  regularization.
\newblock In \emph{Advances in Neural Information Processing Systems
  (NeurIPS)}, 2020.
\newblock arXiv:2003.13726.

\bibitem[Kaddour(2023)]{kaddour2023minipile}
Jean Kaddour.
\newblock The {MiniPile} challenge for data-efficient language models.
\newblock \emph{arXiv preprint arXiv:2304.08442}, 2023.

\bibitem[Kim et~al.(2025)Kim, El~Halabi, Park, Schaefer, Lee, Park, Lee, and
  Song]{kim2025guidedquant}
Jinuk Kim, Marwa El~Halabi, Wonpyo Park, Clemens~Js Schaefer, Deokjae Lee,
  Yeonhong Park, Jae~W. Lee, and Hyun~Oh Song.
\newblock {GuidedQuant}: Large language model quantization via exploiting end
  loss guidance.
\newblock In \emph{Proceedings of the 42nd International Conference on Machine
  Learning}, volume 267 of \emph{Proceedings of Machine Learning Research},
  pp.\  30011--30037. PMLR, 2025.

\bibitem[Kirkpatrick et~al.(2017)Kirkpatrick, Pascanu, Rabinowitz, Veness,
  Desjardins, Rusu, Milan, Quan, Ramalho, Grabska-Barwinska, Hassabis, Clopath,
  Kumaran, and Hadsell]{kirkpatrick2017ewc}
James Kirkpatrick, Razvan Pascanu, Neil Rabinowitz, Joel Veness, Guillaume
  Desjardins, Andrei~A. Rusu, Kieran Milan, John Quan, Tiago Ramalho, Agnieszka
  Grabska-Barwinska, Demis Hassabis, Claudia Clopath, Dharshan Kumaran, and
  Raia Hadsell.
\newblock Overcoming catastrophic forgetting in neural networks.
\newblock \emph{Proceedings of the National Academy of Sciences}, 114\penalty0
  (13):\penalty0 3521--3526, 2017.

\bibitem[Kong et~al.(2022)Kong, Liu, Wang, and Tao]{kong2022adns}
Yajing Kong, Liu Liu, Zhen Wang, and Dacheng Tao.
\newblock Balancing stability and plasticity through advanced null space in
  continual learning.
\newblock In \emph{European Conference on Computer Vision (ECCV)}, 2022.
\newblock arXiv:2207.12061.

\bibitem[Kurtz et~al.(2020)Kurtz, Kopinsky, Gelashvili, Matveev, Carr, Goin,
  Leiserson, Moore, Shavit, and Alistarh]{kurtz2020fatrelu}
Mark Kurtz, Justin Kopinsky, Rati Gelashvili, Alexander Matveev, John Carr,
  Michael Goin, William Leiserson, Sage Moore, Nir Shavit, and Dan Alistarh.
\newblock Inducing and exploiting activation sparsity for fast inference on
  deep neural networks.
\newblock In \emph{Proceedings of the 37th International Conference on Machine
  Learning}, volume 119 of \emph{Proceedings of Machine Learning Research}.
  PMLR, 2020.

\bibitem[Li et~al.(2021)Li, Gong, Tan, Yang, Hu, Zhang, Yu, Wang, and
  Gu]{li2021brecq}
Yuhang Li, Ruihao Gong, Xu~Tan, Yang Yang, Peng Hu, Qi~Zhang, Fengwei Yu, Wei
  Wang, and Shi Gu.
\newblock {BRECQ}: Pushing the limit of post-training quantization by block
  reconstruction.
\newblock In \emph{International Conference on Learning Representations}, 2021.

\bibitem[Liang \& Li(2024)Liang and Li]{liang2024inflora}
Yan-Shuo Liang and Wu-Jun Li.
\newblock {InfLoRA}: Interference-free low-rank adaptation for continual
  learning.
\newblock In \emph{IEEE/CVF Conference on Computer Vision and Pattern
  Recognition (CVPR)}, 2024.

\bibitem[Lin et~al.(2022)Lin, Yang, Fan, and Zhang]{lin2022trgp}
Sen Lin, Li~Yang, Deliang Fan, and Junshan Zhang.
\newblock {TRGP}: Trust region gradient projection for continual learning.
\newblock In \emph{International Conference on Learning Representations
  (ICLR)}, 2022.
\newblock arXiv:2202.02931.

\bibitem[Loshchilov \& Hutter(2019)Loshchilov and Hutter]{loshchilov2019adamw}
Ilya Loshchilov and Frank Hutter.
\newblock Decoupled weight decay regularization.
\newblock In \emph{International Conference on Learning Representations
  (ICLR)}, 2019.
\newblock arXiv:1711.05101.

\bibitem[Maclagan \& Sturmfels(2015)Maclagan and
  Sturmfels]{maclagan2015tropical}
Diane Maclagan and Bernd Sturmfels.
\newblock \emph{Introduction to Tropical Geometry}, volume 161 of
  \emph{Graduate Studies in Mathematics}.
\newblock American Mathematical Society, 2015.

\bibitem[Maragos et~al.(2021)Maragos, Charisopoulos, and
  Theodosis]{maragos2021tropical}
Petros Maragos, Vasileios Charisopoulos, and Emmanouil Theodosis.
\newblock Tropical geometry and machine learning.
\newblock \emph{Proceedings of the IEEE}, 109\penalty0 (5), 2021.

\bibitem[Martens \& Grosse(2015)Martens and Grosse]{martens2015kfac}
James Martens and Roger Grosse.
\newblock Optimizing neural networks with {Kronecker}-factored approximate
  curvature.
\newblock In \emph{International Conference on Machine Learning (ICML)},
  volume~37 of \emph{PMLR}, pp.\  2408--2417, 2015.

\bibitem[Merity et~al.(2017)Merity, Xiong, Bradbury, and
  Socher]{merity2017pointer}
Stephen Merity, Caiming Xiong, James Bradbury, and Richard Socher.
\newblock Pointer sentinel mixture models.
\newblock In \emph{International Conference on Learning Representations}, 2017.

\bibitem[Mirzadeh et~al.(2024)Mirzadeh, Alizadeh, Mehta, Del~Mundo, Tuzel,
  Samei, Rastegari, and Farajtabar]{mirzadeh2024relu}
Iman Mirzadeh, Keivan Alizadeh, Sachin Mehta, Carlo~C Del~Mundo, Oncel Tuzel,
  Golnoosh Samei, Mohammad Rastegari, and Mehrdad Farajtabar.
\newblock {ReLU} strikes back: Exploiting activation sparsity in large language
  models.
\newblock In \emph{International Conference on Learning Representations}, 2024.

\bibitem[Misiakos et~al.(2022)Misiakos, Smyrnis, Retsinas, and
  Maragos]{misiakos2022tropical}
Panagiotis Misiakos, Georgios Smyrnis, George Retsinas, and Petros Maragos.
\newblock Neural network approximation based on {Hausdorff} distance of
  tropical zonotopes.
\newblock In \emph{International Conference on Learning Representations
  (ICLR)}, 2022.

\bibitem[Mont{\'u}far et~al.(2014)Mont{\'u}far, Pascanu, Cho, and
  Bengio]{montufar2014linear}
Guido Mont{\'u}far, Razvan Pascanu, Kyunghyun Cho, and Yoshua Bengio.
\newblock On the number of linear regions of deep neural networks.
\newblock In \emph{Advances in Neural Information Processing Systems},
  volume~27, 2014.

\bibitem[Mukhoti et~al.(2024)Mukhoti, Gal, Torr, and Dokania]{mukhoti2024ldifs}
Jishnu Mukhoti, Yarin Gal, Philip H.~S. Torr, and Puneet~K. Dokania.
\newblock Fine-tuning can cripple your foundation model; preserving features
  may be the solution.
\newblock \emph{Transactions on Machine Learning Research (TMLR)}, 2024.

\bibitem[Pan et~al.(2025)Pan, Wang, Cao, Shi, Yang, Li, and Wang]{pan2025fine}
Haowen Pan, Xiaozhi Wang, Yixin Cao, Zenglin Shi, Xun Yang, Juanzi Li, and Meng
  Wang.
\newblock Precise localization of memories: A fine-grained neuron-level
  knowledge editing technique for {LLMs}.
\newblock In \emph{International Conference on Learning Representations
  (ICLR)}, 2025.
\newblock arXiv:2503.01090.

\bibitem[Paperno et~al.(2016)Paperno, Kruszewski, Lazaridou, Pham, Bernardi,
  Pezzelle, Baroni, Boleda, and Fern{\'a}ndez]{paperno2016lambada}
Denis Paperno, Germ{\'a}n Kruszewski, Angeliki Lazaridou, Quan~Ngoc Pham,
  Raffaella Bernardi, Sandro Pezzelle, Marco Baroni, Gemma Boleda, and Raquel
  Fern{\'a}ndez.
\newblock The {LAMBADA} dataset: Word prediction requiring a broad discourse
  context.
\newblock In \emph{Annual Meeting of the Association for Computational
  Linguistics (ACL)}, 2016.

\bibitem[Pilanci \& Ergen(2020)Pilanci and Ergen]{pilanci2020convex}
Mert Pilanci and Tolga Ergen.
\newblock Neural networks are convex regularizers: Exact polynomial-time convex
  optimization formulations for two-layer networks.
\newblock In \emph{Proceedings of the 37th International Conference on Machine
  Learning}, volume 119 of \emph{Proceedings of Machine Learning Research},
  pp.\  7695--7705. PMLR, 2020.

\bibitem[Saha \& Roy(2023)Saha and Roy]{saha2023sgp}
Gobinda Saha and Kaushik Roy.
\newblock Continual learning with scaled gradient projection.
\newblock In \emph{AAAI Conference on Artificial Intelligence}, 2023.
\newblock arXiv:2302.01386.

\bibitem[Saha et~al.(2021)Saha, Garg, and Roy]{saha2021gpm}
Gobinda Saha, Isha Garg, and Kaushik Roy.
\newblock Gradient projection memory for continual learning.
\newblock In \emph{International Conference on Learning Representations
  (ICLR)}, 2021.

\bibitem[Sakaguchi et~al.(2020)Sakaguchi, Le~Bras, Bhagavatula, and
  Choi]{sakaguchi2020winogrande}
Keisuke Sakaguchi, Ronan Le~Bras, Chandra Bhagavatula, and Yejin Choi.
\newblock {WinoGrande}: An adversarial {W}inograd schema challenge at scale.
\newblock In \emph{AAAI Conference on Artificial Intelligence}, 2020.

\bibitem[Schuirmann(1987)]{schuirmann1987tost}
Donald~J. Schuirmann.
\newblock A comparison of the two one-sided tests procedure and the power
  approach for assessing the equivalence of average bioavailability.
\newblock \emph{Journal of Pharmacokinetics and Biopharmaceutics}, 15\penalty0
  (6):\penalty0 657--680, 1987.

\bibitem[Serr{\`a} et~al.(2018)Serr{\`a}, Sur{\'i}s, Miron, and
  Karatzoglou]{serra2018hat}
Joan Serr{\`a}, D{\'i}dac Sur{\'i}s, Marius Miron, and Alexandros Karatzoglou.
\newblock Overcoming catastrophic forgetting with hard attention to the task.
\newblock In \emph{International Conference on Machine Learning (ICML)},
  volume~80, pp.\  4548--4557, 2018.
\newblock arXiv:1801.01423.

\bibitem[Shazeer(2020)]{shazeer2020glu}
Noam Shazeer.
\newblock {GLU} variants improve transformer.
\newblock \emph{arXiv preprint arXiv:2002.05202}, 2020.

\bibitem[Smyrnis \& Maragos(2020)Smyrnis and Maragos]{smyrnis2020multiclass}
Georgios Smyrnis and Petros Maragos.
\newblock Multiclass neural network minimization via tropical {Newton} polytope
  approximation.
\newblock In \emph{International Conference on Machine Learning (ICML)}, volume
  119 of \emph{PMLR}, 2020.

\bibitem[Song et~al.(2025)Song, Han, Zhang, Hu, Shi, Li, Chen, Liu, Li, Yang,
  and Sun]{song2025prosparse}
Chenyang Song, Xu~Han, Zhengyan Zhang, Shengding Hu, Xiyu Shi, Kuai Li, Chen
  Chen, Zhiyuan Liu, Guangli Li, Tao Yang, and Maosong Sun.
\newblock {ProSparse}: Introducing and enhancing intrinsic activation sparsity
  within large language models.
\newblock In \emph{Proceedings of the 31st International Conference on
  Computational Linguistics (COLING)}, 2025.
\newblock arXiv:2402.13516.

\bibitem[{SparseLLM Team}(2023)]{sparsellm2023relullamafix}
{SparseLLM Team}.
\newblock {ReluLLaMA-7B}.
\newblock \url{https://huggingface.co/SparseLLM/ReluLLaMA-7B}, 2023.
\newblock A ReLU-activated fine-tune of Llama-2-7B.

\bibitem[Sun et~al.(2024)Sun, Liu, Bair, and Kolter]{sun2024wanda}
Mingjie Sun, Zhuang Liu, Anna Bair, and J.~Zico Kolter.
\newblock A simple and effective pruning approach for large language models.
\newblock In \emph{International Conference on Learning Representations}, 2024.

\bibitem[Tang et~al.(2026)Tang, Hu, Liu, Ding, Zhang, Wu, and
  Zhang]{tang2026loranull}
Pengwei Tang, Xiaolin Hu, Yong Liu, Lizhong Ding, Dongjie Zhang, Xing Wu, and
  Debing Zhang.
\newblock Put the space of {LoRA} initialization to the extreme to preserve
  pre-trained knowledge.
\newblock In \emph{AAAI Conference on Artificial Intelligence}, 2026.
\newblock arXiv:2503.02659.

\bibitem[Tao et~al.(2023)Tao, Nawas, Mitchell, and Thakur]{tao2023aprnn}
Zhe Tao, Stephanie Nawas, Jacqueline Mitchell, and Aditya~V. Thakur.
\newblock Architecture-preserving provable repair of deep neural networks.
\newblock \emph{Proceedings of the ACM on Programming Languages}, 7\penalty0
  (PLDI):\penalty0 124, 2023.
\newblock arXiv:2304.03496.

\bibitem[Wang et~al.(2021)Wang, Li, Sun, and Xu]{wang2021adamnscl}
Shipeng Wang, Xiaorong Li, Jian Sun, and Zongben Xu.
\newblock Training networks in null space of feature covariance for continual
  learning.
\newblock In \emph{IEEE/CVF Conference on Computer Vision and Pattern
  Recognition (CVPR)}, 2021.
\newblock arXiv:2103.07113.

\bibitem[Wang et~al.(2023)Wang, Chen, Ge, Xia, Bao, Zheng, Zhang, Gui, and
  Huang]{wang2023olora}
Xiao Wang, Tianze Chen, Qiming Ge, Han Xia, Rong Bao, Rui Zheng, Qi~Zhang, Tao
  Gui, and Xuanjing Huang.
\newblock Orthogonal subspace learning for language model continual learning.
\newblock In \emph{Findings of the Association for Computational Linguistics:
  EMNLP}, 2023.

\bibitem[Welbl et~al.(2017)Welbl, Liu, and Gardner]{welbl2017sciq}
Johannes Welbl, Nelson~F. Liu, and Matt Gardner.
\newblock Crowdsourcing multiple choice science questions.
\newblock In \emph{Workshop on Noisy User-generated Text (W-NUT)}, 2017.

\bibitem[Wortsman et~al.(2022)Wortsman, Ilharco, Kim, Li, Kornblith, Roelofs,
  Gontijo-Lopes, Hajishirzi, Farhadi, Namkoong, and
  Schmidt]{wortsman2022wiseft}
Mitchell Wortsman, Gabriel Ilharco, Jong~Wook Kim, Mike Li, Simon Kornblith,
  Rebecca Roelofs, Raphael Gontijo-Lopes, Hannaneh Hajishirzi, Ali Farhadi,
  Hongseok Namkoong, and Ludwig Schmidt.
\newblock Robust fine-tuning of zero-shot models.
\newblock In \emph{IEEE/CVF Conference on Computer Vision and Pattern
  Recognition (CVPR)}, 2022.

\bibitem[Yang et~al.(2024)Yang, Li, Zhou, Song, Wu, Nie, and
  Ghanem]{yang2024corda}
Yibo Yang, Xiaojie Li, Zhongzhu Zhou, Shuaiwen~Leon Song, Jianlong Wu, Liqiang
  Nie, and Bernard Ghanem.
\newblock {CorDA}: Context-oriented decomposition adaptation of large language
  models for task-aware parameter-efficient fine-tuning.
\newblock In \emph{Advances in Neural Information Processing Systems
  (NeurIPS)}, 2024.
\newblock arXiv:2406.05223.

\bibitem[Zellers et~al.(2019)Zellers, Holtzman, Bisk, Farhadi, and
  Choi]{zellers2019hellaswag}
Rowan Zellers, Ari Holtzman, Yonatan Bisk, Ali Farhadi, and Yejin Choi.
\newblock {HellaSwag}: Can a machine really finish your sentence?
\newblock In \emph{Annual Meeting of the Association for Computational
  Linguistics (ACL)}, 2019.

\bibitem[Zhang et~al.(2018)Zhang, Naitzat, and Lim]{zhang2018tropical}
Liwen Zhang, Gregory Naitzat, and Lek-Heng Lim.
\newblock Tropical geometry of deep neural networks.
\newblock In \emph{Proceedings of the 35th International Conference on Machine
  Learning}, volume~80 of \emph{Proceedings of Machine Learning Research}, pp.\
   5824--5832. PMLR, 2018.

\bibitem[Zhang et~al.(2022)Zhang, Roller, Goyal, Artetxe, Chen, Chen, Dewan,
  Diab, Li, Lin, Mihaylov, Ott, Shleifer, Shuster, Simig, Koura, Sridhar, Wang,
  and Zettlemoyer]{zhang2022opt}
Susan Zhang, Stephen Roller, Naman Goyal, Mikel Artetxe, Moya Chen, Shuohui
  Chen, Christopher Dewan, Mona Diab, Xian Li, Xi~Victoria Lin, Todor Mihaylov,
  Myle Ott, Sam Shleifer, Kurt Shuster, Daniel Simig, Punit~Singh Koura, Anjali
  Sridhar, Tianlu Wang, and Luke Zettlemoyer.
\newblock {OPT}: Open pre-trained transformer language models.
\newblock \emph{arXiv preprint arXiv:2205.01068}, 2022.

\end{thebibliography}
